%% file: main.tex
\documentclass{article} 
\usepackage{iclr2027_conference,times}

\input{math_commands.tex}

\usepackage{amsmath,amssymb,amsthm,mathtools}
\usepackage{graphicx}
\usepackage{subcaption}
\usepackage{hyperref}
\usepackage{url}
\usepackage[capitalize,noabbrev]{cleveref}

\allowdisplaybreaks
\newtheorem{theorem}{Theorem}[section]
\newtheorem{lemma}[theorem]{Lemma}
\newtheorem{proposition}[theorem]{Proposition}
\newtheorem{corollary}[theorem]{Corollary}
\theoremstyle{definition}
\newtheorem{assumption}[theorem]{Assumption}

\theoremstyle{remark}

\makeatletter
\renewenvironment{proof}[1][\proofname]{\par
  \pushQED{\qed}%
  \normalfont \topsep6\p@\@plus6\p@\relax
  \trivlist
  \item[\hskip\labelsep
        \bfseries
    #1\@addpunct{.}]\ignorespaces
}{%
  \popQED\endtrivlist\@endpefalse
}
\makeatother

\newcounter{stepcount}
\newcounter{proofstepgroup}
\crefname{stepcount}{step}{steps}
\Crefname{stepcount}{Step}{Steps}
\AddToHook{env/proof/begin}{%
  \stepcounter{proofstepgroup}%
  \setcounter{stepcount}{0}%
}

\newcounter{substepcount}[stepcount]
\renewcommand{\thesubstepcount}{\thestepcount-\arabic{substepcount}} 

\crefname{substepcount}{step}{steps}
\Crefname{substepcount}{Step}{Steps}
\crefname{ineq}{Inequality}{Inequalities}
\creflabelformat{ineq}{(#2\textup{#1}#3)}
\crefname{assumption}{assumption}{assumptions}
\Crefname{assumption}{Assumption}{Assumptions}

\newcommand{\step}[1]{%
  \leavevmode
  \refstepcounter{stepcount}%
  \paragraph{Step \thestepcount: #1}%
}

\newcommand{\substep}[1]{%
  \refstepcounter{substepcount}%
  \paragraph{Step \thesubstepcount: #1}%
}

\hypersetup{unicode=true,colorlinks=true,linkcolor=blue,urlcolor=blue}

\title{Depth Improves Training-Time Scaling for Difficult Tasks in Deep Linear Networks}

\author{Taishi Watanabe \\
  Kyoto University\\
  \texttt{taishi.watanabe.63w@st.kyoto-u.ac.jp} \\
  \And
  Ryo Karakida \\
   AIST\\
   Institute of Science Tokyo\\
   \texttt{karakida.ryo@aist.go.jp} \\
   \And
 Jun-nosuke Teramae \\
  Kyoto University\\
  \texttt{teramae@acs.i.kyoto-u.ac.jp} \\
}

\iclrfinalcopy 
\begin{document}

\maketitle
\lhead{Preprint}

\begin{abstract}
  Neural scaling laws describe the predictable power-law dependence of model performance on training time, dataset size, and model scale, with the exponents quantifying how performance changes with resources. Recent theory has shown how scaling exponents are shaped by the data distribution and the structure of the learning task, but whether network depth can improve training-time scaling has remained unsettled. Through a refined analysis of deep linear networks, we clarify that the task spectrum determines when increasing network depth improves training-time scaling. For spectrally difficult tasks, increasing depth improves the asymptotic training-time scaling exponent, whereas for easy tasks, the exponent remains independent of depth, with depth affecting only a logarithmic correction at the boundary between these regimes. Remarkably, we find that a reparameterization of time governed by the predictor norm provides a unified description of the dynamics across task difficulties. Under this reparameterization, the loss follows a common scaling law whose exponent is independent of depth. Numerical experiments on a task derived from benchmark image data and with nonlinear networks exhibit qualitatively consistent task-dependent effects of depth on training-time scaling. Our findings advance the theoretical understanding of how task structure shapes the effects of network architecture on learning and offer a basis for studying scaling behavior across broader classes of neural networks.
\end{abstract}

\section{Introduction}

Performance in modern deep learning often improves predictably as training resources increase. Across a wide range of models and tasks, learning curves have often been observed to follow power laws with respect to training time, dataset size, and model scale \citep{hestness2017deep,kaplan2020scaling,hoffmann2022training}. These empirical power-law relationships, commonly referred to as neural scaling laws, provide useful guidance for extrapolating performance to larger training scales and for allocating limited computational resources. Identifying what determines the exponents of these scaling laws is therefore an important step toward a theoretical understanding of learning in large neural networks.

Recent theory has clarified how scaling exponents are shaped by the structure of the data and the learning task \citep{berthier2020tight,bordelon2022learning,maloney2022solvable,bordelon2024dynamical}. For example, in kernel methods and random-feature models, scaling exponents depend on the spectral structure of the data and the target function \citep{spigler2020asymptotic,bordelon2020spectrum,canatar2021spectral,defilippis2024dimension}. More recently, \citet{bordelon2025feature} analyzed a two-layer model on power-law tasks and showed that training-time scaling depends on task difficulty. By comparison, the role of network architecture in determining these exponents remains less well characterized.

Depth provides a particularly direct way to study architectural effects on scaling. Deep linear networks are a fundamental model for studying the effects of depth \citep{baldi1989neural,Fukumizu1998Effect,saxe2013exact}. Although depth does not change the class of functions represented by the network, factorization across multiple layers induces nonlinear optimization dynamics, and this setting has enabled detailed analyses of convergence, optimization, and implicit bias \citep{arora2018convergence,du2019width,gidel2019implicit,gissin2020the,tarmoun2021understanding}. In particular, \citet{arora2018optimization} showed that network depth modifies the preconditioner governing the end-to-end predictor dynamics and can thereby accelerate optimization. Their explicit acceleration result, however, is based on $\ell_p$ regression with $p>2$, rather than the standard squared loss. More directly for scaling laws, \citet{bordelon2025deep} derived scaling exponents for deep linear networks using a self-consistency argument inherited from the two-layer analysis of \citet{bordelon2025feature},  predicting no depth-induced acceleration in training-time scaling, with the same asymptotic exponent across depths. However, this argument directly applies the scaling relation obtained in the two-layer setting to general depth, without establishing its validity beyond two layers. Thus, even in deep linear networks, whether and when depth improves the training-time scaling exponent under standard squared loss has remained open.

In this work, we resolve this question through a refined analysis of deep linear networks on tasks with power-law spectral structure. By explicitly tracking the end-to-end predictor dynamics induced by the multilayer parameterization, we characterize the role of depth in training-time scaling. Our contributions are as follows:

\begin{itemize}
  \item\textbf{Depth improves training-time scaling for spectrally difficult tasks.}
  We show that increasing depth can improve asymptotic training-time scaling under squared loss. For spectrally difficult tasks, where the target has relatively large components along low-eigenvalue directions of the data covariance, the decay exponent of the loss strictly increases with depth under sufficiently small random initialization (\cref{thm:training-time-loss-scaling}). We further show that the same asymptotic scaling extends to standard random initialization and characterize finite-dimensional effects that can delay its onset (\cref{sec:standard-init-finite-size}).
  \item \textbf{Predictor-norm time reparameterization separates the effects of task spectrum and depth.}
        We introduce a reparameterized time defined through the predictor norm (\cref{phased:cor:rescaled-predictor-tau-dynamics}). In this time coordinate, the loss has the same asymptotic scaling exponent across depths, determined solely by the task spectrum (\cref{prop:common-spectral-law-loss}).
  \item \textbf{Predictor-norm growth distinguishes three task regimes.}
        We prove that the predictor norm grows polynomially, logarithmically, or remains bounded depending on the task spectrum (\cref{prop:predictor-norm-scaling}). This characterization shows how predictor growth changes with task difficulty and, through the relation between reparameterized time and training time, determines the depth dependence of the training-time scaling law (\cref{thm:training-time-loss-scaling}).
\end{itemize}
Finally, experiments on a regression task derived from benchmark image data and on nonlinear ReLU networks exhibit qualitatively similar depth dependence across task regimes. These results provide a theoretical basis for studying how architecture and task structure jointly determine scaling laws in more complex neural networks.

\section{Related Work}

\paragraph{Theory of Scaling Laws and Spectral Learning Curves.}

Scaling laws have been derived analytically in a range of tractable learning models with structured spectra. Convergence rates have been characterized through spectral and regularity conditions on the data and target \citep{caponnetto2007optimal,pillaud2018statistical,cui2021generalization}. Sample-size learning curves have also been characterized in terms of the kernel eigenspectrum and the decomposition or alignment of the target in the associated eigenbasis \citep{spigler2020asymptotic,bordelon2020spectrum,canatar2021spectral}.

Beyond sample-size scaling, theoretical work has characterized training-time scaling \citep{berthier2020tight,bordelon2022learning,kramp2026dynamics}. Scaling with dataset and model size \citep{maloney2022solvable,bahri2024explaining,lin2024scaling,defilippis2024dimension,atanasov2026scaling} and compute-optimal scaling \citep{bordelon2024dynamical,paquette20244+} have also been studied.

Recent theoretical work has also begun to characterize scaling laws in nonlinear models with structured data or targets. These analyses include problems with structured or anisotropic data spectra \citep{worschech2025analyzing,braun2026fast}, as well as single- and multi-index models in which scaling laws emerge from the sequential acquisition of target components \citep{ren2025emergence,ben2025learning,defilippis2026optimal}. More directly related to our setting, \citet{bordelon2025feature} analyzed a two-layer linear model on power-law tasks and showed that the training-time scaling exponent depends on task difficulty.

\paragraph{Deep Linear Networks and the Role of Depth.}

Deep linear networks provide a tractable setting for studying the role of depth in learning dynamics. Although the predictor remains linear in the input, its factorization across multiple layers induces nonlinear parameter dynamics that depend on depth. Early work characterized the loss landscapes and learning dynamics \citep{baldi1989neural,Fukumizu1998Effect,saxe2013exact}. Subsequent work studied convergence \citep{arora2018convergence,du2019width}, generalization \citep{lampinen2018analytic,advani2020high}, and implicit bias \citep{gidel2019implicit,gissin2020the}. Later work further characterized how initialization scale and layer balance shape the learning regime of deep linear networks \citep{woodworth2020kernel,tarmoun2021understanding,min2021explicit,jacot2021saddle,braun2022exact,atanasov2022neural,kunin2024get,domine2025lazy}.

An explicit mechanism by which depth can accelerate optimization was identified by \citet{arora2018optimization}, who showed that additional layers modify the end-to-end predictor dynamics through a preconditioner. Their acceleration result, however, crucially exploits the structure of the $\ell_p$ loss for $p>2$. For $\ell_p$ regression, the prediction-error dynamics involve a multiplicative interaction between the predictor-dependent preconditioning induced by depth and an error-dependent weighting arising from the $\ell_p$ loss. In their ill-conditioned two-coordinate example, the dominant coordinate constrains the global step size in the shallow parameterization, causing the other coordinate to converge much more slowly. In the deep parameterization, progress along the dominant coordinate increases the predictor norm, which raises the effective step size for the remaining coordinate and accelerates its convergence. At the same time, the shrinking error-dependent factor on the dominant coordinate suppresses its own update through the $\ell_p$-dependent factor, preserving stability. For squared loss, this error-dependent factor is absent, so the mechanism underlying their acceleration result does not directly extend.

Most closely related to our setting, \citet{bordelon2025deep} study learning on power-law structured tasks in deep linear networks and derive an asymptotic training-time scaling exponent that varies with task structure but not with network depth. Their scaling analysis builds on the self-consistency framework developed for the two-layer model in \citet{bordelon2025feature}. One component of this framework, the relation between the decay of the spectral error modes and the accumulated kernel scale, also applies beyond two layers. The depth dependence enters through the way layer growth determines the kernel scale. In the two-layer model, the kernel scale is determined by contributions from the hidden and readout layers, so its overall scale depends linearly on their growth. For deeper factorizations, however, each contribution to the kernel involves products across multiple layers, so the resulting kernel-growth relation depends on the number of layers. Their general-depth analysis nevertheless determined the scaling exponent using the two-layer kernel-growth relation without deriving its multilayer counterpart. Although the two analyses concern different settings and are therefore not directly comparable, they leave open whether and when depth changes training-time scaling under standard squared loss.

\section{Problem Setup}

\vspace{-8pt}
\paragraph{Notation.}

For nonnegative quantities $f$ and $g$, we write $f\lesssim g$ if $f\le Cg$, and $f\asymp g$ if $cg\le f\le Cg$, for positive constants $c,C$ uniform in the parameters under consideration. We define $f\gtrsim g$ analogously. We take the maximum of an empty family of nonnegative quantities to be zero.

\vspace{-8pt}
\paragraph{Data distribution.}

We consider a supervised regression problem with training data $\{(\boldsymbol x_i, y_i)\}_{i=1}^n$ where $\boldsymbol{x}_i\in\mathbb{R}^d$ and $y_i\in\mathbb{R}$. The inputs $\boldsymbol{x}_i$ are drawn i.i.d. from a distribution with zero mean and covariance
\begin{equation}
  \boldsymbol{\Sigma}=\mathrm{diag}(\lambda_1,\ldots,\lambda_d),\qquad \lambda_i=c_\lambda i^{-a}.
  \label{eq:power-law-spectrum}
\end{equation}
Here, $c_\lambda>0$ and $a>0$. The labels are generated by the linear model $y_i=\boldsymbol{\beta}^\top \boldsymbol{x}_i$, with
\begin{equation}
  \boldsymbol{\beta}=(\beta_1,\ldots,\beta_d)^\top,\qquad \beta_i=c_\beta i^{-b},
\end{equation}
where $c_\beta>0$ and $b\in\mathbb R$.

We define the task exponent $p=(a+2b-1)/a$ and assume $p>0$. This condition ensures that $\boldsymbol\beta^\top\boldsymbol\Sigma\boldsymbol\beta=c_\lambda c_\beta^2\sum_{i=1}^d i^{-ap-1}$ remains bounded as $d\to\infty$. The exponent $p$ characterizes the task spectrum and distinguishes the task regimes analyzed below. The case $p=1$ (equivalently, $b=1/2$) separates the polynomially divergent and bounded regimes of the target norm, with $\|\boldsymbol\beta\|^2$ growing logarithmically at the boundary.

\vspace{-8pt}
\paragraph{Model.}

We parameterize the predictor using a linear network with $L\ge1$ layers:
\begin{equation}
  f_{\boldsymbol \theta}(\boldsymbol x)
  = \boldsymbol\theta^\top \boldsymbol x,
  \qquad
  \boldsymbol\theta
  = (\boldsymbol W_L \cdots \boldsymbol W_1)^\top.
\end{equation}
Here $\boldsymbol W_k\in\mathbb R^{n_k\times n_{k-1}}$, where $n_k$ denotes the width of the $k$-th layer. We set $n_0=d$ and $n_L=1$ and assume $n_k\asymp d$ for $1\le k<L$. To track the scale associated with the depth-$L$ factorization, we write $\sigma=\|\boldsymbol\theta\|^{1/L}$. Depth does not change the represented class of linear predictors. For $L\ge2$, the factorized parameterization induces nonlinear dynamics.

\vspace{-8pt}
\paragraph{Training algorithm and initialization.}

We train all layers jointly by population gradient flow for the squared loss $\mathcal L(\boldsymbol\theta)=\frac12\|\boldsymbol\Sigma^{1/2}(\boldsymbol\beta-\boldsymbol\theta)\|^2$. For $j\ge k$, write $\boldsymbol W_{j:k}=\boldsymbol W_j\cdots\boldsymbol W_k$ with the conventions $\boldsymbol W_{0:1}=\boldsymbol I$ and $\boldsymbol W_{L:L+1}=1$. For $1\le k\le L$, the layer dynamics are
\begin{equation}
  \frac{d}{dt}\boldsymbol{W}_k
  =-\nabla_{\boldsymbol W_k}\mathcal L=\boldsymbol{W}_{L:k+1}^\top(\boldsymbol{\Sigma}(\boldsymbol{\beta}-\boldsymbol{\theta}))^\top\boldsymbol{W}_{k-1:1}^\top.
  \label{phased:eq:gradient-flow-dln}
\end{equation}
All entries are initialized independently according to $(\boldsymbol W_k(0))_{ij}\sim \mathcal N(0,s^2/n_{k-1})$ for $1\le k\le L$, where $s>0$ is the initialization scale.

\vspace{-8pt}
\section{Depth-Dependent Training-Time Scaling under Small Initialization}

In this section, we derive the training-time loss scaling under sufficiently small random initialization and characterize when network depth improves the asymptotic decay exponent. The precise assumptions on the initialization scale $s$ used in the analysis are stated in \cref{phased:ass:small-initialization}. To focus on the subsequent power-law scaling phase while accounting for the initial dynamics near the origin, we introduce a time $t_{\mathrm e}$, defined precisely below, and use the shifted training time $\bar t=t-t_{\mathrm e}$.

\vspace{-8pt}
\paragraph{Numerical observation.}
\Cref{fig:vanishing-loss-shifted-time} shows the loss as a function of $\bar t$ for representative values of $p$ below, at, and above $1$. For $p<1$, increasing depth produces a steeper power-law decay of the loss. At $p=1$, the effect of depth is weaker, and for $p>1$, the power-law exponent does not change.

\begin{figure}[t]
  \centering

  \begin{subfigure}[t]{0.32\textwidth}
    \centering
    \includegraphics[width=\linewidth]
    {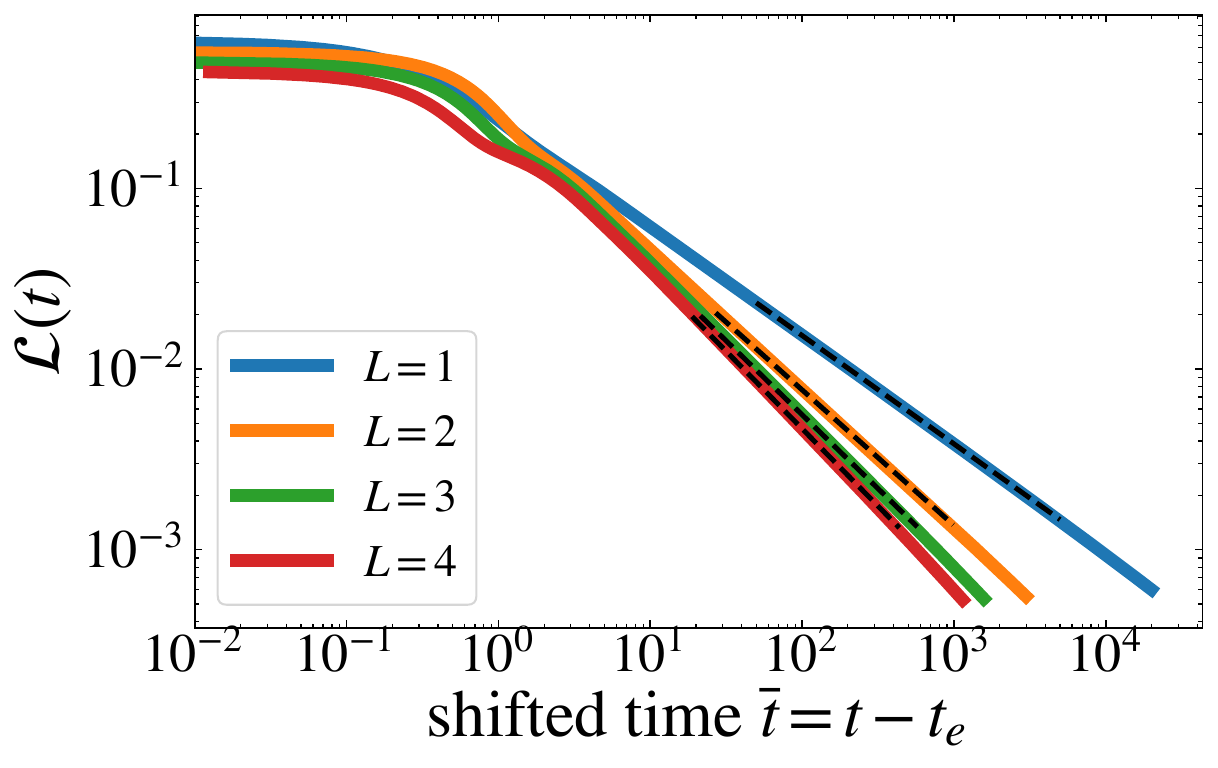}
    \caption{$p=0.6$}
    \label{fig:vanishing-loss-p06}
  \end{subfigure}
  \hfill
  \begin{subfigure}[t]{0.32\textwidth}
    \centering
    \includegraphics[width=\linewidth]
    {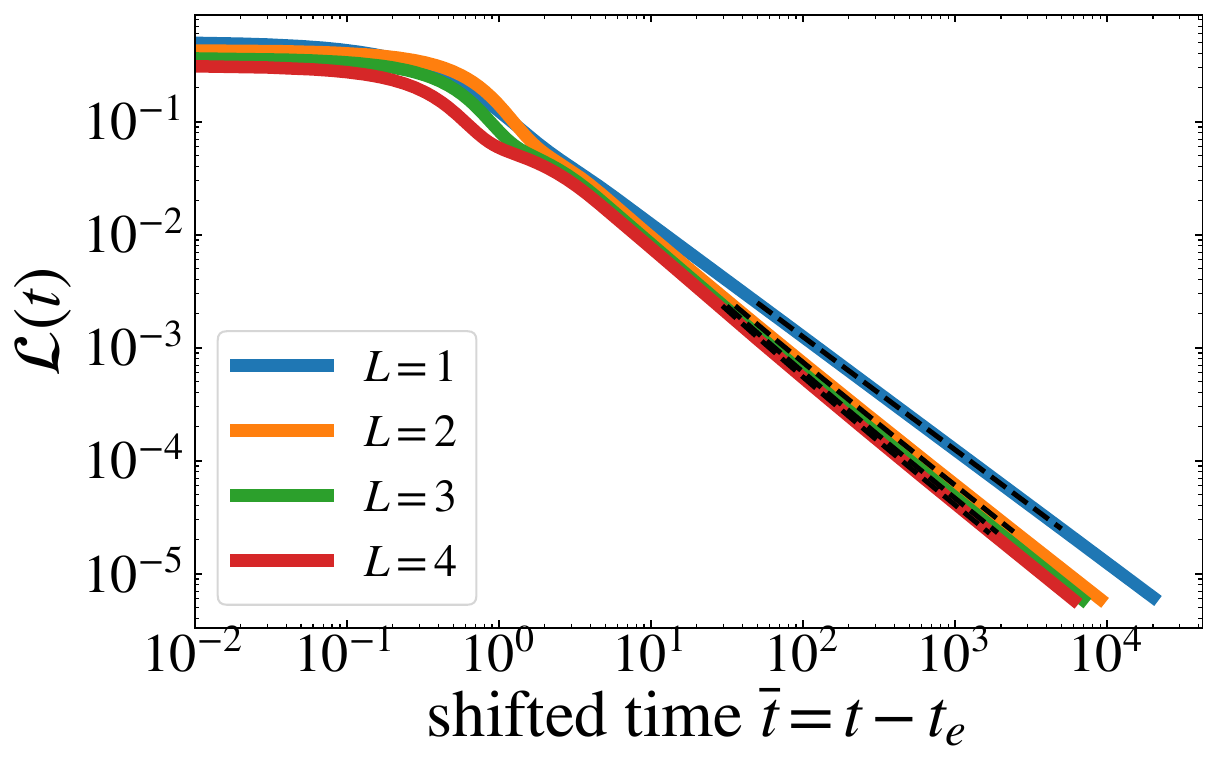}
    \caption{$p=1.0$}
    \label{fig:vanishing-loss-p10}
  \end{subfigure}
  \hfill
  \begin{subfigure}[t]{0.32\textwidth}
    \centering
    \includegraphics[width=\linewidth]
    {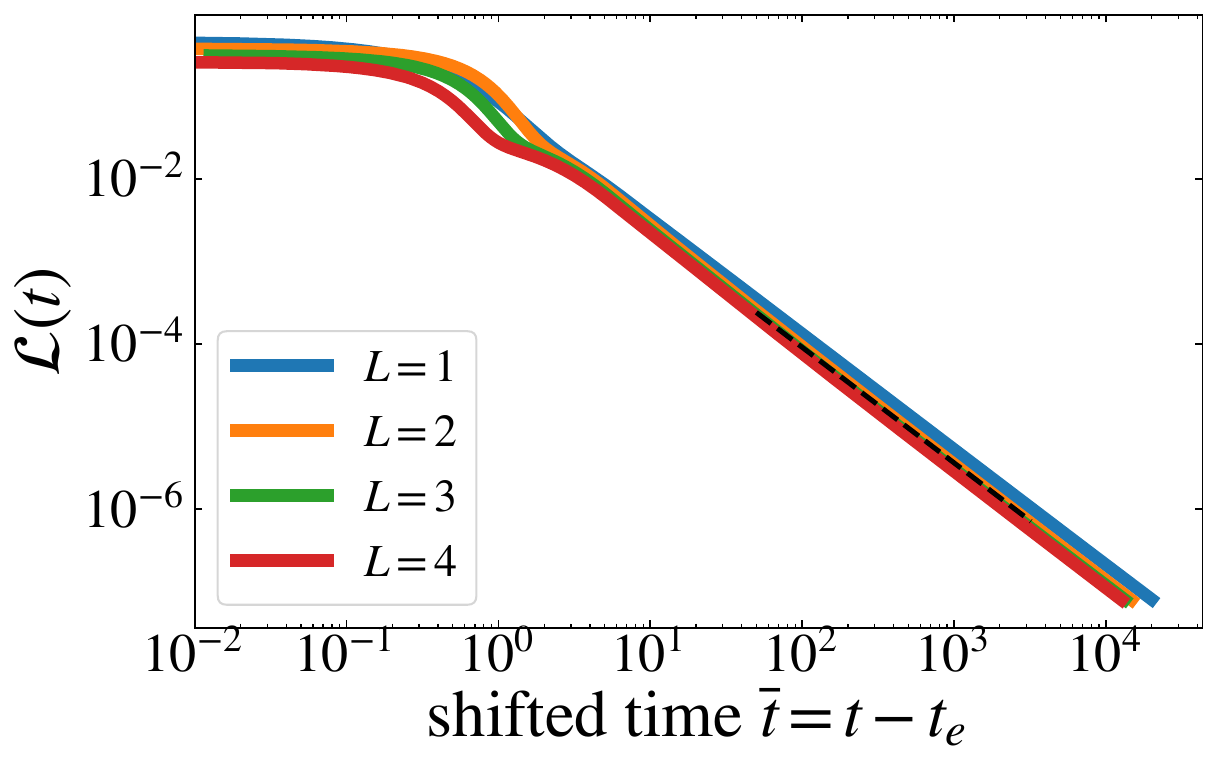}
    \caption{$p=1.4$}
    \label{fig:vanishing-loss-p14}
  \end{subfigure}

  \caption{
    Increasing depth improves the training-time scaling exponent for difficult tasks
    with $p<1$, affects only the logarithmic correction at the boundary $p=1$, and leaves the
    power-law exponent unchanged for easy tasks with $p>1$.
    Dashed lines indicate the asymptotic training-time scaling laws given by
    \cref{thm:training-time-loss-scaling}.
  }
  \vspace{-8pt}
  \label{fig:vanishing-loss-shifted-time}
\end{figure}

\vspace{-8pt}
\paragraph{Overview of the analysis.}
We proceed in several steps. We first analyze the end-to-end predictor dynamics and introduce the rescaled predictor $\boldsymbol{z}=\sigma^{-(L-1)}\boldsymbol{\theta}$, which cancels the leading depth-dependent anisotropy of the preconditioner (\cref{phased:ss:depth-effect-predictor-dynamics}). We next analyze the initial escape phase and characterize the dependence of the time $t_{\mathrm e}$ on the initialization scale (\cref{phased:ss:initial-dynamics-near-origin}). In the subsequent power-law scaling phase, we introduce the reparameterized time $\tau$ and analyze the dynamics in three steps. We first characterize the growth of the predictor norm as a function of $\tau$ (\cref{phased:ss:predictor-norm-scaling}), then derive the loss scaling in $\tau$ and show that its exponent is independent of depth (\cref{sec:common-spectral-law-loss}), and finally translate this law back to the shifted training time $\bar t$ (\cref{sec:depth-dependence-loss}).

\subsection{Depth-Dependent Preconditioning and Rescaled Predictor Dynamics}
\label{phased:ss:depth-effect-predictor-dynamics}

In this subsection, we introduce the rescaled predictor $\boldsymbol{z}=\sigma^{-(L-1)}\boldsymbol{\theta}$ and derive its dynamics, for which the leading depth-dependent anisotropy of the end-to-end preconditioner is canceled.
Since the loss depends on the network parameters only through the end-to-end predictor $\boldsymbol\theta$, we begin with its dynamics, in which the effect of depth appears through a preconditioner. Specifically,
\begin{equation}
  \frac{d\boldsymbol\theta}{dt}
  =
  \boldsymbol K_L
  \boldsymbol\Sigma
  (\boldsymbol\beta-\boldsymbol\theta).
\end{equation}
The preconditioner induced by the factorization is $\boldsymbol K_L=\sum_{k=1}^L\|\boldsymbol W_{L:k+1}\|^2\boldsymbol W_{k-1:1}^{\top}\boldsymbol W_{k-1:1}$.
From \cref{phased:eq:gradient-flow-dln}, the adjacent-layer Gram differences are conserved: $\frac{d}{dt}(\boldsymbol W_k\boldsymbol W_k^\top)=\frac{d}{dt}(\boldsymbol W_{k+1}^\top\boldsymbol W_{k+1})$. We quantify this imbalance by $\varepsilon=\max_{1\le k<L}\|\boldsymbol{W}_k\boldsymbol{W}_k^\top-\boldsymbol{W}_{k+1}^\top \boldsymbol{W}_{k+1}\|_{\mathrm{op}}$. When $\varepsilon/\sigma^2$ is small, this near-balance keeps all layers at the common scale $\sigma$, so each term in $\boldsymbol K_L$ has overall scale $\sigma^{2(L-1)}$. As shown in \cref{commonwork:prop:perturbative-kernel-expansion}, we get
\begin{equation}
  \boldsymbol{K}_L = \sigma^{2(L-1)}\left(\boldsymbol I + (L-1)\boldsymbol P_{\boldsymbol\theta} + \Delta\overline{\boldsymbol K}_L\right).
  \label{eq:depth-preconditioner}
\end{equation}
Here, for $\boldsymbol\theta\neq\boldsymbol 0$, we define $\boldsymbol P_{\boldsymbol\theta}=\boldsymbol\theta\boldsymbol\theta^\top/\|\boldsymbol\theta\|^2$, the orthogonal projector onto $\operatorname{span}\{\boldsymbol\theta\}$, and $\|\Delta\overline{\boldsymbol K}_L\|_{\mathrm{op}}=O(\varepsilon/\sigma^2)$.
\cref{eq:depth-preconditioner} shows that, for $L\ge2$, the leading part of the preconditioner is anisotropic, with the depth dependence appearing through the rank-one term $(L-1)\boldsymbol P_{\boldsymbol\theta}$. This motivates the rescaling $\boldsymbol z=\sigma^{-(L-1)}\boldsymbol\theta$. Defining $\boldsymbol J_L=\sigma^{L-1}\partial_{\boldsymbol\theta}\boldsymbol z=\boldsymbol I-\frac{L-1}{L}\boldsymbol P_{\boldsymbol\theta}$, we have $\boldsymbol{J}_L(\boldsymbol I+(L-1)\boldsymbol P_{\boldsymbol\theta})=\boldsymbol{I}$, so the leading anisotropic factor is exactly canceled in the dynamics of $\boldsymbol z$.
\begin{lemma}[Rescaled predictor dynamics]
  \label{phased:lem:rescaled-predictor-dynamics}
  For $\boldsymbol z=\sigma^{-(L-1)}\boldsymbol\theta$, the dynamics satisfy
  \begin{equation}
    \frac{d\boldsymbol{z}}{dt}=(\boldsymbol{I}+\boldsymbol{J}_L\Delta\overline{\boldsymbol{K}}_L)\boldsymbol{\Sigma}(\sigma^{L-1}\boldsymbol{\beta}-\sigma^{2(L-1)}\boldsymbol{z}).
    \label{eq:dzdt-depth-preconditioned}
  \end{equation}
\end{lemma}
See \cref{commonwork:ss:rescaled-predictor-dynamics} for the derivation.

\subsection{Escape from the origin}
\label{phased:ss:initial-dynamics-near-origin}

Before analyzing the power-law scaling phase, we characterize the initial escape from the origin. In particular, we determine how the time $t_{\mathrm e}$ at which the predictor scale $\sigma$ has reached an $O(1)$ scale depends on the initialization scale $s$. From \cref{eq:dzdt-depth-preconditioned}, when $\sigma \ll 1$, the relaxation term is lower order and the forcing term has scale $\sigma^{L-1}$, suggesting $d\sigma/dt\approx \sigma^{L-1}$. Hence the time $t_{\mathrm e}$ satisfies
\begin{equation}
  t_{\mathrm e}\asymp
  \begin{cases}
    1,          & L=1,   \\
    \log(1/s),  & L=2,   \\
    s^{-(L-2)}, & L\ge3.
  \end{cases}
\end{equation}
A rigorous derivation is given in \cref{phased:cor:time-entrance-delay}. We therefore measure the subsequent dynamics using the shifted time $\bar t=t-t_{\mathrm e}$.

\subsection{Predictor Norm and Loss Scaling}
\label{phased:sec:predictor-norm-loss-scaling}

To characterize the loss scaling in the power-law scaling phase, we introduce a reparameterized time in which the leading spectral dynamics take a common form across depths. This separates the common spectral scaling in reparameterized time from the depth-dependent conversion back to training time.
\begin{corollary}[Dynamics in reparameterized time]
  \label{phased:cor:rescaled-predictor-tau-dynamics}
  We define the time reparameterization
  \begin{equation}
    \tau(t)=\int_0^t\sigma(v)^{2(L-1)}\,dv.
    \label{phased:eq:time-reparametrization}
  \end{equation}
  Then the rescaled predictor satisfies:
  \begin{equation}
    \frac{d\boldsymbol{z}}{d\tau}=-\boldsymbol{\Sigma}\boldsymbol{z}+\sigma^{-(L-1)}\boldsymbol{\Sigma}\boldsymbol{\beta}+\boldsymbol{f}_K,
    \label{phased:eq:rescaled-predictor-dynamics}
  \end{equation}
  where $\boldsymbol{f}_K=\boldsymbol J_L\Delta\overline{\boldsymbol K}_L\boldsymbol\Sigma\left(\sigma^{-(L-1)}\boldsymbol\beta-\boldsymbol z\right)$.
\end{corollary}
See \cref{commonwork:ss:rescaled-predictor-dynamics} for the derivation.
The relaxation term $-\boldsymbol\Sigma\boldsymbol z$ is independent of depth. Depth enters only through the predictor scale $\sigma$ and the perturbation $\boldsymbol f_K$. The subsequent analysis therefore reduces to characterizing the growth of the predictor norm and the resulting spectral learning dynamics in $\tau$.

Since $\lambda_d \asymp d^{-a}$ is the smallest eigenvalue of $\boldsymbol{\Sigma}$, the slowest spectral time scale is of order $d^a$. We therefore fix a sufficiently small constant $c_0>0$ and set $\tau_+=c_0d^a$. Let $t_+>t_{\mathrm e}$ be the corresponding time defined by $\tau(t_+)=\tau_+$, and set $\bar t_+=t_+-t_{\mathrm e}$. We restrict the analysis in this subsection to $t_{\mathrm e}\le t\le t_+$, equivalently $\tau_{\mathrm e}\le\tau\le\tau_+$ with $\tau_{\mathrm e} = \tau(t_{\mathrm e})$, before effects associated with the smallest eigenvalue become relevant. Beyond this window, the loss decays exponentially in reparameterized time (see \cref{phased:cor:exponential-loss-decay}).

\subsubsection{Scaling of the predictor norm}
\label{phased:ss:predictor-norm-scaling}

The relation between reparameterized time $\tau$ and training time is controlled by the predictor scale, since $d\tau/dt=\sigma^{2(L-1)}$. We therefore characterize $\sigma(\tau)=\|\boldsymbol{\theta}(\tau)\|^{1/L}$ through the spectral dynamics of the predictor.
Applying the variation of constants formula to the dynamics above and using $\boldsymbol{\theta}=\sigma^{L-1}\boldsymbol{z}$, we obtain
\begin{equation}
  \boldsymbol\theta(\tau) = \sigma(\tau)^{L-1} \int_0^\tau e^{-(\tau-v)\boldsymbol\Sigma} \boldsymbol\Sigma\boldsymbol\beta\, \sigma(v)^{-(L-1)}\,dv + \boldsymbol R(\tau),
  \label{eq:predictor-representation}
\end{equation}
where $\boldsymbol R(\tau)$ collects the effects of initialization and the conserved layer imbalance (see \cref{phased:cor:post-entry-representation} for its exact representation).
Under sufficiently small initialization, $\boldsymbol R(\tau)$ does not change the predictor-norm scaling derived below, as shown in \cref{phased:sec:radius-law}.

In the $i$-th eigendirection, the exponential factor $e^{-(\tau-v)\lambda_i}$ suppresses contributions from times $v$ much earlier than $\tau$, over a time scale $1/\lambda_i$. The behavior therefore changes around $\tau\lambda_i\asymp1$.
For $\tau\lambda_i\gg1$, the integral is dominated by $v$ near $\tau$, giving $\theta_i(\tau)\approx\beta_i$. For $\tau\lambda_i\ll1$, we have $e^{-(\tau-v)\lambda_i}\approx 1$, so $\theta_i(\tau)\approx \beta_i\lambda_i\,\sigma(\tau)^{L-1} \int_0^\tau \sigma(v)^{-(L-1)}\,dv$. To obtain the scaling behavior, we make the self-consistent ansatz $\sigma(\tau)^{L-1}\int_0^\tau\sigma(v)^{-(L-1)}\,dv\asymp\tau$. This gives $\theta_i(\tau)\asymp\beta_i\tau\lambda_i$. Combining the two cases and summing over the eigendirections gives $\|\boldsymbol{\theta}(\tau)\|^2\asymp\sum_{i=1}^d\beta_i^2\min\{\tau\lambda_i,1\}^2$.
Substituting $\lambda_i=c_\lambda i^{-a}$ and $\beta_i=c_\beta i^{-b}$ into this expression, we obtain
\begin{proposition}[Predictor-norm scaling, informal]
  \label{prop:predictor-norm-scaling}
  Under sufficiently small initialization, there exists a constant $\tau_0>0$, independent of $d$ and $s$, such that, for all sufficiently large $d$, with high probability, uniformly for $\tau_0\le\tau\le\tau_+$,
  \begin{equation}
    \label{eq:predictor-norm-scaling}
    \|\boldsymbol{\theta}(\tau)\|^2\asymp
    \begin{cases}
      \tau^{1-p}, & p < 1, \\
      \log\tau,   & p = 1, \\
      1,          & p > 1.
    \end{cases}
  \end{equation}
\end{proposition}
The resulting scaling of $\sigma(\tau)=\|\boldsymbol{\theta}(\tau)\|^{1/L}$  satisfies the assumed relation, confirming the self-consistency of the ansatz. A rigorous derivation is given in \cref{cor:small-init-predictor-norm-scaling} and does not rely on the self-consistency ansatz. \cref{fig:vanishing-predictor-norm-tau} illustrates the three regimes for the squared predictor norm: polynomial growth, logarithmic growth, and bounded behavior, respectively.

\begin{figure}[t]
  \centering
  \begin{subfigure}[t]{0.32\textwidth}
    \centering
    \includegraphics[width=\linewidth]{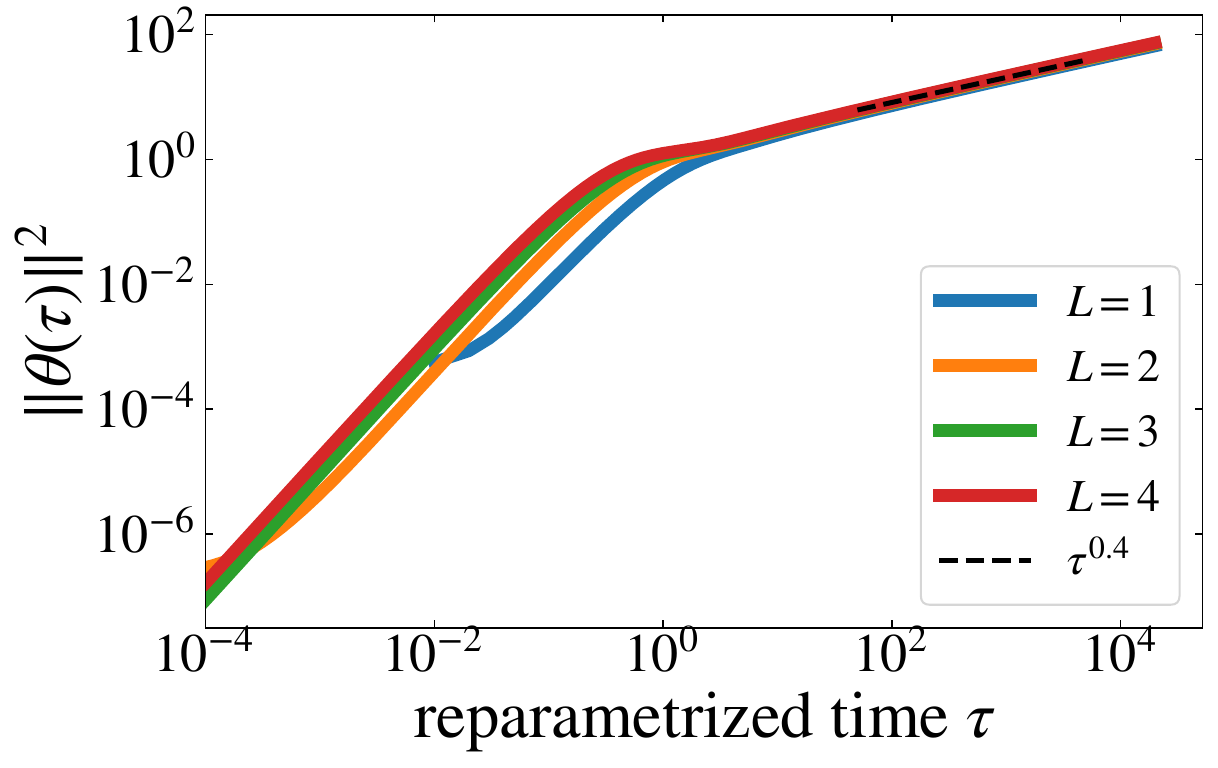}
    \caption{$p=0.6$}
  \end{subfigure}
  \hfill
  \begin{subfigure}[t]{0.32\textwidth}
    \centering
    \includegraphics[width=\linewidth]{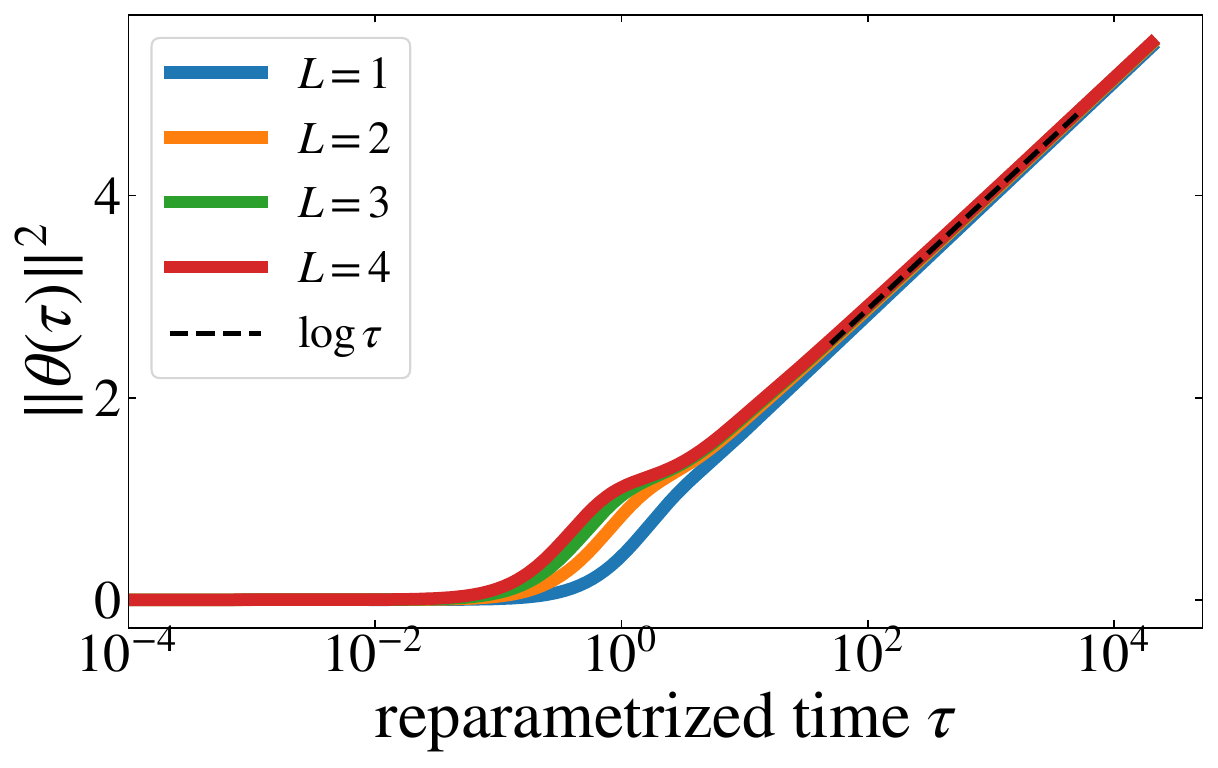}
    \caption{$p=1.0$}
  \end{subfigure}
  \hfill
  \begin{subfigure}[t]{0.32\textwidth}
    \centering
    \includegraphics[width=\linewidth]{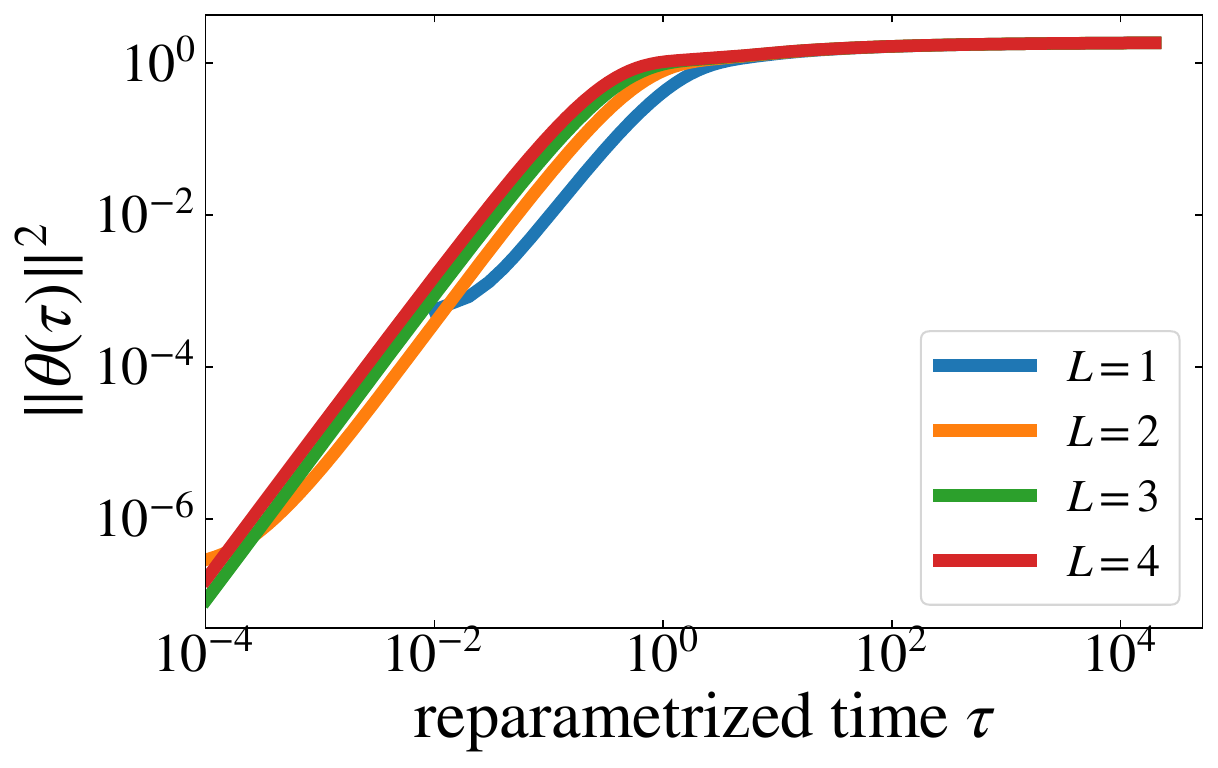}
    \caption{$p=1.4$}
  \end{subfigure}
  \caption{
    Growth of the squared predictor norm distinguishes three task regimes: polynomial growth for difficult tasks with $p<1$, logarithmic growth at the boundary $p=1$, and bounded behavior for easy tasks with $p>1$. Dashed lines indicate the asymptotic scaling laws given by \cref{prop:predictor-norm-scaling}. Panel (b) uses a linear vertical axis to make the logarithmic growth visible.
  }
  \label{fig:vanishing-predictor-norm-tau}
\end{figure}

\subsubsection{Common spectral law for the loss}
\label{sec:common-spectral-law-loss}

We next derive the loss scaling in reparameterized time $\tau$ and show that its scaling exponent depends on the task spectrum but is independent of network depth.
The same spectral threshold $\tau\lambda_i\asymp1$ that controls the predictor norm also determines which eigendirections contribute to the loss. Recall that $\mathcal L(\tau)=\frac12\sum_{i=1}^d\lambda_i(\beta_i-\theta_i(\tau))^2$. For $\tau\lambda_i\gg1$, the corresponding mode is essentially learned, with $\theta_i(\tau)\approx\beta_i$, whereas for $\tau\lambda_i\ll1$, it remains largely unlearned, with $\beta_i-\theta_i(\tau)\asymp\beta_i$. Thus, to leading order, the loss is governed by the unresolved spectral tail $\sum_{\lambda_i\lesssim1/\tau}\lambda_i\beta_i^2$. Under $\lambda_i=c_\lambda i^{-a}$ and $\beta_i=c_\beta i^{-b}$, this tail is governed by the exponent $p=(a+2b-1)/a$.
\begin{proposition}[Common spectral law for the loss, informal]
  \label{prop:common-spectral-law-loss}
  Under sufficiently small initialization, there exists a constant $\tau_0>0$, independent of $d$ and $s$, such that, for sufficiently large $d$, with high probability, uniformly for $\tau_0\le\tau\le\tau_+$,
  \begin{equation}
    \mathcal L(\tau)\asymp\tau^{-p}.
  \end{equation}
\end{proposition}
A rigorous proof is given in \cref{phased:prop:common-spectral-loss-law}.
Thus, in reparameterized time, the loss decays with the same power-law exponent $p$ across all depths.
\Cref{fig:reparameterized-time-loss} illustrates this common scaling numerically.

\begin{figure}[t]
  \centering
  \begin{subfigure}[t]{0.32\textwidth}
    \centering
    \includegraphics[width=\linewidth]{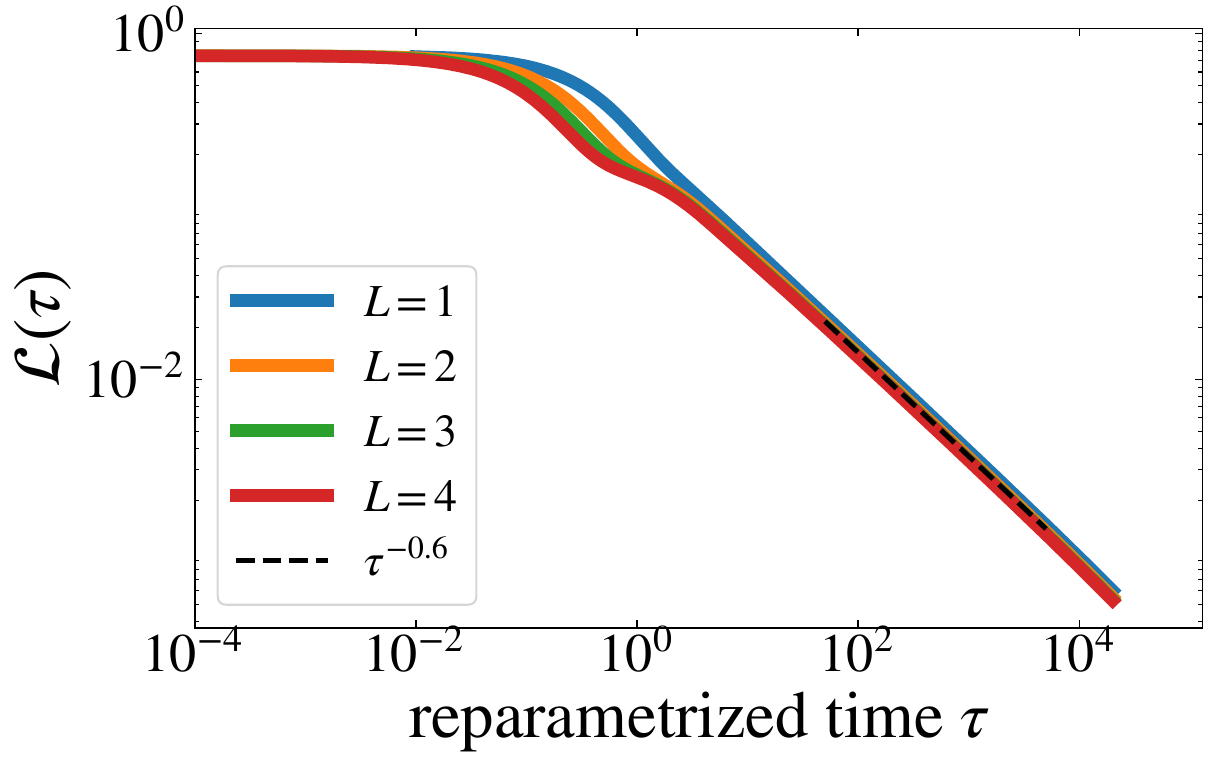}
    \caption{$p=0.6$}
  \end{subfigure}
  \hfill
  \begin{subfigure}[t]{0.32\textwidth}
    \centering
    \includegraphics[width=\linewidth]{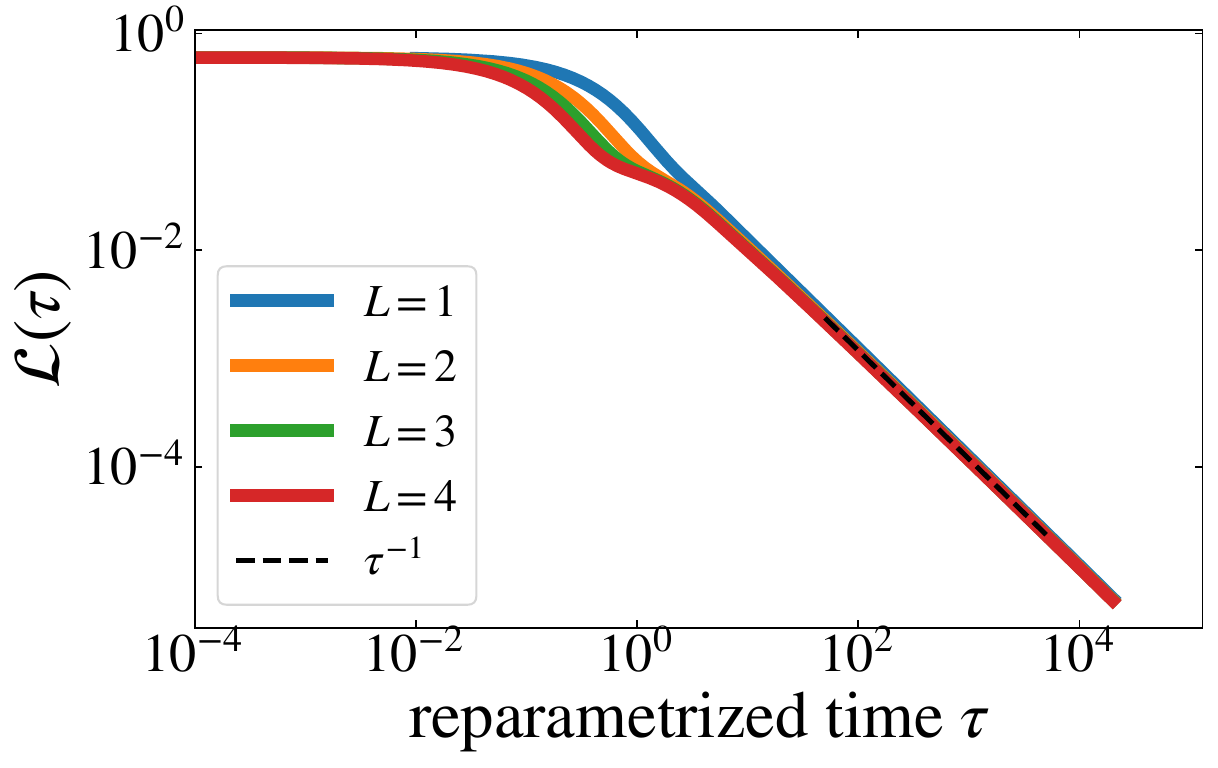}
    \caption{$p=1.0$}
  \end{subfigure}
  \hfill
  \begin{subfigure}[t]{0.32\textwidth}
    \centering
    \includegraphics[width=\linewidth]{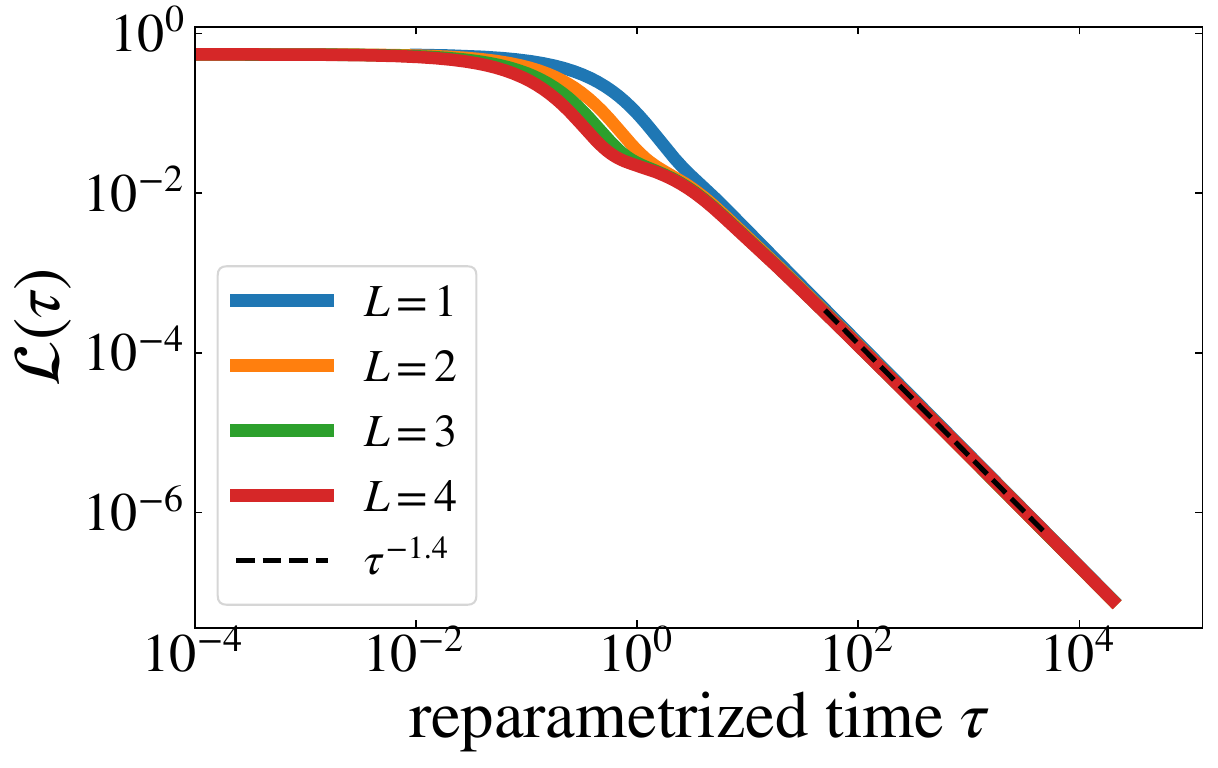}
    \caption{$p=1.4$}
  \end{subfigure}
  \caption{
    In reparameterized time, the loss follows a common depth-independent spectral scaling law, with exponent determined solely by the task spectrum.
    Dashed lines indicate the asymptotic reparameterized-time scaling laws given by
    \cref{prop:common-spectral-law-loss}.
  }
  \label{fig:reparameterized-time-loss}
\end{figure}

\subsubsection{Training-time loss scaling}
\label{sec:depth-dependence-loss}

We finally translate the loss scaling in reparameterized time $\tau$ back to the shifted training time $\bar t=t-t_{\mathrm e}$ and determine how the resulting training-time scaling depends on network depth. Since $d\tau/d\bar t=\sigma^{2(L-1)}$, the predictor-norm scaling in \cref{prop:predictor-norm-scaling} determines the relation between $\tau$ and $\bar t$. Combining this relation with the loss scaling in \cref{prop:common-spectral-law-loss} gives the following result.
\begin{theorem}[Training-time loss scaling, informal]\label{thm:training-time-loss-scaling}
  Under sufficiently small initialization, there exists a constant $\bar t_0>0$, independent of $d$ and $s$, such that, for sufficiently large $d$, with high probability, uniformly for $\bar t_0\le\bar t\le\bar t_+$,
  \begin{equation}
    \mathcal{L}(t_{\mathrm e}+\bar{t})\asymp
    \begin{cases}
      \bar{t}^{-pL/[1+p(L-1)]},              & p < 1, \\
      \bar{t}^{-1}(\log \bar{t})^{-(L-1)/L}, & p = 1, \\
      \bar{t}^{-p},                          & p > 1.
    \end{cases}
  \end{equation}
\end{theorem}
A rigorous proof is given in \cref{phased:thm:training-time-loss-scaling}. For $p<1$, the decay exponent $pL/[1+p(L-1)]$ increases with depth. For $p>1$, the power-law exponent is independent of depth, while at $p=1$, depth affects only the logarithmic correction. \Cref{fig:loss-decay-phase-diagram} summarizes the resulting dependence of the decay exponent on the task exponent $p$ and network depth.

\begin{figure}[t]
  \centering
  \begin{minipage}[c]{0.36\textwidth}
    \includegraphics[width=\linewidth]{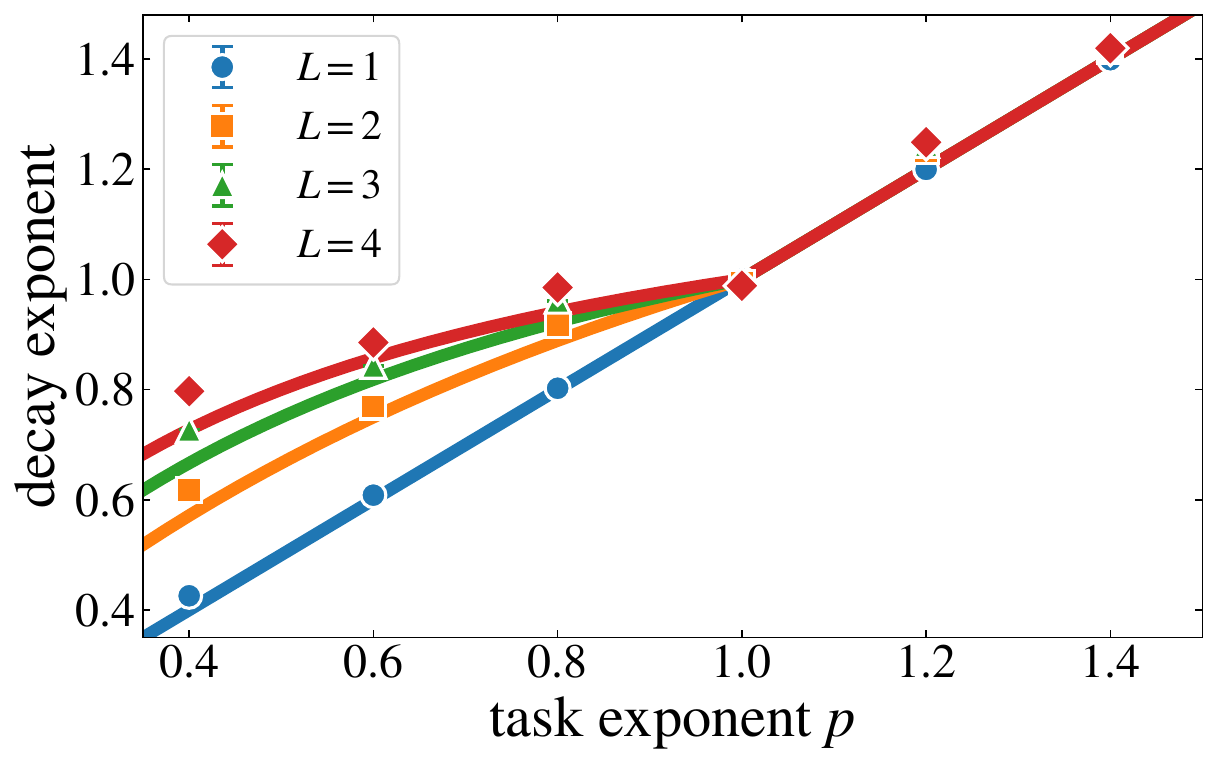}
  \end{minipage}\hfill
  \begin{minipage}[c]{0.62\textwidth}
    \caption{Task difficulty determines the magnitude of the depth-induced improvement in the training-time power-law exponent. Solid curves show the asymptotic exponents given by \cref{thm:training-time-loss-scaling}. Markers show numerical estimates averaged over five random seeds at dimension $d=2048$, and error bars show one standard deviation across seeds. At $p=1$, we account for the logarithmic correction by fitting $\mathcal{L}(\bar t)(\log \bar t)^{(L-1)/L}$ to a power law.}
    \label{fig:loss-decay-phase-diagram}
  \end{minipage}
\end{figure}

\section{Standard Initialization and Finite-Size Effects for Difficult Tasks}
\label{sec:standard-init-finite-size}

Having shown that depth improves the training-time scaling exponent for difficult tasks under sufficiently small initialization, we next ask whether this improvement persists under standard initialization. We focus on the difficult-task regime $p<1$, where the predictor norm grows polynomially during training. This growth makes the effect of the $O(1)$ initialization scale asymptotically negligible.
\begin{corollary}[Training-time loss scaling under standard initialization, informal]
  \label{cor:standard-init-training-time-loss-scaling}
  Suppose $p<1$ and consider standard initialization with $s=1$. There exists a constant $t_0>0$, independent of $d$, such that, for sufficiently large $d$, with high probability, uniformly for $t_0\le t\le t_+$,
  \begin{equation}
    \mathcal{L}(t)\asymp t^{-pL/[1+p(L-1)]}.
  \end{equation}
\end{corollary}
A rigorous proof is given in \cref{phased:cor:standard-init-training-time-loss-scaling}. \Cref{fig:standard-init-loss:loss-dyn} shows that the loss decays faster with increasing depth, consistent with the predicted depth dependence.

At finite input dimension, however, the predicted asymptotic scaling may not be clearly observable. \cref{fig:standard-init-loss:fitted-exponent} shows the fitted decay exponent as a function of the input dimension $d$. As $d$ increases, the fitted exponents approach their asymptotic predictions, but this convergence becomes slower with increasing depth. To understand this finite-size effect, recall that the power-law scaling window is restricted to $\tau\le\tau_+=c_0d^a$. Within this window, the largest asymptotic scale available to the squared predictor norm is $\tau_+^{1-p}\asymp d^{a(1-p)}=d^{1-2b}$. Although this quantity diverges as $d\to\infty$ for $p<1$, the divergence becomes slower as $p\uparrow1$. Moreover, the corresponding per-layer scale grows only as $\sigma(\tau_+)\asymp d^{a(1-p)/(2L)}=d^{(1-2b)/(2L)}$. This suggests that, as depth $L$ increases, the per-layer scale separates more slowly from its $O(1)$ initialization scale. Consequently, especially for large $L$ or for $p$ close to $1$, increasingly large dimensions may be required for the asymptotic power-law scaling to emerge clearly, making the ideal depth-dependent exponent difficult to observe at finite input dimension.

\begin{figure}[t]
  \centering
  \begin{subfigure}[t]{0.49\textwidth}
    \centering
    \includegraphics[width=0.71\linewidth]{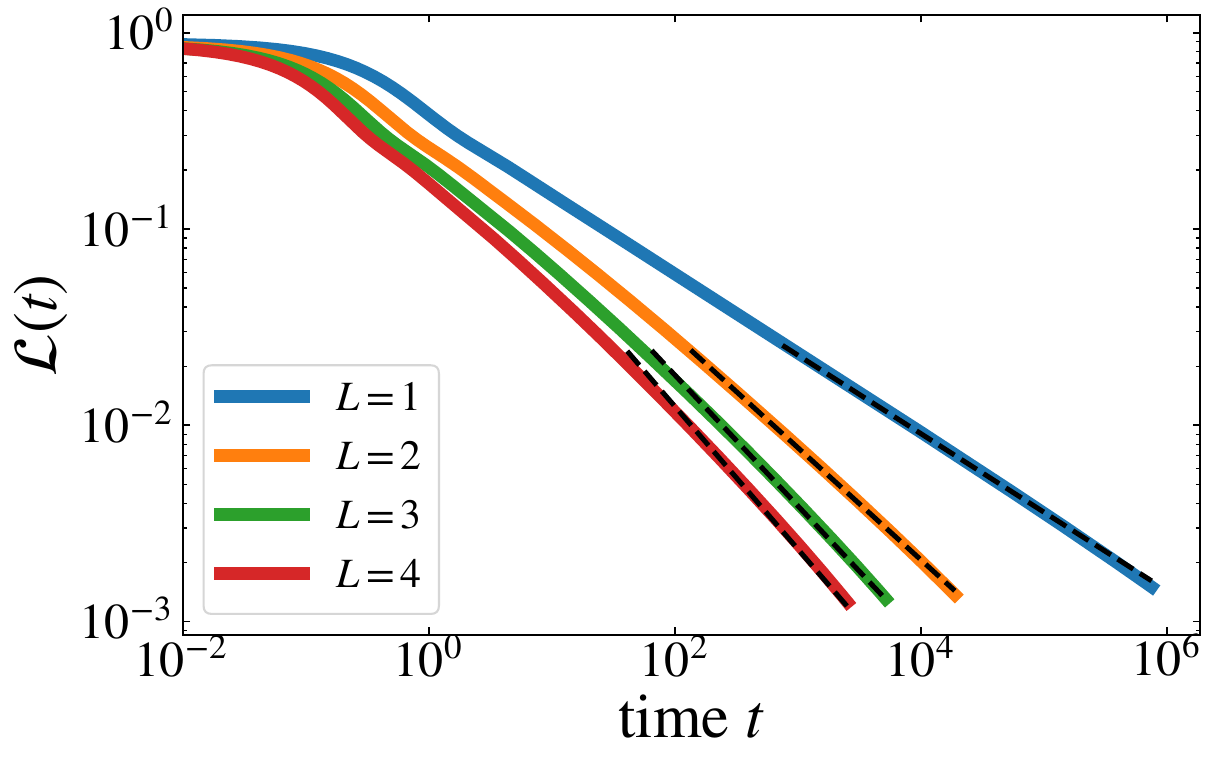}
    \caption{Loss dynamics}
    \label{fig:standard-init-loss:loss-dyn}
  \end{subfigure}\hfill
  \begin{subfigure}[t]{0.49\textwidth}
    \centering
    \includegraphics[width=0.71\linewidth]{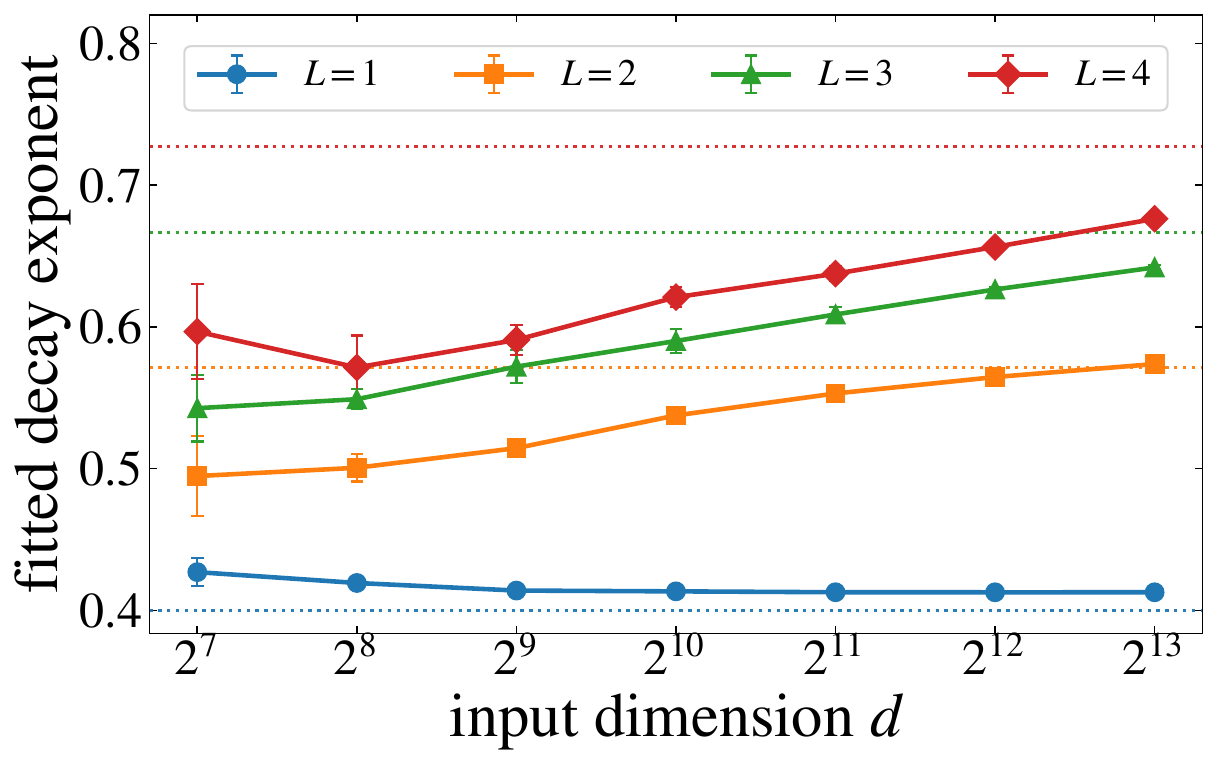}
    \caption{Finite-size effects on the fitted decay exponent}
    \label{fig:standard-init-loss:fitted-exponent}
  \end{subfigure}
  \caption{
    The depth-dependent asymptotic scaling under standard initialization ($s=1$) emerges with increasing dimension, while finite-size effects make this scaling harder to observe in deeper networks. (a) Loss dynamics for the $p=0.4$ task at dimension $d=8192$.
    Dashed lines show the asymptotic power-law decay predicted by \cref{cor:standard-init-training-time-loss-scaling}, with exponents corresponding to the asymptotic values shown by the dotted horizontal lines in panel (b).
    (b) As dimension increases, the fitted decay exponents approach these asymptotic values. Markers show the mean over five random seeds, and error bars show one standard deviation across seeds.}
  \label{fig:standard-initialization-loss}
\end{figure}

\vspace{-8pt}
\section{Experiments on Real Data and Nonlinear Networks}

Finally, we numerically examine whether the task-dependent effect of depth extends beyond the analytically tractable setting to real data and nonlinear networks.

First, we train deep linear networks with standard initialization ($s=1$) on a regression task constructed from CIFAR-5M \citep{nakkiran2020deep}, whose empirical task spectrum is approximately power-law with $p_{\mathrm{eff}}\approx0.478$, placing it in the $p<1$ regime. As shown in \cref{fig:cifar5m-real-label-dynamics}, the loss decay becomes steeper with increasing depth in training time. The loss trajectories nearly collapse under the predictor-norm time reparameterization, while the squared predictor norm exhibits approximately the predicted polynomial growth.

We also test ReLU MLPs on the synthetic power-law tasks used in our theoretical analysis. As shown in \cref{fig:relu-mlp-loss}, for the difficult $p=0.6$ task, the loss trajectories separate at late training times, with deeper networks attaining lower loss, whereas for the easy $p=1.4$ task the trajectories remain much closer and show no systematic advantage with increasing depth. These observations are qualitatively consistent with the task-dependent trend predicted by our theory, although the distinction is less pronounced than for the deep linear networks.

\begin{figure}[!t]
  \centering
  \begin{subfigure}[t]{0.32\textwidth}
    \centering
    \includegraphics[width=\linewidth]{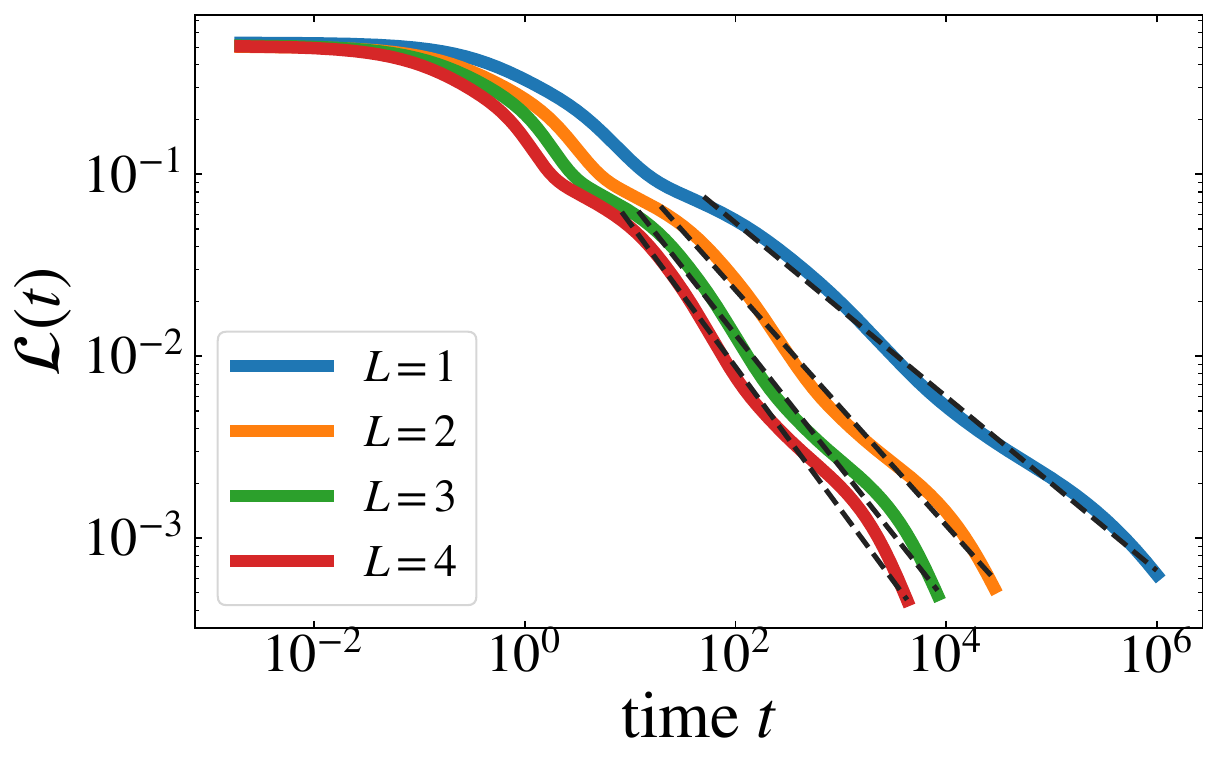}
    \caption{Loss in training time $t$}
  \end{subfigure}\hfill
  \begin{subfigure}[t]{0.32\textwidth}
    \centering
    \includegraphics[width=\linewidth]{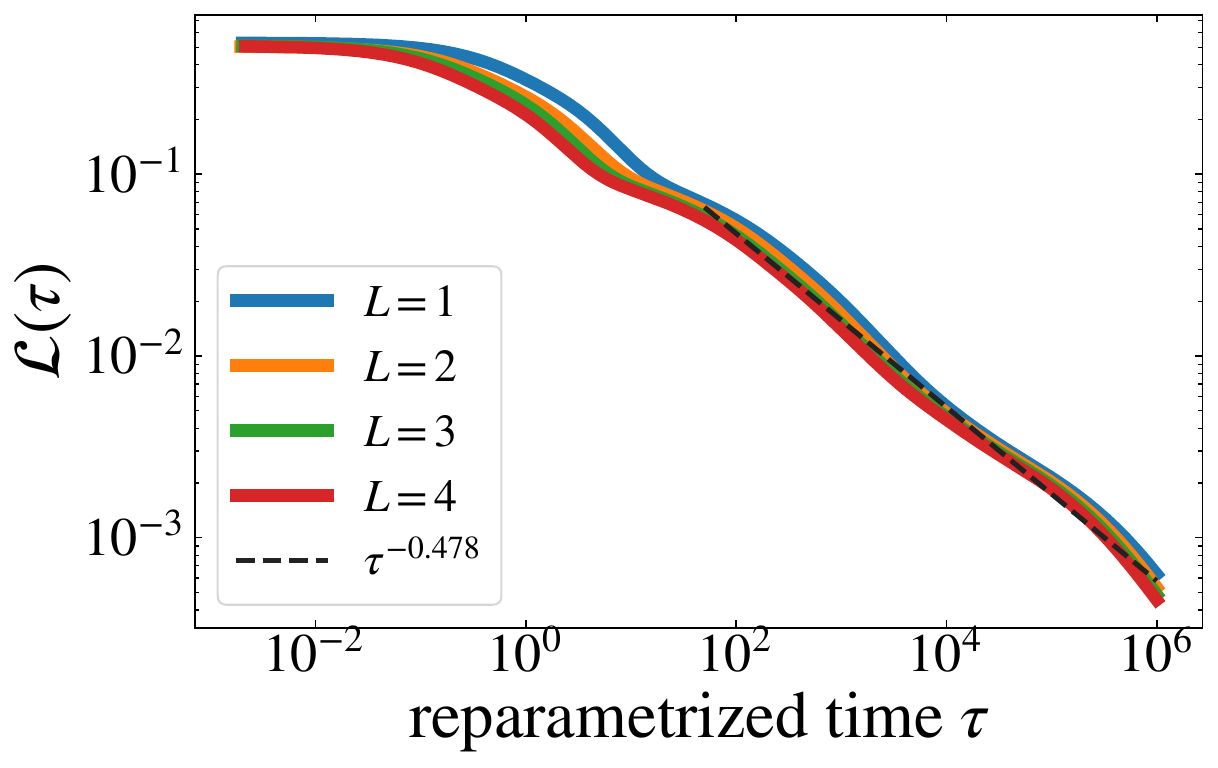}
    \caption{Loss in reparameterized time $\tau$}
  \end{subfigure}\hfill
  \begin{subfigure}[t]{0.32\textwidth}
    \centering
    \includegraphics[width=\linewidth]{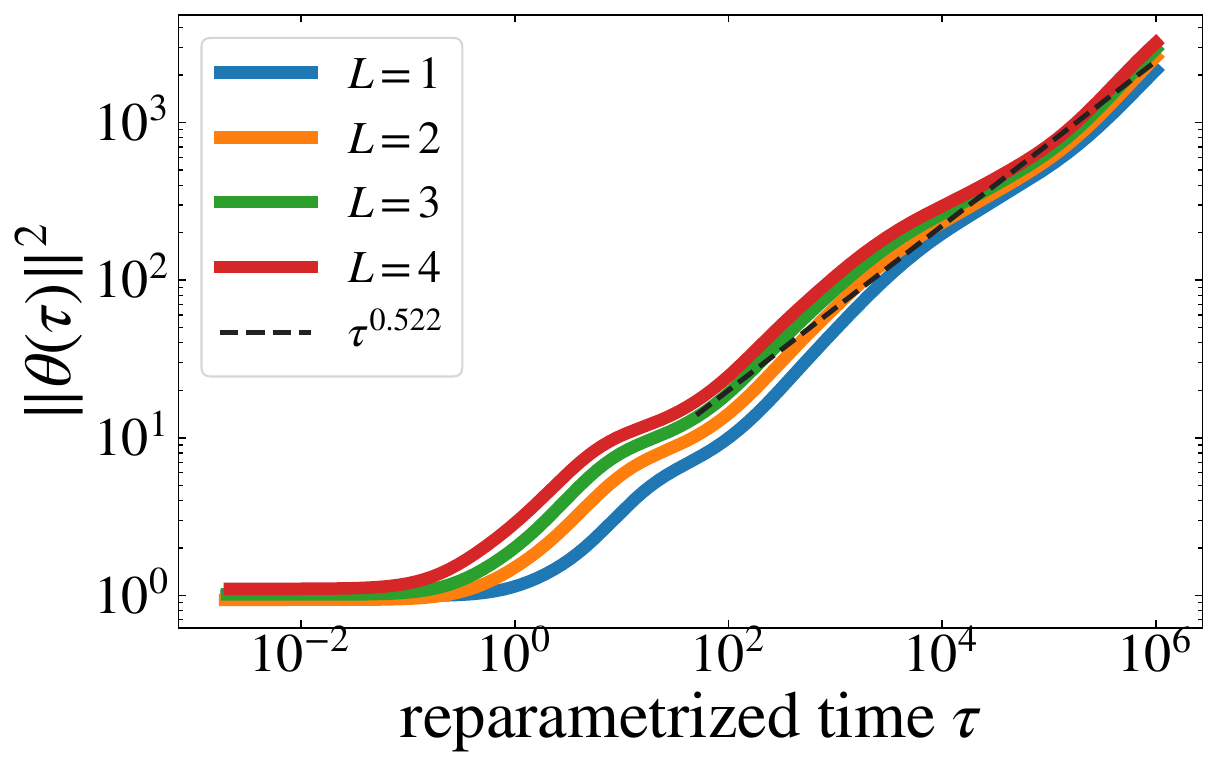}
    \caption{Growth of squared predictor norm}
  \end{subfigure}
  \caption{The depth-dependent acceleration persists on a task derived from benchmark image data.
  (a) In training time, the loss decays faster with increasing depth.
  (b) In reparameterized time, the loss trajectories nearly collapse onto a common spectral scaling law.
  (c) The squared predictor norm exhibits the predicted polynomial growth.
  Dashed lines show the theoretical slopes obtained from $p_{\mathrm{eff}}\approx0.478$:
  $p_{\mathrm{eff}}L/[1+p_{\mathrm{eff}}(L-1)]$ in (a), $p_{\mathrm{eff}}$ in (b), and $1-p_{\mathrm{eff}}$ in (c).
  }
  \label{fig:cifar5m-real-label-dynamics}
\end{figure}

\begin{figure}[t]
  \centering
  \begin{subfigure}[t]{0.49\textwidth}
    \centering
    \includegraphics[width=0.71\linewidth]{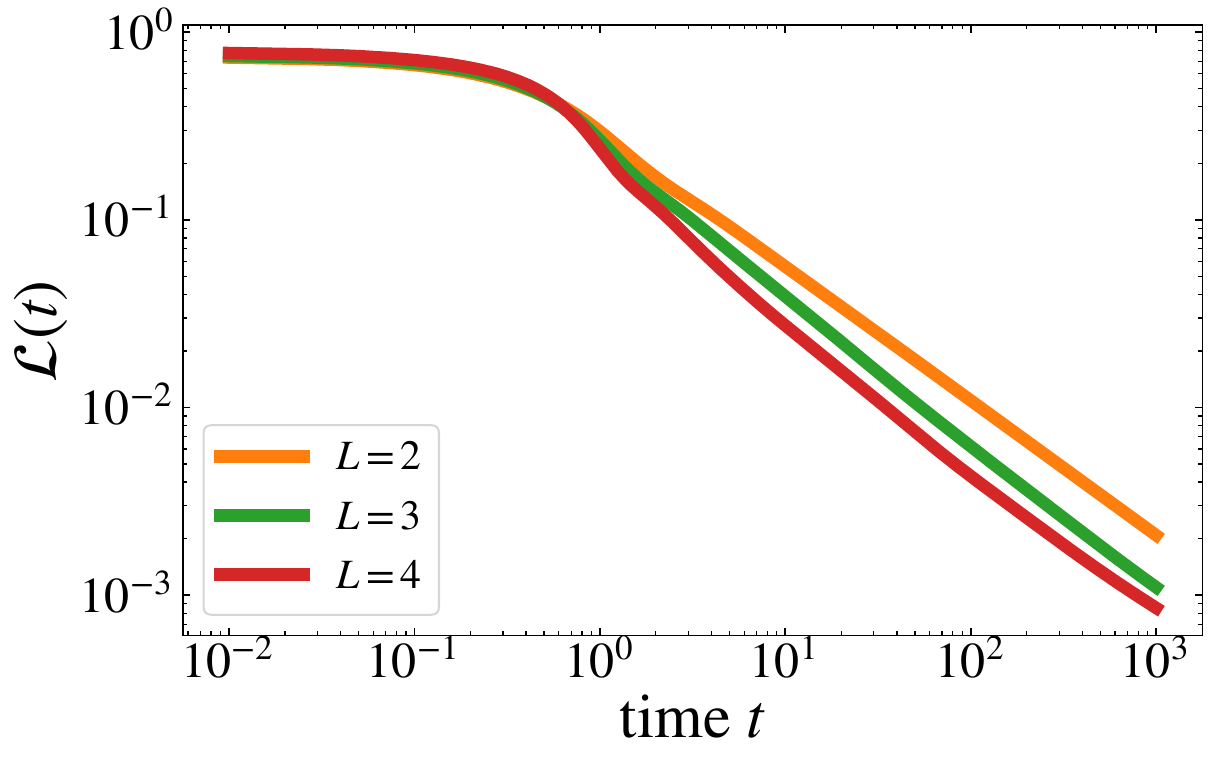}
    \caption{$p=0.6$}
  \end{subfigure}\hfill
  \begin{subfigure}[t]{0.49\textwidth}
    \centering
    \includegraphics[width=0.71\linewidth]{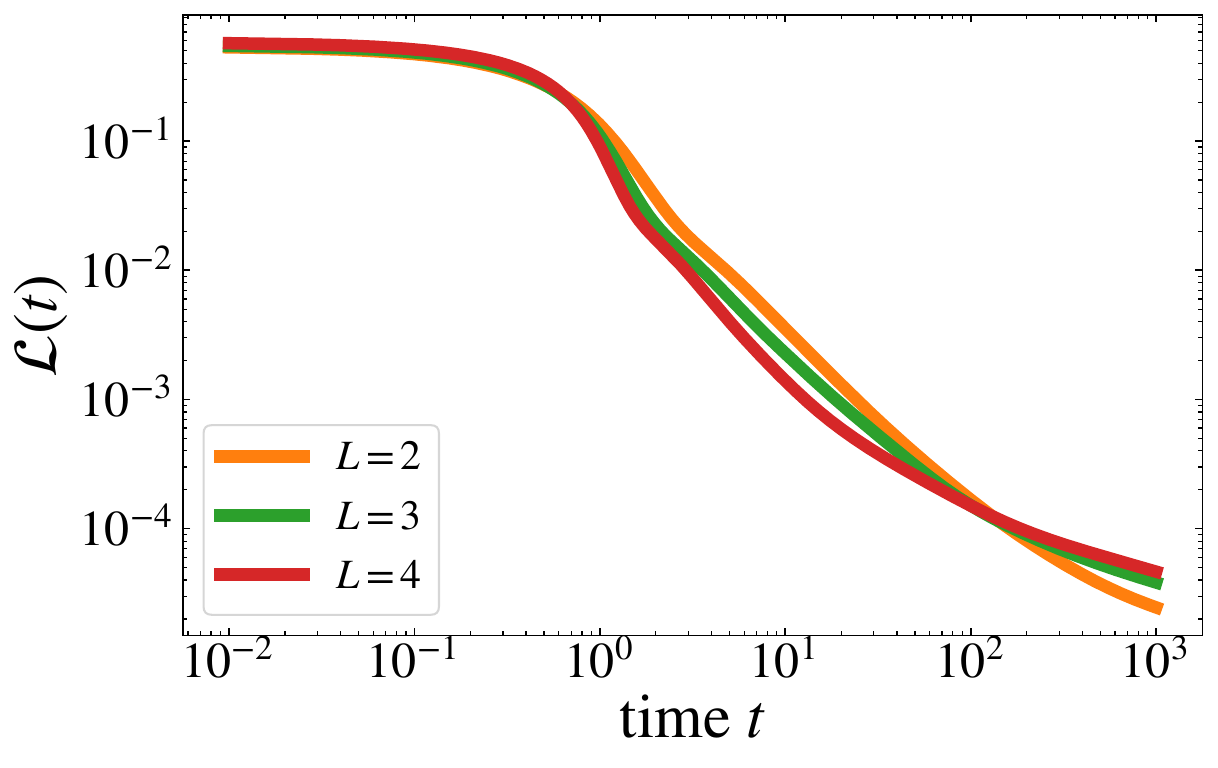}
    \caption{$p=1.4$}
  \end{subfigure}
  \caption{
    The task-dependent effect of depth persists qualitatively in nonlinear ReLU networks.
    (a) For the difficult task with $p=0.6$, deeper networks achieve lower loss at late training times.
    (b) For the easy task with $p=1.4$, the loss trajectories converge and cross, with no consistent late-time advantage from increasing depth.
  }
  \label{fig:relu-mlp-loss}
\end{figure}

\section{Conclusion}

In this work, we showed for the first time that task difficulty explains the previously unresolved depth dependence of convergence in the most fundamental setting of deep linear networks trained with MSE loss. That is, increasing network depth improves the asymptotic training-time scaling exponent for spectrally difficult tasks. A time reparameterization defined through the predictor norm reveals that the loss follows a depth-independent spectral scaling law in reparameterized time, while depth changes its relation to training time. This also explains why depth does not improve the power-law exponent for easy tasks.

Our analysis has several limitations. The theory assumes a linear teacher and student and population gradient flow, while our results for standard initialization are limited to difficult tasks. It remains to understand how the mechanisms identified here extend to more general task structures, nonlinear networks, and finite-sample training with discrete-time stochastic optimization.

Even in deep linear networks, explicitly accounting for task structure can substantially complicate a sharp characterization of the training dynamics. Our results show that this structure can qualitatively change how network depth affects training-time scaling. These results provide a basis for studying such architectural effects in broader classes of neural networks.

\subsection*{AI use statement}

In this work, we used generative AI tools to help refine and formalize aspects of the theoretical framework and presentation, to refine the formulation of some mathematical statements and technical assumptions, and to assist with writing and formalizing proofs. The central research hypotheses, core mathematical ideas, and main theoretical mechanisms were developed by the authors. We also used generative AI tools to generate and edit code implementing numerical experiments, fitting procedures, and plotting routines, and for translation. The experimental settings and comparisons were specified by the authors.
Additionally, we used generative AI tools to suggest some experimental parameters, to draft and edit parts of the manuscript, and to identify relevant literature and examine specific mathematical arguments in prior work.
The authors reviewed and verified all AI-assisted text, mathematical claims and proofs, code, experimental results, analyses, and citations, and take responsibility for the final content.

\subsection*{Reproducibility statement}

The assumptions, technical derivations, and complete proofs for the theoretical results are provided in \cref{sec:gradient-flow-preconditioner-initialization,phased:sec:escape-from-origin,sec:spectral-estimates-auxiliary-loss-bounds,phased:sec:radius-law,sec:common-spectral-law-for-loss,phased:sec:main-scaling-proofs}.
Details for reproducing the numerical experiments are provided in \cref{sec:numerical-experiments-details}, including experimental setups, data processing, optimization settings, and fitting procedures.

\subsubsection*{Acknowledgments}
T.W. was supported by JST BOOST (Grant Number JPMJBS2407). R.K. was supported by JST FOREST (Grant Number JPMJFR226Q), JST BOOST (Grant Number JPMJBY25B0), JSPS KAKENHI (Grant Numbers JP23K16965 and JP22H05116), and FRONTia, a Japanese national program led by METI and NEDO. J.T. was supported by JSPS KAKENHI (Grant Number JP24K15104).

\bibliography{iclr2027_conference}
\bibliographystyle{iclr2027_conference}

\appendix

\newpage

\tableofcontents

Throughout the appendices, $c,C>0$ denote generic constants that may change from line to line. Unless otherwise specified, they do not depend on $d$ or on the variables quantified in the corresponding estimate.

\section{Gradient-Flow Dynamics, Preconditioner Geometry, and Initialization Bounds}
\label{sec:gradient-flow-preconditioner-initialization}

This section establishes the gradient-flow identities, preconditioner estimates, and Gaussian initialization bounds used throughout the analysis.

\subsection{Exact Gradient-Flow Dynamics}

We write $\boldsymbol g:=\boldsymbol\Sigma(\boldsymbol\beta-\boldsymbol\theta)$ and $r:=\|\boldsymbol\theta\|=\sigma^L$. On any interval on which $r>0$, write $\widehat{\boldsymbol\theta}:=\boldsymbol\theta/r$.

The preconditioner induced by the factorization is
\begin{equation}
  \boldsymbol K_L
  :=
  \sum_{k=1}^L
  \|\boldsymbol W_{L:k+1}\|^2
  \boldsymbol W_{k-1:1}^\top\boldsymbol W_{k-1:1}.
  \label{commonwork:eq:tangent-kernel}
\end{equation}
For every $1\le k\le L$, using $\nabla_{\boldsymbol\theta}\mathcal L=-\boldsymbol g$ and applying the chain rule to the factorization $\boldsymbol\theta=(\boldsymbol W_{L:k+1}\boldsymbol W_k \boldsymbol W_{k-1:1})^\top$, we obtain
\begin{equation}
  \frac{d}{dt}\boldsymbol W_k
  =\boldsymbol W_{L:k+1}^\top\boldsymbol g^\top
  \boldsymbol W_{k-1:1}^\top.
  \label{commonwork:eq:exact-weight-dynamics}
\end{equation}
Differentiating the layer product and using \cref{commonwork:eq:exact-weight-dynamics} yields
\begin{equation}
  \frac{d}{dt}\boldsymbol\theta
  =\sum_{k=1}^L
  \|\boldsymbol W_{L:k+1}\|^2
  \boldsymbol W_{k-1:1}^\top\boldsymbol W_{k-1:1}\boldsymbol g
  =\boldsymbol K_L\boldsymbol g.
  \label{commonwork:eq:predictor-dynamics}
\end{equation}
Together with $\nabla_{\boldsymbol\theta}\mathcal L=-\boldsymbol g$, this gives
\begin{equation}
  \frac{d}{dt}\mathcal L
  =-\boldsymbol g^\top\boldsymbol K_L\boldsymbol g
  =-\sum_{k=1}^L
  \left\|\frac{d}{dt}\boldsymbol W_k\right\|_{\mathrm F}^2.
  \label{commonwork:eq:loss-dissipation-expanded}
\end{equation}
The last equality follows by taking the squared Frobenius norm of both sides of \cref{commonwork:eq:exact-weight-dynamics} and summing over the layers.

On such an interval, define the normalized preconditioner by
\begin{equation}
  \overline{\boldsymbol K}_L
  :=\sigma^{-2(L-1)}\boldsymbol K_L.
  \label{commonwork:eq:normalized-tangent-kernel}
\end{equation}
Since $d\tau/dt=\sigma^{2(L-1)}>0$ on this interval, reparameterizing the predictor and loss identities in terms of $\tau$ gives
\begin{equation}
  \frac{d}{d\tau}\boldsymbol\theta
  =\overline{\boldsymbol K}_L\boldsymbol g,
  \label{commonwork:eq:tau-predictor-dynamics}
\end{equation}
\begin{equation}
  \frac{d}{d\tau}\mathcal L
  =-\boldsymbol g^\top\overline{\boldsymbol K}_L\boldsymbol g.
  \label{commonwork:eq:tau-loss-dynamics}
\end{equation}

\begin{lemma}
  \label{commonwork:lem:global-existence}
  For every finite $d$ and any initialization, the gradient-flow system in \cref{commonwork:eq:exact-weight-dynamics} admits a unique solution for all $t\ge0$.
\end{lemma}

\begin{proof}
  The vector field is polynomial in the weights and hence locally Lipschitz. Therefore, the system admits a unique local solution. Let $[0,T_{\max})$ be the maximal interval of existence of this solution. For every $T<T_{\max}$, the Cauchy--Schwarz inequality and \cref{commonwork:eq:loss-dissipation-expanded} give
  \begin{equation}
    \begin{split}
      \left(\sum_{k=1}^L
      \|\boldsymbol W_k(T)-\boldsymbol W_k(0)\|_{\mathrm F}^2\right)^{1/2}
      & \le \int_0^T\left(\sum_{k=1}^L
      \left\|\frac{d}{dt}\boldsymbol W_k\right\|_{\mathrm F}^2
      \right)^{1/2}dt                                \\
      & \le \sqrt{T}
      \left(\mathcal L(0)-\mathcal L(T)\right)^{1/2} \\
      & \le \sqrt{T\mathcal L(0)}.
    \end{split}
    \label{commonwork:eq:finite-time-displacement}
  \end{equation}
  If $T_{\max}<\infty$, then \cref{commonwork:eq:finite-time-displacement} shows that every weight matrix remains bounded on $[0,T_{\max})$, contradicting the standard continuation criterion for ODEs. Therefore $T_{\max}=\infty$.
\end{proof}

\subsection{Layer Balance and Preconditioner Geometry}

To quantify the deviation from layer balance at initialization, define
\begin{equation}
  \varepsilon
  :=\max_{1\le k<L}
  \left\|
  \boldsymbol W_k(0)\boldsymbol W_k(0)^\top
  -\boldsymbol W_{k+1}(0)^\top\boldsymbol W_{k+1}(0)
  \right\|_{\mathrm{op}}.
  \label{commonwork:eq:initial-gram-defect}
\end{equation}

\begin{lemma}
  \label{commonwork:lem:conserved-gram}
  For every $1\le k<L$ and every $t\ge0$,
  \begin{equation}
    \boldsymbol W_k(t)\boldsymbol W_k(t)^\top
    -\boldsymbol W_{k+1}(t)^\top\boldsymbol W_{k+1}(t)
    =
    \boldsymbol W_k(0)\boldsymbol W_k(0)^\top
    -\boldsymbol W_{k+1}(0)^\top\boldsymbol W_{k+1}(0).
    \label{commonwork:eq:conserved-gram}
  \end{equation}
  Consequently, at every $t\ge0$,
  \begin{equation}
    \max_{1\le k<L}
    \left\|
    \boldsymbol W_k(t)\boldsymbol W_k(t)^\top
    -\boldsymbol W_{k+1}(t)^\top\boldsymbol W_{k+1}(t)
    \right\|_{\mathrm{op}}
    =\varepsilon.
    \label{commonwork:eq:conserved-gram-defect}
  \end{equation}
\end{lemma}

\begin{proof}
  Using \cref{commonwork:eq:exact-weight-dynamics} together with
  $\boldsymbol W_{k:1}=\boldsymbol W_k\boldsymbol W_{k-1:1}$ and
  $\boldsymbol W_{L:k+1}=\boldsymbol W_{L:k+2}\boldsymbol W_{k+1}$ gives
  \begin{equation}
    \frac{d}{dt}
    (\boldsymbol W_k\boldsymbol W_k^\top)
    =\boldsymbol W_{L:k+1}^\top\boldsymbol g^\top
    \boldsymbol W_{k:1}^\top
    +\boldsymbol W_{k:1}\boldsymbol g\boldsymbol W_{L:k+1}
    =\frac{d}{dt}
    (\boldsymbol W_{k+1}^\top\boldsymbol W_{k+1}).
  \end{equation}
  Hence the difference in \cref{commonwork:eq:conserved-gram} has zero derivative. Evaluating it at $t=0$ yields both claims.
\end{proof}

Define
\begin{equation}
  \Delta\overline{\boldsymbol K}_L
  :=\overline{\boldsymbol K}_L
  -\boldsymbol I-(L-1)\boldsymbol P_{\boldsymbol\theta}.
  \label{commonwork:eq:normalized-kernel-defect}
\end{equation}

Suppose that $r>0$. Set $\boldsymbol\rho_L=1$ and, for $k=L,L-1,\ldots,1$, recursively define
\begin{equation}
  \sigma_k:=\|\boldsymbol W_k^\top\boldsymbol\rho_k\|,\quad
  \boldsymbol\rho_{k-1}
  :=\frac{\boldsymbol W_k^\top\boldsymbol\rho_k}{\sigma_k},\quad
  \boldsymbol W_k^\perp
  :=(\boldsymbol I-\boldsymbol\rho_k\boldsymbol\rho_k^\top)
  \boldsymbol W_k.
\end{equation}
A backward induction shows simultaneously that this recursion is well defined and that
\begin{equation}
  \boldsymbol W_{L:k+1}
  =\left(\prod_{j=k+1}^L\sigma_j\right)\boldsymbol\rho_k^\top,
  \label{commonwork:eq:backward-suffix-product}
\end{equation}
for $0\le k\le L$.
Indeed, the identity holds at $k=L$ by the convention $\boldsymbol W_{L:L+1}=1$. If it holds at $k$ and $\sigma_k=0$, then $\boldsymbol W_{L:k}=0$ and hence $\boldsymbol W_{L:1}=0$, contradicting $r>0$. Thus $\sigma_k>0$, and the defining relation $\boldsymbol W_k^\top\boldsymbol\rho_k =\sigma_k\boldsymbol\rho_{k-1}$ yields the identity at $k-1$.

By construction, each $\boldsymbol\rho_k$ is a unit vector and
\begin{equation}
  \boldsymbol{W}_k
  = \sigma_k \boldsymbol{\rho}_k \boldsymbol{\rho}_{k-1}^\top
  + \boldsymbol{W}_k^\perp,
  \qquad
  \boldsymbol{\rho}_k^\top \boldsymbol{W}_k^\perp = 0.
  \label{commonwork:eq:backward-layer-decomposition}
\end{equation}
Taking $k=0$ in \cref{commonwork:eq:backward-suffix-product} and comparing with $\boldsymbol W_{L:1}=\boldsymbol\theta^\top =r\widehat{\boldsymbol\theta}^\top$ gives
\begin{equation}
  \prod_{k=1}^L\sigma_k=r=\sigma^L,
  \qquad
  \boldsymbol\rho_0=\widehat{\boldsymbol\theta}.
  \label{commonwork:eq:backward-product-direction}
\end{equation}

\begin{lemma}
  \label{commonwork:lem:balanced-layer-amplitudes}
  There exists a constant $C>0$, depending only on $L$, such that at every time for which $r>0$, the decomposition in \cref{commonwork:eq:backward-layer-decomposition} satisfies
  \begin{equation}
    \max_{1\le k\le L}\|\boldsymbol W_k^\perp\|_{\mathrm{op}}^2
    \le C\varepsilon,
    \qquad
    \max_{1\le k\le L}|\sigma_k^2-\sigma^2|
    \le C\varepsilon.
  \end{equation}
  Consequently, for any fixed $M>0$, if $\varepsilon\le M\sigma^2$, then
  \begin{equation}
    \sigma_k\asymp\sigma,
  \end{equation}
  for $1\le k\le L$, where the comparison constants depend only on $L$ and $M$.
\end{lemma}

\begin{proof}
  The orthogonality in the decomposition gives
  \begin{equation}
    \boldsymbol W_{k+1}^\top\boldsymbol W_{k+1}
    =\sigma_{k+1}^2\boldsymbol\rho_k\boldsymbol\rho_k^\top
    +(\boldsymbol W_{k+1}^\perp)^\top
    \boldsymbol W_{k+1}^\perp.
  \end{equation}
  Taking the quadratic form of the adjacent Gram difference along $\boldsymbol\rho_k$ and using \cref{commonwork:eq:conserved-gram-defect}, we obtain
  \begin{equation}
    |\sigma_k^2-\sigma_{k+1}^2|
    \le \|\boldsymbol W_{k+1}^\perp\|_{\mathrm{op}}^2+\varepsilon.
  \end{equation}
  Projecting the same Gram difference onto the orthogonal complement of $\boldsymbol\rho_k$ gives
  \begin{equation}
    \|\boldsymbol W_k^\perp\|_{\mathrm{op}}^2
    \le \|\boldsymbol W_{k+1}^\perp\|_{\mathrm{op}}^2+\varepsilon.
  \end{equation}
  Since $\boldsymbol W_L^\perp=0$, backward induction shows that
  \begin{equation}
    \|\boldsymbol W_k^\perp\|_{\mathrm{op}}^2
    \le C\varepsilon,
  \end{equation}
  for $1\le k\le L$.
  Combining $|\sigma_k^2-\sigma_{k+1}^2|\le \|\boldsymbol W_{k+1}^\perp\|_{\mathrm{op}}^2+\varepsilon$ with the bounds on $\|\boldsymbol W_k^\perp\|_{\mathrm{op}}^2$ above shows that the values $\sigma_1^2,\ldots,\sigma_L^2$ differ pairwise by at most $C\varepsilon$. Their geometric mean is $\sigma^2$ by \cref{commonwork:eq:backward-product-direction}, so
  \begin{equation}
    |\sigma_k^2-\sigma^2|
    \le C\varepsilon,
  \end{equation}
  for $1\le k\le L$.

  If $\varepsilon\le M\sigma^2$, this bound gives
  \begin{equation}
    \sigma_k^2
    \le \sigma^2+C\varepsilon
    \le (1+CM)\sigma^2.
  \end{equation}
  Using \cref{commonwork:eq:backward-product-direction} and applying this upper bound to all factors other than $\sigma_k^2$ gives
  \begin{equation}
    \sigma_k^2
    =\frac{\sigma^{2L}}{\prod_{j\ne k}\sigma_j^2}
    \ge (1+CM)^{-(L-1)}\sigma^2.
  \end{equation}
  These two bounds prove the claimed comparison.
\end{proof}

\begin{proposition}
  \label{commonwork:prop:perturbative-kernel-expansion}
  Fix $M>0$, and suppose that $r>0$ and $\varepsilon\le M\sigma^2$. Then there exist constants $c,C>0$, depending only on $L$ and $M$, such that
  \begin{equation}
    \|\Delta\overline{\boldsymbol K}_L\|_{\mathrm{op}}
    \le C\frac{\varepsilon}{\sigma^2}
    \label{commonwork:eq:perturbative-kernel-normal-form}
  \end{equation}
  and
  \begin{equation}
    c\boldsymbol I\preceq\overline{\boldsymbol K}_L\preceq C\boldsymbol I.
    \label{commonwork:eq:perturbative-kernel-sandwich}
  \end{equation}
\end{proposition}

\begin{proof}
  Fix a time satisfying the assumptions of the proposition. By \cref{commonwork:lem:balanced-layer-amplitudes}, for every $1\le k\le L$, the layer amplitudes are comparable to $\sigma$, while $|\sigma_k^2-\sigma^2|\le C\varepsilon$ and $\|\boldsymbol W_k^\perp\|_{\mathrm{op}}^2\le C\varepsilon$.
  Together with the decomposition in \cref{commonwork:eq:backward-layer-decomposition}, these bounds imply
  \begin{equation}
    \|\boldsymbol{W}_k\|_{\mathrm{op}}^2\le C\sigma^2,
    \label{commonwork:eq:perturbative-layer-upper}
  \end{equation}
  for $1\le k\le L$.

  For $1\le j<L$, define
  \begin{equation}
    \boldsymbol{E}_j
    :=\boldsymbol{W}_{j-1:1}^\top(\boldsymbol{W}_j^\perp)^\top
    \boldsymbol{W}_j^\perp\boldsymbol{W}_{j-1:1}
    \succeq0.
    \label{commonwork:eq:perturbative-prefix-remainder}
  \end{equation}
  By \cref{commonwork:lem:balanced-layer-amplitudes} and \cref{commonwork:eq:perturbative-layer-upper},
  \begin{equation}
    \|\boldsymbol{E}_j\|_{\mathrm{op}}
    \le \|\boldsymbol{W}_{j-1:1}\|_{\mathrm{op}}^2
    \|\boldsymbol{W}_j^\perp\|_{\mathrm{op}}^2
    \le C\sigma^{2(j-1)}\varepsilon.
    \label{commonwork:eq:perturbative-prefix-remainder-bound}
  \end{equation}
  The backward recursion gives
  \begin{equation}
    \boldsymbol{W}_{j-1:1}^\top\boldsymbol{\rho}_{j-1}
    =\left(\prod_{k=1}^{j-1}\sigma_k\right)\boldsymbol{\rho}_0.
    \label{commonwork:eq:perturbative-prefix-direction}
  \end{equation}
  Combining this identity with the orthogonality of the transverse component yields
  \begin{equation}
    \boldsymbol{W}_{j:1}^\top\boldsymbol{W}_{j:1}
    =\left(\prod_{k=1}^j\sigma_k^2\right)
    \boldsymbol{\rho}_0\boldsymbol{\rho}_0^\top+\boldsymbol{E}_j.
    \label{commonwork:eq:perturbative-prefix-gram}
  \end{equation}
  Substituting \cref{commonwork:eq:backward-suffix-product} and \cref{commonwork:eq:perturbative-prefix-gram} into \cref{commonwork:eq:tangent-kernel}, and then using \cref{commonwork:eq:backward-product-direction}, yields
  \begin{equation}
    \overline{\boldsymbol K}_L
    =\frac{\sigma^2}{\sigma_1^2}\boldsymbol{I}_d
    +\sum_{k=2}^L\frac{\sigma^2}{\sigma_k^2}
    \boldsymbol P_{\boldsymbol\theta}+
    \sum_{k=2}^L
    \sigma^{-2(L-1)}
    \left(\prod_{j=k+1}^L\sigma_j^2\right)
    \boldsymbol{E}_{k-1}.
    \label{commonwork:eq:perturbative-kernel-expansion}
  \end{equation}
  For every $1\le k\le L$, the amplitude bounds imply
  \begin{equation}
    \left|\frac{\sigma^2}{\sigma_k^2}-1\right|
    =\frac{|\sigma^2-\sigma_k^2|}{\sigma_k^2}
    \le C\frac{\varepsilon}{\sigma^2}.
    \label{commonwork:eq:perturbative-kernel-coefficient-error}
  \end{equation}
  For every $2\le k\le L$, \cref{commonwork:lem:balanced-layer-amplitudes,commonwork:eq:perturbative-prefix-remainder-bound} give
  \begin{equation}
    \left\|
    \sigma^{-2(L-1)}
    \left(\prod_{j=k+1}^L\sigma_j^2\right)
    \boldsymbol{E}_{k-1}
    \right\|_{\mathrm{op}}
    \le C\frac{\varepsilon}{\sigma^2}.
    \label{commonwork:eq:perturbative-kernel-remainder-summand}
  \end{equation}
  Subtracting $\boldsymbol I_d+(L-1)\boldsymbol P_{\boldsymbol\theta}$ from \cref{commonwork:eq:perturbative-kernel-expansion} and bounding each resulting term by \cref{commonwork:eq:perturbative-kernel-coefficient-error,commonwork:eq:perturbative-kernel-remainder-summand}, we obtain
  \begin{equation}
    \|\Delta\overline{\boldsymbol K}_L\|_{\mathrm{op}}
    \le C\frac{\varepsilon}{\sigma^2}.
  \end{equation}
  This establishes \cref{commonwork:eq:perturbative-kernel-normal-form}.

  Finally, all terms on the right-hand side of \cref{commonwork:eq:perturbative-kernel-expansion} are positive semidefinite. Thus the amplitude comparison in \cref{commonwork:lem:balanced-layer-amplitudes}, together with \cref{commonwork:eq:perturbative-kernel-expansion}, gives $\overline{\boldsymbol K}_L\succeq c\boldsymbol I$. The same comparison, \cref{commonwork:eq:perturbative-kernel-expansion,commonwork:eq:perturbative-kernel-remainder-summand}, and $\varepsilon/\sigma^2\le M$ give $\overline{\boldsymbol K}_L\preceq C\boldsymbol I$.
\end{proof}

\subsection{Rescaled Predictor Dynamics and Time Reparameterization}
\label{commonwork:ss:rescaled-predictor-dynamics}

We prove \cref{phased:lem:rescaled-predictor-dynamics,phased:cor:rescaled-predictor-tau-dynamics} and record the resulting variation of constants formula.
Throughout this subsection, we work on an interval on which $r>0$.

\begin{proof}[Proof of \Cref{phased:lem:rescaled-predictor-dynamics}]
  Recall that $r=\sigma^L$,
  \begin{equation}
    \boldsymbol z = \sigma^{-(L-1)}\boldsymbol\theta,
    \label{phased:eq:rescaled-predictor}
  \end{equation}
  and
  \begin{equation}
    \boldsymbol J_L =
    \sigma^{L-1}\frac{\partial\boldsymbol z}{\partial\boldsymbol\theta}
    = \boldsymbol I - \frac{L-1}{L}
    \boldsymbol P_{\boldsymbol\theta}.
    \label{phased:eq:rescaling-operator}
  \end{equation}
  By \cref{commonwork:eq:normalized-kernel-defect}, $\overline{\boldsymbol K}_L =\boldsymbol I+(L-1)\boldsymbol P_{\boldsymbol\theta}+\Delta\overline{\boldsymbol K}_L$.
  Since $\boldsymbol P_{\boldsymbol\theta}^2=\boldsymbol P_{\boldsymbol\theta}$, we have $\boldsymbol J_L(\boldsymbol I+(L-1)\boldsymbol P_{\boldsymbol\theta})=\boldsymbol I$.
  Moreover, $\partial\boldsymbol z/\partial\boldsymbol\theta =\sigma^{-(L-1)}\boldsymbol J_L$.
  By \cref{commonwork:eq:predictor-dynamics,commonwork:eq:normalized-tangent-kernel}, we have
  \begin{equation}
    \frac{d\boldsymbol\theta}{dt}
    =\boldsymbol K_L\boldsymbol g
    =\sigma^{2(L-1)}\overline{\boldsymbol K}_L\boldsymbol g.
  \end{equation}
  We therefore obtain
  \begin{equation}
    \begin{split}
      \frac{d\boldsymbol z}{dt}
      &=\sigma^{-(L-1)}\boldsymbol J_L
      \frac{d\boldsymbol\theta}{dt} \\
      &=\sigma^{L-1}\boldsymbol J_L
      \bigl(\boldsymbol I+(L-1)\boldsymbol P_{\boldsymbol\theta}
      +\Delta\overline{\boldsymbol K}_L\bigr)\boldsymbol g \\
      &=\sigma^{L-1}
      (\boldsymbol I+\boldsymbol J_L\Delta\overline{\boldsymbol K}_L)
      \boldsymbol g \\
      &=(\boldsymbol I+\boldsymbol J_L\Delta\overline{\boldsymbol K}_L)
      \boldsymbol\Sigma
      \bigl(\sigma^{L-1}\boldsymbol\beta
      -\sigma^{2(L-1)}\boldsymbol z\bigr).
    \end{split}
  \end{equation}
\end{proof}

\begin{proof}[Proof of \Cref{phased:cor:rescaled-predictor-tau-dynamics}]
  Differentiating \cref{phased:eq:time-reparametrization} gives $d\tau/dt=\sigma^{2(L-1)}$. Combining this identity with \cref{phased:lem:rescaled-predictor-dynamics}, we obtain
  \begin{equation}
    \frac{d\boldsymbol z}{d\tau}
    =(\boldsymbol I+\boldsymbol J_L\Delta\overline{\boldsymbol K}_L)
    \boldsymbol\Sigma
    \bigl(\sigma^{-(L-1)}\boldsymbol\beta-\boldsymbol z\bigr).
  \end{equation}
  Hence,
  \begin{equation}
    \frac{d\boldsymbol z}{d\tau}
    =-\boldsymbol\Sigma\boldsymbol z
    +\sigma^{-(L-1)}\boldsymbol\Sigma\boldsymbol\beta
    +\boldsymbol f_K,
  \end{equation}
  where the forcing term in the corollary statement is given by
  \begin{equation}
    \boldsymbol f_K = \boldsymbol J_L
    \Delta\overline{\boldsymbol K}_L\boldsymbol\Sigma
    \left(\sigma^{-(L-1)}\boldsymbol\beta-\boldsymbol z\right).
    \label{phased:eq:kernel-forcing}
  \end{equation}
\end{proof}

For any $0\le \tau_0<\tau$, the corresponding variation of constants formula is
\begin{equation}
  \boldsymbol z(\tau)
  =
  e^{-(\tau-\tau_0)\boldsymbol\Sigma}\boldsymbol z(\tau_0)
  +
  \int_{\tau_0}^\tau
  e^{-(\tau-v)\boldsymbol\Sigma}
  \sigma(v)^{-(L-1)}
  \boldsymbol\Sigma\boldsymbol\beta\,dv
  +
  \int_{\tau_0}^\tau
  e^{-(\tau-v)\boldsymbol\Sigma}
  \boldsymbol f_K(v)\,dv.
  \label{phased:eq:variation-constants-formula}
\end{equation}
The perturbation term also satisfies the following pointwise bound. For any fixed $M>0$, at every time at which $\varepsilon\le M\sigma^2$, \cref{commonwork:prop:perturbative-kernel-expansion} and $\|\boldsymbol J_L\|_{\mathrm{op}}=1$ imply
\begin{equation}
  \|\boldsymbol f_K\|
  \le
  \sigma^{-(L-1)}
  \|\boldsymbol J_L\|_{\mathrm{op}}
  \|\Delta\overline{\boldsymbol K}_L\|_{\mathrm{op}}\|\boldsymbol g\|
  \le
  C\varepsilon\sigma^{-(L+1)}\|\boldsymbol g\|.
  \label{phased:eq:forcing-bound}
\end{equation}

\subsection{Gaussian Initialization}

Recall that the initial weight matrices are independent across layers, with i.i.d. entries within each layer distributed as $(\boldsymbol W_k(0))_{ij}\sim\mathcal N(0,s^2/n_{k-1})$ for $1\le k\le L$.

\begin{lemma}
  \label{commonwork:lem:fixed-s-isotropy}
  Fix $L$, and suppose that $n_k\asymp d$ for $1\le k<L$ and that $s>0$. There exist constants $c,C>0$, independent of $d$, $s$, and $\eta_0$, such that, for every fixed $\eta_0>0$ and all sufficiently large $d$, the following bounds hold with probability at least
  $1-C\exp(-c\min\{1,\eta_0^2\}d)$:
  \begin{equation}
    \sigma(0)^2\asymp s^2,
    \label{commonwork:eq:init-radius}
  \end{equation}
  \begin{equation}
    \varepsilon\lesssim s^2,
    \label{commonwork:eq:init-defect}
  \end{equation}
  \begin{equation}
    \max_{1\le k\le L}\|\boldsymbol W_k(0)\|_{\mathrm{op}}
    \lesssim s,
    \qquad
    \|\boldsymbol W_{L:2}(0)\|\asymp s^{L-1},
    \label{commonwork:eq:init-layer-scales}
  \end{equation}
  and
  \begin{equation}
    \left|
    \mathcal L(0)
    -\frac12\|\boldsymbol\Sigma^{1/2}\boldsymbol\beta\|^2
    \right|
    \le
    C(\eta_0r(0)+\eta_0^2r(0)^2).
    \label{commonwork:eq:initial-loss-closeness}
  \end{equation}
  Moreover, on the same event, if $s=1$ and $\eta_0$ is sufficiently small, then
  \begin{equation}
    \|\boldsymbol g(0)\|\ge c.
    \label{commonwork:eq:initial-gradient-lower}
  \end{equation}
\end{lemma}

\begin{proof}
  Write $\boldsymbol W_k(0)=s\boldsymbol G_k$, where the entries of $\boldsymbol G_k$ are i.i.d. $\mathcal N(0,1/n_{k-1})$. Since $L$ is fixed and $n_k\asymp d$ for $k<L$, the standard Gaussian operator-norm bound and a union bound over the layers give
  \begin{equation}
    \max_{1\le k\le L}\|\boldsymbol W_k(0)\|_{\mathrm{op}}
    \le Cs
    \label{commonwork:eq:init-opnorm-proof}
  \end{equation}
  with probability at least $1-Ce^{-cd}$. Consequently,
  \begin{equation}
    \varepsilon
    \le
    \max_{1\le k<L}
    \left(
    \|\boldsymbol W_k(0)\|_{\mathrm{op}}^2
    +\|\boldsymbol W_{k+1}(0)\|_{\mathrm{op}}^2
    \right)
    \le Cs^2,
  \end{equation}
  proving \cref{commonwork:eq:init-defect} and the first bound in \cref{commonwork:eq:init-layer-scales}.

  For $1\le k\le L$, set $q_k:=\|\boldsymbol W_{L:k}(0)\|$, with $q_{L+1}:=1$. Conditioned on $\boldsymbol W_{L:k+1}(0)$, the coordinates of $\boldsymbol W_{L:k+1}(0)\boldsymbol W_k(0)$ are independent centered Gaussians with variance $s^2q_{k+1}^2/n_{k-1}$. Hence, conditionally,
  \begin{equation}
    q_k^2
    \stackrel{\rm d}{=}
    s^2q_{k+1}^2\frac{\chi_{n_{k-1}}^2}{n_{k-1}}.
  \end{equation}
  Applying standard chi-square concentration successively for $k=L,L-1,\ldots,1$, together with a union bound over the layers, yields
  \begin{equation}
    q_k\asymp s^{L-k+1},
    \label{commonwork:eq:init-product-scales-proof}
  \end{equation}
  with probability at least $1-Ce^{-cd}$. In particular, $\|\boldsymbol W_{L:2}(0)\|=q_2\asymp s^{L-1}$ and $r(0)=q_1\asymp s^L$. Since $\sigma(0)^2=r(0)^{2/L}$, this proves \cref{commonwork:eq:init-radius} and the remaining bound in \cref{commonwork:eq:init-layer-scales}.

  It remains to control the loss. Conditioned on $\boldsymbol W_{L:2}(0)$, the end-to-end predictor is a centered isotropic Gaussian vector:
  \begin{equation}
    \boldsymbol\theta(0)
    \stackrel{\rm d}{=}
    \frac{s\|\boldsymbol W_{L:2}(0)\|}{\sqrt d}\boldsymbol h,
    \qquad
    \boldsymbol h\sim\mathcal N(\boldsymbol 0,\boldsymbol I_d),
  \end{equation}
  where $\boldsymbol h$ is independent of $\boldsymbol W_{L:2}(0)$.
  Thus, $\widehat{\boldsymbol\theta}(0)$ is uniformly distributed on $\mathbb S^{d-1}$. Since $a>0$ and $p>0$, we have $\|\boldsymbol\Sigma\|_{\mathrm{op}}=c_\lambda$, $\|\boldsymbol\Sigma\boldsymbol\beta\|\lesssim1$, and $\operatorname{tr}\boldsymbol\Sigma=o(d)$.
  Recall that $\widehat{\boldsymbol\theta}(0)=\boldsymbol h/\|\boldsymbol h\|$, Gaussian norm concentration gives $\|\boldsymbol h\|^2\ge d/2$ except with probability $Ce^{-cd}$. The map $\boldsymbol h\mapsto\|\boldsymbol\Sigma^{1/2}\boldsymbol h\|$ is $\sqrt{c_\lambda}$-Lipschitz, and $\mathbb E[\|\boldsymbol\Sigma^{1/2}\boldsymbol h\|]\le\sqrt{\operatorname{tr}\boldsymbol\Sigma}$. Gaussian concentration therefore gives
  \begin{equation}
    \Pr\left(
    \|\boldsymbol\Sigma^{1/2}\boldsymbol h\|
    >\sqrt{\operatorname{tr}\boldsymbol\Sigma}+\eta_0\sqrt d
    \right)
    \le\exp\left(-\frac{\eta_0^2d}{2c_\lambda}\right).
  \end{equation}
  For every fixed $\eta_0>0$, we have $\operatorname{tr}\boldsymbol\Sigma\le\eta_0^2d$ for all sufficiently large $d$. Combining these estimates with Gaussian concentration for $\langle\boldsymbol\Sigma\boldsymbol\beta,\boldsymbol h\rangle$ and $h_1$, and taking a union bound, gives
  \begin{equation}
    |\langle\boldsymbol\Sigma\boldsymbol\beta,\boldsymbol h\rangle|
    \le C\eta_0\sqrt d,
    \qquad
    \boldsymbol h^\top\boldsymbol\Sigma\boldsymbol h
    \le 4\eta_0^2d,
    \qquad
    |h_1|\le\eta_0\sqrt{d/2}
  \end{equation}
  with failure probability at most $C\exp(-c\min\{1,\eta_0^2\}d)$, for every fixed $\eta_0>0$ and all sufficiently large $d$. Dividing the first and third inequalities by $\|\boldsymbol h\|$ and the second by $\|\boldsymbol h\|^2$ therefore yields
  \begin{equation}
    |\langle\boldsymbol\Sigma\boldsymbol\beta,\widehat{\boldsymbol\theta}(0)\rangle|
    \le C\eta_0,
    \qquad
    \widehat{\boldsymbol\theta}(0)^\top
    \boldsymbol\Sigma\widehat{\boldsymbol\theta}(0)
    \le C\eta_0^2,
    \qquad
    |\widehat\theta_1(0)|\le\eta_0,
    \label{commonwork:eq:init-spherical-concentration-proof}
  \end{equation}
  with probability at least $1-C\exp(-c\min\{1,\eta_0^2\}d)$. It follows that
  \begin{equation}
    \begin{split}
      \left|
      \mathcal L(0)
      -\frac12\|\boldsymbol\Sigma^{1/2}\boldsymbol\beta\|^2
      \right|
      &\le
      r(0)|\langle\boldsymbol\Sigma\boldsymbol\beta,
      \widehat{\boldsymbol\theta}(0)\rangle|
      +\frac12r(0)^2
      \widehat{\boldsymbol\theta}(0)^\top
      \boldsymbol\Sigma\widehat{\boldsymbol\theta}(0) \\
      &\le C\bigl(\eta_0r(0)+\eta_0^2r(0)^2\bigr),
    \end{split}
  \end{equation}
  which proves \cref{commonwork:eq:initial-loss-closeness}.

  Finally, suppose that $s=1$. On the intersection of the high-probability events used to establish \cref{commonwork:eq:init-product-scales-proof,commonwork:eq:init-spherical-concentration-proof}, we have $r(0)\le C$ and $|\widehat\theta_1(0)|\le\eta_0$. Since $\lambda_1=c_\lambda$ and $\beta_1=c_\beta>0$,
  \begin{equation}
    |g_1(0)|
    =\lambda_1|\beta_1-r(0)\widehat\theta_1(0)|
    \ge c_\lambda(c_\beta-C\eta_0).
  \end{equation}
  Choosing $\eta_0>0$ sufficiently small gives $\|\boldsymbol g(0)\|\ge|g_1(0)|\ge c$, proving \cref{commonwork:eq:initial-gradient-lower}.
\end{proof}

\section{From Initialization to the Onset of the Power-Law Scaling Phase}
\label{phased:sec:escape-from-origin}

Throughout this section, fix $\eta_0>0$ sufficiently small for all arguments
below and work on the high-probability event in
\cref{commonwork:lem:fixed-s-isotropy}.

\subsection{Escape from the Origin under Small Initialization}

We now state the precise small-initialization assumption used in the main text.

\begin{assumption}[Small initialization]
  \label{phased:ass:small-initialization}
  Let $0<\delta_0\le1$ be sufficiently small. For any fixed $0<\delta\le\delta_0$, assume that $d$ and the initialization scale $s>0$ satisfy
  \begin{equation}
    s^2\le
    \begin{cases}
      \delta,                    & p < 1, \\
      \delta(\log d)^{-(L-1)/L}, & p = 1, \\
      \delta d^{-a(p-1)},        & p > 1.
    \end{cases}
    \label{phased:eq:regime-smallness}
  \end{equation}
\end{assumption}

The escape estimates below require only a dimension-independent upper bound on $s$. The full condition in \cref{phased:ass:small-initialization} is first used in the proof of \cref{cor:small-init-predictor-norm-scaling} to control initialization effects uniformly up to $\tau_+=c_0d^a$.

We also define
\begin{equation}
  \Phi(\boldsymbol\theta)
  :=
  \langle\boldsymbol\Sigma\boldsymbol\beta,\boldsymbol\theta\rangle
  -\frac12\boldsymbol\theta^\top\boldsymbol\Sigma\boldsymbol\theta
  =
  \frac12\|\boldsymbol\Sigma^{1/2}\boldsymbol\beta\|^2
  -\mathcal L(\boldsymbol\theta).
  \label{phased:eq:zero-predictor-loss-difference}
\end{equation}

\begin{proposition}
  \label{phased:prop:early-radius-entrance-loss}
  There exist fixed constants $s_0>0$, $0<\tau_{\mathrm e}<\min\{1,c_\lambda^{-1}\}$, and $c,C>0$, independent of $d$ and $s$, such that, for all sufficiently large $d$, if $0<s\le s_0$, then for $0\le \tau\le \tau_{\mathrm e}$,
  \begin{equation}
    c(s^L+\tau)\le r(\tau)\le C(s^L+\tau),
    \label{phased:eq:early-radius}
  \end{equation}
  and
  \begin{equation}
    \mathcal L(\tau_{\mathrm e}) \le
    \frac12\|\boldsymbol\Sigma^{1/2}\boldsymbol\beta\|^2-c.
    \label{phased:eq:entrance-loss-upper-gap}
  \end{equation}
\end{proposition}

\begin{proof}
  Let $s_0>0$ be chosen sufficiently small below, and suppose that $0<s\le s_0$.
  By \cref{commonwork:eq:init-radius,commonwork:eq:init-defect,commonwork:eq:initial-loss-closeness}, $\sigma(0)\asymp s$ and $\varepsilon\lesssim\sigma(0)^2$ for all sufficiently large $d$. After decreasing $s_0$ if necessary, we also have $\sigma(0)\le1$, so \cref{commonwork:eq:initial-loss-closeness} gives
  \begin{equation}
    \Phi(\boldsymbol\theta(0))
    \ge-C\eta_0\sigma(0)^L.
    \label{phased:eq:early-initial-loss}
  \end{equation}
  We define
  \begin{equation}
    u(t):=\int_0^t\sigma(v)^{L-1}\,dv.
    \label{phased:eq:auxiliary-u-time}
  \end{equation}
  On any interval on which $\sigma>0$, we have $du/dt=\sigma^{L-1}>0$. Hence $u$ is strictly increasing and admits an inverse $t(u)$ on its image. With a slight abuse of notation, we write $\boldsymbol\theta(u)=\boldsymbol\theta(t(u))$.
  Then
  \begin{equation}
    \frac{d\boldsymbol z}{du}
    =\boldsymbol J_L\overline{\boldsymbol K}_L\boldsymbol g,
    \label{phased:eq:u-time-dynamics}
  \end{equation}
  and
  \begin{equation}
    \frac{d}{du}\Phi(\boldsymbol\theta(u))
    =\sigma^{L-1}
    \boldsymbol g^\top\overline{\boldsymbol K}_L\boldsymbol g.
  \end{equation}

  \step{Growth in the $u$-parameterization.}

  We split the argument into two substeps. On $[0,u_1]$, we show that $\sigma$ remains comparable to $\sigma(0)$ and that $\Phi(\boldsymbol\theta(u_1))>0$. On $[u_1,u_*]$, we use this positivity to derive a lower bound for $\Phi(\boldsymbol\theta(u))^{1/L}$. We first record estimates common to both substeps.

  Since $\|\boldsymbol g-\boldsymbol\Sigma\boldsymbol\beta\|\le C r=C\sigma^L$ and $\|\boldsymbol\Sigma\boldsymbol\beta\|\asymp1$, we may fix $\sigma_*>0$ sufficiently small so that $\|\boldsymbol g\|\asymp1$ whenever $\sigma\le\sigma_*$.

  For any fixed $\gamma\in(0,1]$, as long as $\gamma\sigma(0)\le\sigma\le\sigma_*$, the bound $\varepsilon\lesssim\sigma(0)^2$ implies $\varepsilon\le C_\gamma\sigma^2$. Hence \cref{commonwork:eq:perturbative-kernel-sandwich} gives
  \begin{equation}
    \frac{d}{du}\Phi(\boldsymbol\theta(u))\ge c_\gamma\sigma^{L-1}.
    \label{phased:eq:early-phi-growth}
  \end{equation}
  Moreover, combining \cref{phased:eq:u-time-dynamics} with $\|\boldsymbol z\|=\sigma$, $\|\boldsymbol J_L\|_{\mathrm{op}}=1$, and the same kernel bound yields
  \begin{equation}
    \left|\frac{d\sigma}{du}\right|
    \le\left\|\frac{d\boldsymbol z}{du}\right\|
    \le C_\gamma.
    \label{phased:eq:early-sigma-speed}
  \end{equation}

  \substep{The interval $0\le u\le u_1$.}\label{phased:substep:early-growth-u1}

  We first apply the preceding bounds with $\gamma=1/2$. Choose a fixed $h\in(0,1]$ such that $C_{1/2}h\le1/4$, and set $u_1:=h\sigma(0)$. Then, as long as $\frac12\sigma(0)\le\sigma\le\sigma_*$,
  \begin{equation}
    |\sigma(u)-\sigma(0)|\le C_{1/2}u\le C_{1/2}h\sigma(0)
    \le \frac{1}{4}\sigma(0).
  \end{equation}
  This implies
  \begin{equation}
    \frac34\sigma(0)\le\sigma(u)\le\frac54\sigma(0).
  \end{equation}
  Using $\sigma(0)\lesssim s$, decrease $s_0$ if necessary so that $5\sigma(0)/4<\sigma_*$ for all sufficiently large $d$ whenever $0<s\le s_0$. Thus neither $\sigma=\sigma(0)/2$ nor $\sigma=\sigma_*$ can be reached before $u_1$. By continuity, the bounds hold throughout $0\le u\le u_1$.

  Integrating \cref{phased:eq:early-phi-growth} over this interval gives
  \begin{equation}
    \begin{split}
      \Phi(\boldsymbol\theta(u_1))
      & \ge \Phi(\boldsymbol\theta(0))+c_{1/2}\int_0^{u_1}\sigma(v)^{L-1}\,dv \\
      & \ge-C\eta_0\sigma(0)^L
      +c_{1/2}h\left(\frac34\right)^{L-1}\sigma(0)^L.
    \end{split}
  \end{equation}
  Since $\eta_0$ is sufficiently small,
  \begin{equation}
    \Phi(\boldsymbol\theta(u_1))\ge c\sigma(0)^L>0.
    \label{phased:eq:early-positive-phi}
  \end{equation}
  For every $\boldsymbol\theta$,
  \begin{equation}
    \Phi(\boldsymbol\theta)
    \le
    \langle\boldsymbol\Sigma\boldsymbol\beta,\boldsymbol\theta\rangle
    \le C\sigma^L.
    \label{phased:eq:early-phi-radius-upper}
  \end{equation}
  Therefore $\sigma(u_1)\ge c\sigma(0)$.
  We may thus fix $\gamma_*\in(0,1/2)$ sufficiently small so that
  \begin{equation}
    \sigma(u_1)\ge2\gamma_*\sigma(0).
    \label{phased:eq:early-sigma-at-u1}
  \end{equation}

  \substep{The interval $u_1\le u\le u_*$.}

  We now apply the preceding bounds with $\gamma=\gamma_*$.
  As long as $\gamma_*\sigma(0)\le\sigma\le\sigma_*$ after $u_1$, \cref{phased:eq:early-phi-growth} shows that $\Phi(\boldsymbol\theta(u))$ remains positive.
  Moreover, \cref{phased:eq:early-phi-growth,phased:eq:early-phi-radius-upper} give
  \begin{equation}
    \frac{d}{du}\Phi(\boldsymbol\theta(u))^{1/L}
    =
    \frac1L\Phi(\boldsymbol\theta(u))^{-(L-1)/L}
    \frac{d}{du}\Phi(\boldsymbol\theta(u))
    \ge c.
    \label{phased:eq:early-phi-root-growth}
  \end{equation}
  Integrating this inequality from $u_1$ to $u$ gives
  \begin{equation}
    \begin{split}
      \Phi(\boldsymbol\theta(u))^{1/L}
      & \ge \Phi(\boldsymbol\theta(u_1))^{1/L}+c(u-u_1) \\
      & \ge c\sigma(0)+c(u-u_1)                         \\
      & \ge c(\sigma(0)+u).
    \end{split}
    \label{phased:eq:early-phi-root-lower}
  \end{equation}
  Combining this with \cref{phased:eq:early-phi-radius-upper} gives $\sigma(u)\ge c(\sigma(0)+u)$.

  On the other hand, integrating \cref{phased:eq:early-sigma-speed} gives
  \begin{equation}
    \begin{split}
      \sigma(u)
      & \le \sigma(u_1)
      +\int_{u_1}^u
      \left|\frac{d\sigma}{dv}\right|\,dv         \\
      & \le \frac54\sigma(0)+C_{\gamma_*}(u-u_1) \\
      & \le C(\sigma(0)+u).
    \end{split}
  \end{equation}
  Choose a fixed $u_*>0$ sufficiently small that $C_{\gamma_*}u_*\le \sigma_*/4$. After decreasing $s_0$ if necessary, we have $u_1<u_*$ and $5\sigma(0)/4+C_{\gamma_*}u_*<\sigma_*$ for all sufficiently large $d$ whenever $0<s\le s_0$. Since $u\mapsto\Phi(\boldsymbol\theta(u))$ is nondecreasing, \cref{phased:eq:early-positive-phi,phased:eq:early-phi-radius-upper} and the choice of $\gamma_*$ give $\sigma(u)\ge2\gamma_*\sigma(0)$. Moreover, $\sigma(u)\le 5\sigma(0)/4+C_{\gamma_*}u_*<\sigma_*$. Hence neither boundary $\sigma=\gamma_*\sigma(0)$ nor $\sigma=\sigma_*$ can be reached before $u_*$. By continuity, the estimates hold throughout $u_1\le u\le u_*$. Combining this with \cref{phased:substep:early-growth-u1}, we obtain
  \begin{equation}
    \sigma(u)\asymp\sigma(0)+u
    \label{phased:eq:early-root-growth}
  \end{equation}
  for $0\le u\le u_*$.

  \step{Conversion to the $\tau$-parameterization.}

  Since $d\tau/du=\sigma^{L-1}>0$, the map $u\mapsto\tau(u)$ is strictly increasing on $[0,u_*]$. Combining $d\tau/du=\sigma^{L-1}$ with \cref{phased:eq:early-root-growth} and integrating from $0$ to $u$ yields
  \begin{equation}
    \tau(u)
    \asymp
    (\sigma(0)+u)^L-\sigma(0)^L.
    \label{phased:eq:early-u-tau-relation}
  \end{equation}

  Since $u_*>0$ is fixed, \cref{phased:eq:early-u-tau-relation} gives $\tau(u_*)\gtrsim u_*^L>0$. On the other hand, since $u_1=h\sigma(0)$, the same estimate gives $\tau(u_1)\lesssim\sigma(0)^L\lesssim s^L$. We may therefore fix $0<\tau_{\mathrm e}<\min\{1,c_\lambda^{-1}\}$ sufficiently small and then decrease $s_0$ if necessary so that $\tau(u_1)<\tau_{\mathrm e}<\tau(u_*)$ for all sufficiently large $d$ whenever $0<s\le s_0$.

  For such $d$, the monotonicity of $\tau(u)$ implies that, for every $0\le\tau\le\tau_{\mathrm e}$, there exists a unique $u\in[0,u_*]$ such that $\tau=\tau(u)$. Combining \cref{phased:eq:early-root-growth,phased:eq:early-u-tau-relation}, we obtain
  \begin{equation}
    r(\tau)
    =\sigma(u)^L
    \asymp(\sigma(0)+u)^L
    \asymp\sigma(0)^L+\tau
    \asymp s^L+\tau,
  \end{equation}
  which proves \cref{phased:eq:early-radius}.

  Moreover, the unique $u$ corresponding to $\tau_{\mathrm e}$ satisfies $u>u_1$. Applying \cref{phased:eq:early-phi-root-lower,phased:eq:early-u-tau-relation} at this value of $u$ gives
  \begin{equation}
    \Phi(\boldsymbol\theta(\tau_{\mathrm e}))
    \ge c(\sigma(0)^L+\tau_{\mathrm e})
    \ge c.
  \end{equation}
  The definition of $\Phi$ now proves \cref{phased:eq:entrance-loss-upper-gap}.
\end{proof}

Let $s_0$ and $\tau_{\mathrm e}$ be the constants given by \cref{phased:prop:early-radius-entrance-loss}. Decreasing $s_0$ if necessary, we may assume that $s_0<1$. Let $t_{\mathrm e}$ denote the unique time at which $\tau(t_{\mathrm e})=\tau_{\mathrm e}$.

\begin{corollary}
  \label{phased:cor:post-entry-radius-kernel}
  Suppose $0<s\le s_0$. Then there exist constants $c,c',C>0$, independent of $d$, $s$, and $\tau$, such that, for all sufficiently large $d$ and every $\tau\ge\tau_{\mathrm e}$,
  \begin{equation}
    \sigma(\tau)^2\ge c',
    \label{phased:eq:post-entry-radius-lower}
  \end{equation}
  and
  \begin{equation}
    c\boldsymbol I \preceq\overline{\boldsymbol K}_L(\tau)\preceq C\boldsymbol I.
    \label{phased:eq:post-entry-kernel-bounds}
  \end{equation}
\end{corollary}

\begin{proof}
  For every $t\ge t_{\mathrm e}$, monotonicity of the loss and
  \cref{phased:eq:entrance-loss-upper-gap} give
  \begin{equation}
    c'\le \Phi(\boldsymbol\theta(t_{\mathrm e}))
    \le\Phi(\boldsymbol\theta(t))
    \le\langle\boldsymbol\Sigma\boldsymbol\beta,\boldsymbol\theta(t)\rangle
    \le C\sigma(t)^L.
  \end{equation}
  Hence $\sigma(t)^2\ge c'$ for every $t\ge t_{\mathrm e}$.

  Since $\varepsilon\lesssim s^2\le s_0^2$ and $\sigma(t)^2\ge c'$, there exists a fixed $M>0$, independent of $d$, $s$, and $t$, such that $\varepsilon\le M\sigma(t)^2$. Applying \cref{commonwork:prop:perturbative-kernel-expansion} gives $c\boldsymbol I \preceq \overline{\boldsymbol K}_L(t) \preceq C\boldsymbol I$.

  To return to the $\tau$-parameterization, note that $d\tau/dt\ge(c')^{L-1}$, so $\tau(t)\to\infty$.
  Since $\tau$ is strictly increasing and $\tau(t_{\mathrm e})=\tau_{\mathrm e}$, it defines a bijection from $[t_{\mathrm e},\infty)$ onto $[\tau_{\mathrm e},\infty)$.
  Reparameterizing the preceding bounds by $\tau$ proves \cref{phased:eq:post-entry-radius-lower,phased:eq:post-entry-kernel-bounds}.
\end{proof}

\begin{corollary}
  \label{phased:cor:time-entrance-delay}
  Suppose $0<s\le s_0$. Then, for all sufficiently large $d$,
  \begin{equation}
    t_{\mathrm e}
    \asymp
    \begin{cases}
      1,          & L=1, \\
      \log(1/s),  & L=2, \\
      s^{-(L-2)}, & L>2.
    \end{cases}
    \label{phased:eq:time-entrance-delay}
  \end{equation}
\end{corollary}

\begin{proof}
  For $0 \le \tau \le \tau_{\mathrm e}$,
  \begin{equation}
    \frac{dt}{d\tau} = \sigma(\tau)^{-2(L-1)} = r(\tau)^{-2(L-1)/L} \asymp (s^L + \tau)^{-2(L-1)/L} \asymp (s^{2L} + \tau^2)^{-(L-1)/L}.
  \end{equation}

  Integrating from $0$ to $\tau_{\mathrm e}$ and changing variables via $\tau=s^Lx$, we obtain
  \begin{equation}
    t_{\mathrm e}
    \asymp
    \int_0^{\tau_{\mathrm e}}
    (s^{2L}+\tau^2)^{-(L-1)/L}\,d\tau
    =
    s^{2-L}
    \int_0^{\tau_{\mathrm e}s^{-L}}
    (1+x^2)^{-(L-1)/L}\,dx.
  \end{equation}
  For $L=1$, the time reparameterization satisfies $\tau=t$, so $t_{\mathrm e}=\tau_{\mathrm e}\asymp1$. If $L=2$, then $(L-1)/L=1/2$, and the integral is comparable to $\log(1/s)$. If $L>2$, then $(L-1)/L>1/2$, so the integral is bounded above and below by positive constants for all sufficiently small $s$. This proves \cref{phased:eq:time-entrance-delay}.
\end{proof}

To isolate the contribution accumulated up to $\tau_{\mathrm e}$, define
\begin{equation}
  \boldsymbol\xi_{\mathrm e}
  :=
  \boldsymbol z(\tau_{\mathrm e})
  -
  \int_0^{\tau_{\mathrm e}}
  e^{-(\tau_{\mathrm e}-v)\boldsymbol\Sigma}
  \sigma(v)^{-(L-1)}
  \boldsymbol\Sigma\boldsymbol\beta\,dv.
  \label{phased:eq:entrance-remainder-definition}
\end{equation}

\begin{corollary}
  \label{phased:cor:entrance-remainder}
  Suppose $0<s\le s_0$. Then, for all sufficiently large $d$,
  \begin{equation}
    \|\boldsymbol\xi_{\mathrm e}\|\le Cs.
    \label{phased:eq:entrance-remainder-bound}
  \end{equation}
\end{corollary}

\begin{proof}
  By \cref{phased:eq:early-radius} and $\varepsilon\lesssim s^2$, we have $\varepsilon\le C\sigma(\tau)^2$ for $0\le\tau\le\tau_{\mathrm e}$, so \cref{phased:eq:forcing-bound} applies throughout this interval.
  Moreover, by the monotonicity of the loss, $\|\boldsymbol g\|^2\le2\|\boldsymbol\Sigma\|_{\mathrm{op}}\mathcal L\le C$.
  Therefore,
  \begin{equation}
    \int_0^{\tau_{\mathrm e}}\|\boldsymbol f_K(\tau)\|\,d\tau
    \le C\varepsilon\int_0^{\tau_{\mathrm e}}(s^L+\tau)^{-(L+1)/L}\,d\tau
    \le CL\frac{\varepsilon}{s}
    \le Cs.
    \label{phased:eq:kernel-forcing-integral}
  \end{equation}
  Set $\tau_0=0$ and $\tau=\tau_{\mathrm e}$ in \cref{phased:eq:variation-constants-formula}, and substitute the resulting expression into \cref{phased:eq:entrance-remainder-definition}. This yields
  \begin{equation}
    \boldsymbol\xi_{\mathrm e}
    =
    e^{-\tau_{\mathrm e}\boldsymbol\Sigma}\boldsymbol z(0)
    +
    \int_0^{\tau_{\mathrm e}}
    e^{-(\tau_{\mathrm e}-v)\boldsymbol\Sigma}
    \boldsymbol f_K(v)\,dv.
  \end{equation}
  Using $\|\boldsymbol z(0)\|\le \sigma(0)\le Cs$ and \cref{phased:eq:kernel-forcing-integral}, we conclude that $\|\boldsymbol\xi_{\mathrm e}\|\le Cs$.
\end{proof}

\begin{corollary}
  \label{phased:cor:post-entry-representation}
  Suppose $0<s\le s_0$, and let $d$ be sufficiently large. For each $\tau\ge\tau_{\mathrm e}$, define $\boldsymbol R(\tau)$ by
  \begin{equation}
    \boldsymbol R(\tau) := \sigma(\tau)^{L-1} \left[ e^{-(\tau-\tau_{\mathrm e})\boldsymbol\Sigma}\boldsymbol\xi_{\mathrm e} + \int_{\tau_{\mathrm e}}^{\tau} e^{-(\tau-v)\boldsymbol\Sigma}\boldsymbol f_K(v)\,dv \right].
    \label{phased:eq:post-entry-remainder}
  \end{equation}
  Then
  \begin{equation}
    \boldsymbol\theta(\tau) = \sigma(\tau)^{L-1} \int_0^\tau e^{-(\tau-v)\boldsymbol\Sigma} \boldsymbol\Sigma\boldsymbol\beta\, \sigma(v)^{-(L-1)}\,dv + \boldsymbol R(\tau).
    \label{phased:eq:post-entry-theta}
  \end{equation}
\end{corollary}

\begin{proof}
  Applying \cref{phased:eq:variation-constants-formula} with $\tau_0 = \tau_{\mathrm e}$, substituting the definition of $\boldsymbol\xi_{\mathrm e}$ from \cref{phased:eq:entrance-remainder-definition}, and using $e^{-(\tau-\tau_{\mathrm e})\boldsymbol\Sigma}e^{-(\tau_{\mathrm e}-v)\boldsymbol\Sigma}=e^{-(\tau-v)\boldsymbol\Sigma}$, we obtain
  \begin{equation}
    \boldsymbol z(\tau)
    =
    e^{-(\tau-\tau_{\mathrm e})\boldsymbol\Sigma}
    \boldsymbol\xi_{\mathrm e}
    +
    \int_0^\tau
    e^{-(\tau-v)\boldsymbol\Sigma}
    \sigma(v)^{-(L-1)}
    \boldsymbol\Sigma\boldsymbol\beta\,dv
    +
    \int_{\tau_{\mathrm e}}^\tau
    e^{-(\tau-v)\boldsymbol\Sigma}
    \boldsymbol f_K(v)\,dv.
    \label{phased:eq:post-entry-z-representation}
  \end{equation}
  Multiplying \cref{phased:eq:post-entry-z-representation} by $\sigma(\tau)^{L-1}$ and using \cref{phased:eq:rescaled-predictor} together with the definition of $\boldsymbol R$ in \cref{phased:eq:post-entry-remainder} gives \cref{phased:eq:post-entry-theta}.
\end{proof}

The first term in $\boldsymbol R(\tau)$ propagates $\boldsymbol\xi_{\mathrm e}$, which incorporates the effects of initialization and the kernel perturbation accumulated over $[0,\tau_{\mathrm e}]$. The second term accounts for the kernel perturbation accumulated over $[\tau_{\mathrm e},\tau]$.

\subsection{Standard Initialization}

For $s=1$, the initialization estimates give $\sigma(0)^2\asymp1$. The next lemma shows that $\sigma(\tau)^2$ remains bounded below by a positive constant and that $\overline{\boldsymbol K}_L(\tau)$ remains uniformly comparable to the identity for all $\tau\ge0$.

\begin{lemma}
  \label{phased:lem:fixed-scale-radius-kernel}
  Suppose that $s=1$. Then there exist constants $c,C>0$, independent of $d$ and $\tau$, such that the following bounds hold for all sufficiently large $d$ and all $\tau\ge0$:
  \begin{equation}
    \sigma(\tau)^2\ge c,
    \label{phased:eq:fixed-scale-radius-lower}
  \end{equation}
  and
  \begin{equation}
    c\boldsymbol I \preceq\overline{\boldsymbol K}_L(\tau)\preceq C\boldsymbol I.
    \label{phased:eq:fixed-scale-kernel-bounds}
  \end{equation}
\end{lemma}

\begin{proof}
  By \cref{commonwork:eq:init-layer-scales}, $\max_k\|\boldsymbol W_k(0)\|_{\mathrm{op}}\le C$ and $\|\boldsymbol W_{L:2}(0)\|\ge c$. Moreover, \cref{commonwork:eq:init-radius,commonwork:eq:initial-gradient-lower} give $r(0)\ge c$ and $\|\boldsymbol g(0)\|\ge c$. Furthermore, \cref{commonwork:eq:init-radius,commonwork:eq:initial-loss-closeness} imply $\mathcal L(0)\le\frac12\|\boldsymbol\Sigma^{1/2}\boldsymbol\beta\|^2+C\eta_0$.
  By the Cauchy--Schwarz inequality and \cref{commonwork:eq:loss-dissipation-expanded},
  \begin{equation}
    \left(
    \sum_{k=1}^L
    \|\boldsymbol W_k(t)-\boldsymbol W_k(0)\|_{\mathrm F}^2
    \right)^{1/2}
    \le \sqrt{t\mathcal L(0)}
    \le C\sqrt t
  \end{equation}
  for all $t\ge 0$. In particular, $\|\boldsymbol W_k(t)\|_{\mathrm{op}}\le C$ for all $1\le k\le L$ and $0\le t\le1$. Thus
  \begin{equation}
    \begin{split}
      & |r(t)-r(0)|
      +\left|\|\boldsymbol W_{L:2}(t)\|
      -\|\boldsymbol W_{L:2}(0)\|\right|
      +\|\boldsymbol g(t)-\boldsymbol g(0)\| \\
      & \qquad\le
      C\sum_{k=1}^L
      \|\boldsymbol W_k(t)-\boldsymbol W_k(0)\|_{\mathrm{op}}
      \le C\sqrt t.
    \end{split}
  \end{equation}
  Choose a sufficiently small fixed $t_s>0$. Then
  \begin{equation}
    \min\{r(t),\|\boldsymbol W_{L:2}(t)\|,\|\boldsymbol g(t)\|\}\ge c
    \label{phased:eq:init-min-lower-bound}
  \end{equation}
  for all $0\le t\le t_s$. By \cref{commonwork:eq:loss-dissipation-expanded,phased:eq:init-min-lower-bound},
  \begin{equation}
    \begin{split}
      \frac{d\mathcal L}{dt}
      & =-\boldsymbol g(t)^\top\boldsymbol K_L(t)\boldsymbol g(t) \\
      & \le
      -\|\boldsymbol W_{L:2}(t)\|^2\|\boldsymbol g(t)\|^2          \\
      & \le-c.
    \end{split}
  \end{equation}
  Since $\eta_0$ is sufficiently small, integrating the preceding inequality over $[0,t_s]$ gives
  \begin{equation}
    \begin{split}
      \mathcal L(t_s)
      & \le \mathcal L(0)-ct_s \\
      & \le
      \frac12\|\boldsymbol\Sigma^{1/2}\boldsymbol\beta\|^2
      +C\eta_0-ct_s             \\
      & \le
      \frac12\|\boldsymbol\Sigma^{1/2}\boldsymbol\beta\|^2-c.
    \end{split}
  \end{equation}
  Since the loss is nonincreasing, for all $t \ge t_s$, we have
  \begin{equation}
    c
    \le
    \frac12\|\boldsymbol\Sigma^{1/2}\boldsymbol\beta\|^2-\mathcal L(t)
    \le
    \langle\boldsymbol\Sigma\boldsymbol\beta,\boldsymbol\theta(t)\rangle
    \le Cr(t).
    \label{phased:eq:loss-lower-gap}
  \end{equation}
  \Cref{phased:eq:init-min-lower-bound} covers $0\le t\le t_s$, whereas \cref{phased:eq:loss-lower-gap} applies for $t\ge t_s$. Hence $\sigma(t)^2 \ge c$ for all $t \ge 0$.

  Finally, \cref{commonwork:eq:init-defect} and the preceding lower bound give $\varepsilon/\sigma(t)^2 \le C$.
  Applying \cref{commonwork:prop:perturbative-kernel-expansion}, we obtain $c\boldsymbol I \preceq\overline{\boldsymbol K}_L(t)\preceq C\boldsymbol I$.
  Moreover, $\frac{d\tau}{dt}=\sigma(t)^{2(L-1)}\ge c$, so $\tau(t)\to\infty$. Hence these estimates give the stated bounds in the $\tau$-parameterization.
\end{proof}

\section{Spectral Estimates and Auxiliary Loss Bounds}
\label{sec:spectral-estimates-auxiliary-loss-bounds}

We first characterize the scaling of $S_d(\tau)$ and estimate the components of the target vector $\boldsymbol\beta$ above and below a spectral cutoff. We then use these estimates to bound the loss $\mathcal L(\tau)$.

\subsection{Spectral Scaling and Tail Estimates}

Fix a constant $c_0>0$ sufficiently small that $(c_\lambda c_0)^{1/a}<1/2$, and set
\begin{equation}
  \tau_+:=c_0d^a.
  \label{phased:eq:time-horizon}
\end{equation}
Define
\begin{equation}
  S_d(\tau)
  := \sum_{i=1}^d \beta_i^2 \min\{\tau\lambda_i, 1\}^2.
  \label{phased:eq:active-mass-definition}
\end{equation}
The function $S_d$ is continuous and nondecreasing.
\begin{lemma}
  \label{phased:lem:active-mass-profile}
  For all sufficiently large $d$, we have
  \begin{equation}
    S_d(\tau) \asymp \tau^2
    \label{phased:eq:active-mass-profile-low}
  \end{equation}
  for $0 \le \tau \le c_\lambda^{-1}$, and
  \begin{equation}
    S_d(\tau) \asymp
    \begin{cases}
      (1+\tau)^{1-p},   & p < 1, \\
      1 + \log(1+\tau), & p = 1, \\
      1,                & p > 1
    \end{cases}
    \label{phased:eq:active-mass-profile}
  \end{equation}
  for $c_\lambda^{-1} \le \tau \le \tau_+$.
\end{lemma}

\begin{proof}
  By the definitions of $\beta_i$ and $\lambda_i$,
  \begin{equation}
    S_d(\tau)
    = c_\beta^2 \sum_{i=1}^d i^{-2b} \min\{c_\lambda \tau i^{-a}, 1\}^2.
  \end{equation}
  When $0 \le \tau \le c_\lambda^{-1}$, we have $c_\lambda \tau i^{-a} \le 1$ for all $i$, and since $2a+2b=1+a(1+p)>1$,
  \begin{equation}
    S_d(\tau)
    = c_\lambda^2 c_\beta^2 \tau^2
    \sum_{i=1}^d i^{-(2a+2b)}
    \asymp \tau^2.
  \end{equation}
  To prove \cref{phased:eq:active-mass-profile}, we first consider $2^ac_\lambda^{-1}\le \tau\le \tau_+$ and set $k:=\lfloor(c_\lambda \tau)^{1/a}\rfloor$.
  The remaining fixed interval $[c_\lambda^{-1},2^ac_\lambda^{-1}]$ will be handled by monotonicity.
  The lower bound on $\tau$ implies $k\ge2$, and the choice of $c_0$ ensures $k<d/2$. Moreover, $k\asymp \tau^{1/a}$.
  By the definition of $k$,
  \begin{equation}
    S_d(\tau)
    =
    c_\beta^2\sum_{i=1}^k i^{-2b}
    +
    c_\lambda^2c_\beta^2\tau^2
    \sum_{i=k+1}^d i^{-(2a+2b)}.
  \end{equation}
  Integral comparisons, together with
  $k\asymp \tau^{1/a}$ and $(1-2b)/a=1-p$, give
  \begin{equation}
    \sum_{i=1}^k i^{-2b}
    \asymp
    1+\int_1^k x^{-2b}\,dx
    \asymp
    \begin{cases}
      (1+\tau)^{1-p}, & p < 1, \\
      1+\log(1+\tau), & p = 1, \\
      1,              & p > 1,
    \end{cases}
  \end{equation}
  and
  \begin{equation}
    \tau^2\sum_{i=k+1}^d i^{-(2a+2b)}
    \asymp
    \tau^2\int_k^d x^{-(2a+2b)}\,dx
    \asymp
    \tau^2k^{1-(2a+2b)}
    \asymp \tau^{1-p}.
  \end{equation}
  Thus, for $2^ac_\lambda^{-1}\le \tau\le \tau_+$,
  \begin{equation}
    S_d(\tau)
    \asymp
    \begin{cases}
      (1+\tau)^{1-p}, & p < 1, \\
      1+\log(1+\tau), & p = 1, \\
      1,              & p > 1.
    \end{cases}
  \end{equation}
  It remains to consider $c_\lambda^{-1}\le\tau\le2^ac_\lambda^{-1}$.
  By \cref{phased:eq:active-mass-profile-low}, $S_d(c_\lambda^{-1})\asymp1$, while the preceding calculation gives $S_d(2^ac_\lambda^{-1})\asymp1$. Since $S_d$ is nondecreasing, it follows that $S_d(\tau)\asymp1$ throughout this fixed interval. For $\tau$ in this interval, each case on the right-hand side of \cref{phased:eq:active-mass-profile} is also comparable to one, which completes the proof.
\end{proof}

For $\mu>0$, define
\begin{equation}
  \boldsymbol P_{\le\mu}
  :=\operatorname{diag}
  \bigl(\mathbf 1_{\{\lambda_i\le\mu\}}\bigr)_{i=1}^d,
  \qquad
  \boldsymbol P_{>\mu}
  :=\boldsymbol I_d-\boldsymbol P_{\le\mu}.
  \label{phased:eq:spectral-projectors}
\end{equation}
Fix constants $c_\mu>2^a c_\lambda$ and $0<\mu_\star\le 2^{-a}c_\lambda$.

\begin{lemma}
  \label{phased:lem:powerlaw-estimates}
  For all sufficiently large $d$ and every $\mu>0$, we have
  \begin{equation}
    \bigl\|
    \boldsymbol P_{\le\mu}
    \boldsymbol\Sigma^{1/2}\boldsymbol\beta
    \bigr\|^2
    \lesssim \mu^p.
    \label{phased:eq:teacher-tail-upper}
  \end{equation}
  If $c_\mu d^{-a}\le\mu\le\mu_\star$, we also have
  \begin{equation}
    \bigl\|
    \boldsymbol P_{\le\mu}
    \boldsymbol\Sigma^{1/2}\boldsymbol\beta
    \bigr\|^2
    \asymp \mu^p,
    \label{phased:eq:teacher-tail-two-sided}
  \end{equation}
  and
  \begin{equation}
    \bigl\|
    \boldsymbol P_{>\mu}\boldsymbol\beta
    \bigr\|^2
    \asymp
    \begin{cases}
      \mu^{-(1-p)},          & p < 1, \\
      1+\log(c_\lambda/\mu), & p = 1, \\
      1,                     & p > 1.
    \end{cases}
    \label{phased:eq:active-teacher-mass-estimate}
  \end{equation}
\end{lemma}

\begin{proof}
  For $\mu>0$, define
  $N_\mu:=|\{i:\lambda_i>\mu\}|$. If
  $\lambda_d\le\mu<c_\lambda$, then
  $N_\mu+1=\lceil(c_\lambda/\mu)^{1/a}\rceil
    \asymp\mu^{-1/a}$. Since $a+2b=1+ap>1$,
  \begin{equation}
    \sum_{\lambda_i \le \mu} \lambda_i \beta_i^2
    = c_\lambda c_\beta^2
    \sum_{i=N_\mu+1}^d i^{-(a+2b)}
    \lesssim (N_\mu+1)^{-(a+2b-1)}
    \lesssim \mu^p.
  \end{equation}
  If $\mu < \lambda_d$, the left-hand side is zero. If $\mu \ge c_\lambda$, the sum is bounded above by a convergent series. Thus, this upper bound holds for all $\mu > 0$, establishing \cref{phased:eq:teacher-tail-upper}.

  Next, suppose $c_\mu d^{-a}\le\mu\le\mu_\star$. For sufficiently large $d$, we have $2\le N_\mu+1\le d/2$ and $N_\mu+1\asymp\mu^{-1/a}$. In particular, $2(N_\mu+1)\le d$, so
  \begin{equation}
    \sum_{i=N_\mu+1}^d i^{-(a+2b)}
    \ge
    \sum_{i=N_\mu+1}^{2(N_\mu+1)} i^{-(a+2b)}
    \ge
    (N_\mu+1)\{2(N_\mu+1)\}^{-(a+2b)}
    \asymp (N_\mu+1)^{-(a+2b-1)}.
  \end{equation}
  Therefore, by $N_\mu+1\asymp\mu^{-1/a}$ and
  $(a+2b-1)/a=p$,
  \begin{equation}
    \sum_{\lambda_i \le \mu} \lambda_i \beta_i^2
    =c_\lambda c_\beta^2
    \sum_{i=N_\mu+1}^d i^{-(a+2b)}
    \asymp \mu^p.
  \end{equation}
  This proves \cref{phased:eq:teacher-tail-two-sided}.

  Finally, under the same bounds on $\mu$, we have $N_\mu\asymp\mu^{-1/a}$ and $1\le N_\mu\le d/2$. Therefore, since $(1-2b)/a=1-p$,
  \begin{equation}
    \|\boldsymbol P_{>\mu} \boldsymbol\beta\|^2
    = c_\beta^2 \sum_{i=1}^{N_\mu} i^{-2b}
    \asymp
    \begin{cases}
      \mu^{-(1-p)},            & p < 1, \\
      1 + \log(c_\lambda/\mu), & p = 1, \\
      1,                       & p > 1.
    \end{cases}
  \end{equation}
\end{proof}

\subsection{Loss Bounds via Spectral Cutoffs}

We now combine the preceding spectral bound for $\boldsymbol\beta$ with a bound on $r(\tau)$ to control the low-eigenvalue contribution to the loss. For $\mu>0$, define
\begin{equation}
  \mathcal L_{\le\mu}(\tau)
  :=
  \frac12
  \|\boldsymbol P_{\le\mu}
  \boldsymbol\Sigma^{1/2}
  (\boldsymbol\beta-\boldsymbol\theta(\tau))\|^2.
\end{equation}
The triangle inequality, the bound $\lambda_i\le\mu$ on the range of $\boldsymbol P_{\le\mu}$, and $r(\tau)=\|\boldsymbol\theta(\tau)\|$ give
\begin{equation}
  \begin{split}
    \mathcal L_{\le\mu}(\tau)
    & \le
    \|\boldsymbol P_{\le\mu}
    \boldsymbol\Sigma^{1/2}\boldsymbol\beta\|^2
    +\|\boldsymbol P_{\le\mu}
    \boldsymbol\Sigma^{1/2}\boldsymbol\theta(\tau)\|^2 \\
    & \le
    \|\boldsymbol P_{\le\mu}
    \boldsymbol\Sigma^{1/2}\boldsymbol\beta\|^2
    +\mu r(\tau)^2.
  \end{split}
  \label{phased:eq:coarse-low-loss}
\end{equation}

\begin{lemma}
  \label{phased:lem:loss-propagation}
  Let $c_K>0$ and $0\le\tau_0\le\tau_1$. Assume that $\overline{\boldsymbol K}_L(\tau)\succeq c_K\boldsymbol I$ throughout $[\tau_0,\tau_1]$. Then, for any positive continuous function $\mu\colon[\tau_0,\tau_1]\to(0,\infty)$ and every $\tau\in[\tau_0,\tau_1]$,
  \begin{equation}
    \begin{split}
      \mathcal L(\tau)
      & \le
      \exp\!\left(-2c_K \int_{\tau_0}^\tau \mu(w)\,dw \right)
      \mathcal L(\tau_0)
      +
      2c_K \int_{\tau_0}^\tau
      \exp\!\left( -2c_K \int_v^\tau \mu(w)\,dw \right)
      \mu(v) \mathcal L_{\le\mu(v)}(v)\,dv.
    \end{split}
    \label{phased:eq:loss-propagation}
  \end{equation}
  In particular, for any fixed $\alpha > 0$, setting $\mu(v)=\alpha/\{2c_K(1+v)\}$ gives
  \begin{equation}
    \begin{split}
      \mathcal L(\tau)
      & \le
      \left(\frac{1+\tau_0}{1+\tau}\right)^\alpha \mathcal L(\tau_0)
      +
      \alpha (1+\tau)^{-\alpha}
      \int_{\tau_0}^\tau
      (1+v)^{\alpha-1}\mathcal L_{\le\mu(v)}(v)\,dv.
    \end{split}
    \label{phased:eq:polynomial-loss-propagation}
  \end{equation}
\end{lemma}

\begin{proof}
  For $\tau_0 \le v \le \tau_1$, the identity $\boldsymbol g(v)=\boldsymbol\Sigma(\boldsymbol\beta-\boldsymbol\theta(v))$ gives
  \begin{equation}
    \begin{split}
      \|\boldsymbol g(v)\|^2
      & =
      \|\boldsymbol\Sigma
      (\boldsymbol\beta-\boldsymbol\theta(v))\|^2 \\
      & \ge
      \mu(v)\|\boldsymbol P_{>\mu(v)}
      \boldsymbol\Sigma^{1/2}
      (\boldsymbol\beta-\boldsymbol\theta(v))\|^2 \\
      & =
      2\mu(v)\left(
      \mathcal L(v)-\mathcal L_{\le\mu(v)}(v)
      \right).
    \end{split}
  \end{equation}
  Therefore,
  \begin{equation}
    \begin{split}
      \frac{d}{dv}\mathcal L(v)
      & = -\boldsymbol g(v)^\top
      \overline{\boldsymbol K}_L(v) \boldsymbol g(v) \\
      & \le -c_K \|\boldsymbol g(v)\|^2            \\
      & \le -2c_K \mu(v)
      \left(\mathcal L(v) - \mathcal L_{\le\mu(v)}(v)\right).
    \end{split}
  \end{equation}
  Multiplying the resulting differential inequality by the integrating factor $\exp\left(2c_K\int_{\tau_0}^v \mu(w)\,dw\right)$ and integrating from $\tau_0$ to $\tau$ yields \cref{phased:eq:loss-propagation}.
  The polynomial form \cref{phased:eq:polynomial-loss-propagation} follows by substituting the specified $\mu(v)$.
\end{proof}

If $\overline{\boldsymbol K}_L$ remains uniformly positive definite, the smallest eigenvalue $\lambda_d$ of $\boldsymbol{\Sigma}$ gives the following exponential bound beyond $\tau_+$.

\begin{corollary}
  \label{phased:cor:exponential-loss-decay}
  Suppose that there exist constants $\tau_0\ge0$ and $c_K>0$, independent of $d$, such that, for all sufficiently large $d$, $\overline{\boldsymbol K}_L(\tau)\succeq c_K\boldsymbol I$ for every $\tau\ge\tau_0$.
  Then, for all sufficiently large $d$ and every $\tau\ge\tau_+$,
  \begin{equation}
    \mathcal L(\tau)\le
    \mathcal L(\tau_+)
    \exp\!\left(-2c_Kc_\lambda d^{-a}(\tau-\tau_+)\right).
  \end{equation}
\end{corollary}

\begin{proof}
  By \cref{commonwork:eq:tau-loss-dynamics} and the assumed lower bound on $\overline{\boldsymbol K}_L$, for every $v\ge\tau_0$,
  \begin{equation}
    \begin{split}
      \frac{d}{dv}\mathcal L(v)
      &=-\boldsymbol g(v)^\top
      \overline{\boldsymbol K}_L(v)\boldsymbol g(v) \\
      &\le -c_K\|\boldsymbol g(v)\|^2 \\
      &\le -c_K\lambda_d
      \|\boldsymbol\Sigma^{1/2}
      (\boldsymbol\beta-\boldsymbol\theta(v))\|^2 \\
      &=-2c_K\lambda_d\mathcal L(v).
    \end{split}
  \end{equation}
  For all sufficiently large $d$, we have $\tau_+\ge\tau_0$. Applying Grönwall's inequality on $[\tau_+,\tau]$ and using $\lambda_d=c_\lambda d^{-a}$ proves the claim.
\end{proof}

For sufficiently small initialization, \cref{phased:cor:post-entry-radius-kernel} shows that the preceding corollary applies with $\tau_0=\tau_{\mathrm e}$. Hence the loss decays exponentially beyond $\tau_+=c_0d^a$.
For $s=1$, \cref{phased:lem:fixed-scale-radius-kernel} shows that it applies with $\tau_0=0$, so the same conclusion holds.

\section{Predictor-Norm Scaling in Reparameterized Time}
\label{phased:sec:radius-law}

Fix a constant $\tau_\circ>0$, independent of $d$ and $s$. For all sufficiently large $d$, we have $\tau_\circ\le\tau_+$. Since $\tau_\circ$ is fixed, \cref{phased:lem:active-mass-profile} gives $S_d(\tau_\circ)\asymp1$.
To treat small and standard initialization in a unified way, we establish predictor-norm and loss bounds under the following conditions.
Suppose that the initial predictor scale satisfies
\begin{equation}
  \sigma(\tau_\circ)
  \lesssim S_d(\tau_\circ)^{1/(2L)}.
  \label{phased:eq:common-initial-radius-upper}
\end{equation}
Assume also that the initial loss satisfies
\begin{equation}
  \mathcal L(\tau_\circ)
  \lesssim \frac{S_d(\tau_\circ)}{1+\tau_\circ}.
  \label{phased:eq:common-initial-loss-upper}
\end{equation}
For $\tau_\circ\le\tau\le\tau_+$, we also assume the uniform lower bounds
\begin{equation}
  \sigma(\tau)\gtrsim 1
  \label{phased:eq:common-radius-lower}
\end{equation}
and
\begin{equation}
  \overline{\boldsymbol K}_L(\tau)\succeq c_K\boldsymbol I.
  \label{phased:eq:common-kernel-lower}
\end{equation}

These conditions are verified in the proofs of \cref{cor:small-init-predictor-norm-scaling,cor:standard-init-scaling}.
The following lemma derives a loss bound from an upper bound on $\sigma$.

\begin{lemma}
  \label{phased:lem:loss-under-radius-bound}
  Let $C_+\ge1$ and $\tau_1\in[\tau_\circ,\tau_+]$.
  Suppose that \cref{phased:eq:common-initial-loss-upper} holds and that, for $\tau_\circ\le\tau\le\tau_1$, $\overline{\boldsymbol K}_L(\tau)\succeq c_K\boldsymbol I$ and
  \begin{equation}
    \sigma(\tau)\le C_+S_d(\tau)^{1/(2L)}.
    \label{phased:eq:stopped-radius-upper}
  \end{equation}
  Then, for $\tau_\circ\le\tau\le\tau_1$,
  \begin{equation}
    \mathcal L(\tau)
    \le
    C(1+C_+^{2L})\frac{S_d(\tau)}{1+\tau}.
    \label{phased:eq:coarse-stopped-loss}
  \end{equation}
  Here $C$ is independent of $C_+$ and $\tau_1$.
\end{lemma}

\begin{proof}
  For $\tau_\circ\le v\le\tau_1$, set
  \begin{equation}
    \mu(v):=\frac{p+1}{2c_K(1+v)}.
  \end{equation}
  Since $\tau_\circ>0$ is fixed, the monotonicity of $S_d$ and \cref{phased:eq:active-mass-profile-low,phased:eq:active-mass-profile} give, for $\tau_\circ\le v\le\tau_1$,
  \begin{equation}
    (1+v)^{-p}\lesssim\frac{S_d(v)}{1+v}.
  \end{equation}
  Using the upper bound in \cref{phased:lem:powerlaw-estimates} together with \cref{phased:eq:coarse-low-loss,phased:eq:stopped-radius-upper}, we obtain
  \begin{equation}
    \mathcal L_{\le\mu(v)}(v)
    \le C(1+v)^{-p}
    +C C_+^{2L}\frac{S_d(v)}{1+v}
    \le C(1+C_+^{2L})\frac{S_d(v)}{1+v}.
  \end{equation}
  Using \cref{phased:eq:polynomial-loss-propagation} with
  $\tau_0=\tau_\circ$ and $\alpha=p+1$ gives
  \begin{equation}
    \mathcal L(\tau)
    \le
    \left(\frac{1+\tau_\circ}{1+\tau}\right)^{p+1}
    \mathcal L(\tau_\circ)
    +
    C(1+C_+^{2L})(1+\tau)^{-(p+1)}
    \int_{\tau_\circ}^\tau(1+v)^{p-1}S_d(v)\,dv.
  \end{equation}
  By \cref{phased:eq:common-initial-loss-upper} and the monotonicity of
  $S_d$,
  \begin{equation}
    \left(\frac{1+\tau_\circ}{1+\tau}\right)^{p+1}
    \mathcal L(\tau_\circ)
    \lesssim
    \left(\frac{1+\tau_\circ}{1+\tau}\right)^p
    \frac{S_d(\tau_\circ)}{1+\tau}
    \le
    \frac{S_d(\tau)}{1+\tau}.
  \end{equation}
  The same monotonicity and $p>0$ give
  \begin{equation}
    (1+\tau)^{-(p+1)}
    \int_{\tau_\circ}^\tau (1+v)^{p-1} S_d(v)\,dv
    \le
    S_d(\tau)(1+\tau)^{-(p+1)}
    \int_{\tau_\circ}^\tau (1+v)^{p-1}\,dv
    \lesssim \frac{S_d(\tau)}{1+\tau}.
  \end{equation}
  Combining these bounds proves \cref{phased:eq:coarse-stopped-loss}.
\end{proof}

Write $\psi(t):=1+\log(1+t)$ and define
\begin{equation}
  \psi_p(t)
  :=
  \begin{cases}
    1,                 & p < 1, \\
    \psi(t)^{(L-1)/L}, & p = 1, \\
    \psi(t),           & p > 1.
  \end{cases}
  \label{phased:eq:radius-regime-factor}
\end{equation}

\begin{proposition}
  \label{phased:prop:common-radius-law}
  Suppose that
  \cref{phased:eq:common-initial-radius-upper,phased:eq:common-initial-loss-upper,phased:eq:common-radius-lower,phased:eq:common-kernel-lower}
  hold. Assume in addition that, for some fixed $\bar\varepsilon>0$,
  \begin{equation}
    \varepsilon\psi_p(\tau_+)\le\bar\varepsilon.
    \label{phased:eq:radius-bootstrap-defect-bound}
  \end{equation}
  Then, for $\tau_\circ\le\tau\le\tau_+$,
  \begin{equation}
    \sigma(\tau)\asymp S_d(\tau)^{1/(2L)}.
    \label{phased:eq:common-radius-law}
  \end{equation}
  Moreover, throughout the same interval,
  \begin{equation}
    \mathcal L(\tau)\lesssim\frac{S_d(\tau)}{1+\tau}.
    \label{phased:eq:common-loss-upper}
  \end{equation}
  In particular,
  \begin{equation}
    \mathcal L(\tau)
    \lesssim
    \begin{cases}
      (1+\tau)^{-p},           & p < 1, \\
      \psi(\tau)(1+\tau)^{-1}, & p = 1, \\
      (1+\tau)^{-1},           & p > 1.
    \end{cases}
    \label{phased:eq:coarse-loss-regime-split}
  \end{equation}
  The implicit constants may depend on $\bar\varepsilon$, $\tau_\circ$, and the constants in \cref{phased:eq:common-initial-radius-upper,phased:eq:common-initial-loss-upper,phased:eq:common-radius-lower,phased:eq:common-kernel-lower}, but are independent of $d$, $s$, and $\tau$.
\end{proposition}

\begin{proof}

  \step{Upper bound for $\sigma(\tau)$.}

  By \cref{phased:eq:common-initial-radius-upper}, we may choose $A\ge1$ large
  enough that $\sigma(\tau_\circ)<A S_d(\tau_\circ)^{1/(2L)}$.
  We will impose further lower bounds on $A$ below.
  Suppose for contradiction that $\sigma(\tau)\ge2A S_d(\tau)^{1/(2L)}$ for some $\tau\in[\tau_\circ,\tau_+]$, and let
  \begin{equation}
    \tau^\dagger
    :=
    \min\left\{
    \tau\in[\tau_\circ,\tau_+]:
    \sigma(\tau)\ge
    2A S_d(\tau)^{1/(2L)}
    \right\}.
  \end{equation}
  Then $\tau^\dagger>\tau_\circ$ and, by continuity, $\sigma(\tau^\dagger)=2A S_d(\tau^\dagger)^{1/(2L)}$.
  Define
  \begin{equation}
    \tau_-
    :=
    \max\left\{
    \tau\in[\tau_\circ,\tau^\dagger]:
    \sigma(\tau)\le A S_d(\tau)^{1/(2L)}
    \right\}.
  \end{equation}
  The defining set is nonempty and closed.
  By maximality and continuity, $\sigma(\tau_-)=A S_d(\tau_-)^{1/(2L)}$.
  The definitions of $\tau_-$ and $\tau^\dagger$ imply that, for $\tau_-\le\tau\le\tau^\dagger$,
  \begin{equation}
    A S_d(\tau)^{1/(2L)}
    \le\sigma(\tau)\le
    2A S_d(\tau)^{1/(2L)}.
    \label{phased:eq:intermediate-radius-bounds}
  \end{equation}

  By \cref{phased:eq:radius-bootstrap-defect-bound}, $\varepsilon\le\bar\varepsilon$. Moreover, the monotonicity of $S_d$ and \cref{phased:lem:active-mass-profile} give $S_d(\tau)\ge S_d(\tau_\circ)\gtrsim1$ on this interval. Thus, after increasing $A$ if necessary, depending only on $\bar\varepsilon$ and this fixed lower bound, we have, throughout the same interval,
  \begin{equation}
    \frac{\varepsilon}{\sigma(\tau)^2}
    \le
    \frac{\varepsilon}{A^2S_d(\tau)^{1/L}}
    \le1.
    \label{phased:eq:intermediate-defect-ratio}
  \end{equation}
  Hence \cref{commonwork:prop:perturbative-kernel-expansion} applies on this interval with $M=1$, and the resulting constants are independent of $A$.

  By the definition of $\tau^\dagger$, \cref{phased:eq:stopped-radius-upper} holds with $C_+=2A$ up to this time.
  Therefore, \cref{phased:lem:loss-under-radius-bound} gives, for $\tau_\circ\le\tau\le\tau^\dagger$,
  \begin{equation}
    \mathcal L(\tau)
    \le C A^{2L}\frac{S_d(\tau)}{1+\tau}.
    \label{phased:eq:stopped-loss-bound}
  \end{equation}

  Applying \cref{phased:eq:variation-constants-formula} from $\tau_-$, we obtain
  \begin{equation}
    \sigma(\tau^\dagger)
    \le
    \left\|
    e^{-(\tau^\dagger-\tau_-)\boldsymbol\Sigma}\boldsymbol z(\tau_-)
    \right\|
    +
    \left\|
    \int_{\tau_-}^{\tau^\dagger}
    e^{-(\tau^\dagger-v)\boldsymbol\Sigma}
    \sigma(v)^{-(L-1)}
    \boldsymbol\Sigma\boldsymbol\beta\,dv
    \right\|
    +
    \left\|
    \int_{\tau_-}^{\tau^\dagger}
    e^{-(\tau^\dagger-v)\boldsymbol\Sigma}
    \boldsymbol f_K(v)\,dv
    \right\|.
    \label{phased:eq:radius-integral-bound}
  \end{equation}
  We now estimate each term separately.

  \substep{The term involving $\boldsymbol z(\tau_-)$.}

  \begin{equation}
    \|e^{-(\tau^\dagger-\tau_-)\boldsymbol\Sigma}\boldsymbol z(\tau_-)\|
    \le \sigma(\tau_-)
    \le A S_d(\tau^\dagger)^{1/(2L)}.
    \label{phased:eq:homogeneous-term-bound}
  \end{equation}

  \substep{The term involving $\boldsymbol\beta$.}

  From \cref{phased:eq:intermediate-radius-bounds}, for
  $\tau_-\le\tau\le\tau^\dagger$ we have
  \begin{equation}
    \sigma(\tau)^{-(L-1)}
    \le A^{-(L-1)}S_d(\tau)^{-(L-1)/(2L)}.
  \end{equation}
  If $\tau^\dagger\le4$, then $S_d(\tau)\gtrsim1$ for
  $\tau_-\le\tau\le\tau^\dagger$.
  Therefore, since $\|e^{-u\boldsymbol\Sigma}\|_{\mathrm{op}}\le1$ for $u\ge0$ and $\|\boldsymbol\Sigma\boldsymbol\beta\|\lesssim1$, we obtain
  \begin{equation}
    \left\|
    \int_{\tau_-}^{\tau^\dagger}
    e^{-(\tau^\dagger-v)\boldsymbol\Sigma}
    \sigma(v)^{-(L-1)}
    \boldsymbol\Sigma \boldsymbol\beta\, dv
    \right\|
    \le C A^{-(L-1)}S_d(\tau^\dagger)^{1/(2L)}.
  \end{equation}

  Next, we consider the case $\tau^\dagger > 4$.

  First, consider $v \in [\max\{\tau_-, \tau^\dagger/2\}, \tau^\dagger]$. In this interval, $v \asymp \tau^\dagger$, so by \cref{phased:eq:active-mass-profile}, $S_d(v) \asymp S_d(\tau^\dagger)$. Therefore,
  \begin{equation}
    \begin{split}
      \left\|
      \int_{\max\{\tau_-,\tau^\dagger/2\}}^{\tau^\dagger}
      e^{-(\tau^\dagger-v)\boldsymbol\Sigma}
      \sigma(v)^{-(L-1)}
      \boldsymbol\Sigma\boldsymbol\beta\,dv
      \right\|^2
      & \le C A^{-2(L-1)}S_d(\tau^\dagger)^{-(L-1)/L} \sum_{i=1}^d\beta_i^2\min\{\tau^\dagger\lambda_i,1\}^2 \\
      & \le C A^{-2(L-1)}S_d(\tau^\dagger)^{1/L}.
    \end{split}
    \label{phased:eq:recent-teacher-integral-bound}
  \end{equation}

  Next, suppose $\tau_- \le \tau^\dagger/2$ and consider $v\in[\tau_-,\tau^\dagger/2]$. In this range, $e^{-(\tau^\dagger-v)\lambda_i}\le e^{-\tau^\dagger\lambda_i/2}$.

  By \cref{phased:eq:active-mass-profile}, for $p < 1$,
  \begin{equation}
    \int_{\tau_\circ}^{\tau^\dagger/2} S_d(v)^{-(L-1)/(2L)}\, dv
    \lesssim
    \int_{\tau_\circ}^{\tau^\dagger/2}
    (1+v)^{-(L-1)(1-p)/(2L)}\, dv
    \lesssim
    \tau^\dagger(1+\tau^\dagger)^{-(L-1)(1-p)/(2L)}.
  \end{equation}
  For $p=1$,
  \begin{equation}
    \begin{split}
      \int_{\tau_\circ}^{\tau^\dagger/2} S_d(v)^{-(L-1)/(2L)}\, dv
      & \lesssim \int_{\tau_\circ}^{\sqrt{\tau^\dagger}}\psi(v)^{-(L-1)/(2L)}\, dv + \int_{\sqrt{\tau^\dagger}}^{\tau^\dagger/2}\psi(v)^{-(L-1)/(2L)}\, dv \\
      & \lesssim \sqrt{\tau^\dagger} + \tau^\dagger \psi(\tau^\dagger)^{-(L-1)/(2L)}                                                                       \\
      & \lesssim \tau^\dagger S_d(\tau^\dagger)^{-(L-1)/(2L)}.
    \end{split}
  \end{equation}
  For $p>1$,
  \begin{equation}
    \int_{\tau_\circ}^{\tau^\dagger/2} S_d(v)^{-(L-1)/(2L)}\, dv
    \lesssim \tau^\dagger.
  \end{equation}
  Therefore,
  \begin{equation}
    \int_{\tau_\circ}^{\tau^\dagger/2}
    S_d(v)^{-(L-1)/(2L)}\, dv
    \lesssim \tau^\dagger S_d(\tau^\dagger)^{-(L-1)/(2L)}.
  \end{equation}
  Moreover,
  \begin{equation}
    (\tau^\dagger)^2\sum_{i=1}^d
    \beta_i^2\lambda_i^2 e^{-\tau^\dagger\lambda_i}
    \lesssim \sum_{i=1}^d \beta_i^2 \min\{\tau^\dagger\lambda_i, 1\}^2
    \lesssim S_d(\tau^\dagger).
  \end{equation}
  Hence,
  \begin{equation}
    \label{phased:eq:early-teacher-integral-bound}
    \begin{split}
      \left\|
      \int_{\tau_-}^{\tau^\dagger/2}
      e^{-(\tau^\dagger-v)\boldsymbol\Sigma}
      \sigma(v)^{-(L-1)}
      \boldsymbol\Sigma \boldsymbol\beta\, dv
      \right\|^2
      & \le C A^{-2(L-1)}
      \left( \tau^\dagger S_d(\tau^\dagger)^{-(L-1)/(2L)} \right)^2
      \sum_{i=1}^d \beta_i^2\lambda_i^2 e^{-\tau^\dagger\lambda_i} \\
      & \le C A^{-2(L-1)}S_d(\tau^\dagger)^{1/L}.
    \end{split}
  \end{equation}

  Combining \cref{phased:eq:recent-teacher-integral-bound,phased:eq:early-teacher-integral-bound} when $\tau_-\le\tau^\dagger/2$, and using \cref{phased:eq:recent-teacher-integral-bound} alone otherwise, we obtain
  \begin{equation}
    \left\|
    \int_{\tau_-}^{\tau^\dagger}
    e^{-(\tau^\dagger-v)\boldsymbol\Sigma}
    \sigma(v)^{-(L-1)}
    \boldsymbol\Sigma\boldsymbol\beta\,dv
    \right\|
    \le C A^{-(L-1)}S_d(\tau^\dagger)^{1/(2L)}.
    \label{phased:eq:teacher-convolution-upper}
  \end{equation}

  \substep{The term involving $\boldsymbol f_K$.}

  By \cref{phased:eq:forcing-bound,phased:eq:intermediate-defect-ratio}, for $\tau_-\le\tau\le\tau^\dagger$,
  \begin{equation}
    \|\boldsymbol f_K(\tau)\|
    \le
    C\varepsilon\sigma(\tau)^{-(L+1)}\|\boldsymbol g(\tau)\|.
    \label{phased:eq:bootstrap-forcing-bound}
  \end{equation}
  For $j\ge0$, set $x_j:=2^j(1+\tau_-)-1$, and let $J:=\max\{j\ge0:x_j\le\tau^\dagger\}$. For $0\le j\le J$, the Cauchy--Schwarz inequality and \cref{commonwork:eq:tau-loss-dynamics,phased:eq:common-kernel-lower,phased:eq:stopped-loss-bound} give
  \begin{equation}
    \begin{split}
      \int_{x_j}^{\min\{x_{j+1},\tau^\dagger\}}
      \|\boldsymbol g(v)\|\,dv
      & \le
      (1+x_j)^{1/2}
      \left\{
      \int_{x_j}^{\min\{x_{j+1},\tau^\dagger\}}
      \|\boldsymbol g(v)\|^2\,dv
      \right\}^{1/2} \\
      & \le
      c_K^{-1/2}(1+x_j)^{1/2}
      \left\{
      \mathcal L(x_j)
      -\mathcal L(\min\{x_{j+1},\tau^\dagger\})
      \right\}^{1/2} \\
      & \le
      C A^L S_d(x_j)^{1/2}.
    \end{split}
  \end{equation}
  Since $S_d$ is nondecreasing, \cref{phased:eq:intermediate-radius-bounds,phased:eq:bootstrap-forcing-bound} now imply
  \begin{equation}
    \begin{split}
      \int_{x_j}^{\min\{x_{j+1},\tau^\dagger\}}
      \|\boldsymbol f_K(v)\|\,dv
      & \le
      C\varepsilon A^{-(L+1)}S_d(x_j)^{-(L+1)/(2L)}
      \int_{x_j}^{\min\{x_{j+1},\tau^\dagger\}}
      \|\boldsymbol g(v)\|\,dv \\
      & \le
      C\frac{\varepsilon}{A}S_d(x_j)^{-1/(2L)}.
    \end{split}
    \label{phased:eq:dyadic-forcing-bound}
  \end{equation}
  These intervals cover $[\tau_-,\tau^\dagger]$. We next use \cref{phased:lem:active-mass-profile} and $1+x_j=2^j(1+\tau_-)$.

  For $p<1$,
  \begin{equation}
    \begin{split}
      \sum_{j=0}^J S_d(x_j)^{-1/(2L)}
      & \lesssim
      \sum_{j=0}^J(1+x_j)^{-(1-p)/(2L)} \\
      & \lesssim
      \sum_{j=0}^J2^{-j(1-p)/(2L)}      \\
      & \lesssim 1
      \lesssim S_d(\tau^\dagger)^{1/(2L)}.
    \end{split}
  \end{equation}

  For $p=1$, we have $S_d(x_j)\gtrsim1+j$ and $J+1\lesssim\psi(\tau^\dagger)$. Therefore,
  \begin{equation}
    \begin{split}
      \sum_{j=0}^J S_d(x_j)^{-1/(2L)}
      & \lesssim
      \sum_{j=0}^J(1+j)^{-1/(2L)}   \\
      & \lesssim
      (1+J)^{1-1/(2L)}              \\
      & \lesssim
      \psi(\tau^\dagger)^{1-1/(2L)} \\
      & \lesssim
      S_d(\tau^\dagger)^{1/(2L)}
      \psi(\tau^\dagger)^{(L-1)/L}.
    \end{split}
  \end{equation}

  For $p>1$, we have $S_d(x_j)\asymp1$ and $J+1\lesssim\psi(\tau^\dagger)$. Hence,
  \begin{equation}
    \sum_{j=0}^J S_d(x_j)^{-1/(2L)}
    \lesssim J+1
    \lesssim\psi(\tau^\dagger)
    \lesssim
    S_d(\tau^\dagger)^{1/(2L)}\psi(\tau^\dagger).
  \end{equation}

  Combining the three cases gives
  \begin{equation}
    \sum_{j=0}^J
    S_d(x_j)^{-1/(2L)}
    \le
    C S_d(\tau^\dagger)^{1/(2L)}\psi_p(\tau^\dagger).
    \label{phased:eq:forcing-regime-sum}
  \end{equation}
  Since $\psi_p$ is nondecreasing and $\tau^\dagger\le\tau_+$, \cref{phased:eq:radius-bootstrap-defect-bound} gives $\varepsilon\psi_p(\tau^\dagger)\le\bar\varepsilon$.
  Summing \cref{phased:eq:dyadic-forcing-bound} over $j$ and using \cref{phased:eq:forcing-regime-sum}, we obtain
  \begin{equation}
    \int_{\tau_-}^{\tau^\dagger}\|\boldsymbol f_K(v)\|\,dv
    \le
    C\frac{\varepsilon}{A}\psi_p(\tau^\dagger)S_d(\tau^\dagger)^{1/(2L)}\le
    C\frac{\bar\varepsilon}{A}S_d(\tau^\dagger)^{1/(2L)}.
    \label{phased:eq:normalized-forcing-regime-sums}
  \end{equation}

  \substep{Completion of the upper bound.}

  Substituting \cref{phased:eq:homogeneous-term-bound,phased:eq:teacher-convolution-upper,phased:eq:normalized-forcing-regime-sums} into \cref{phased:eq:radius-integral-bound} gives
  \begin{equation}
    \sigma(\tau^\dagger)= 2A S_d(\tau^\dagger)^{1/(2L)}
    \le
    \left(
    1+C A^{-L}+C\bar\varepsilon A^{-2}
    \right)A S_d(\tau^\dagger)^{1/(2L)}.
  \end{equation}
  Fixing $A$ sufficiently large so that $CA^{-L}+C\bar\varepsilon A^{-2}<1/2$ makes the right-hand side strictly smaller than $2A S_d(\tau^\dagger)^{1/(2L)}$, a contradiction. Thus the set defining $\tau^\dagger$ is empty. Hence, for $\tau_\circ\le\tau\le\tau_+$,
  \begin{equation}
    \sigma(\tau)\le 2A S_d(\tau)^{1/(2L)}.
    \label{phased:eq:closed-radius-upper}
  \end{equation}
  Applying \cref{phased:lem:loss-under-radius-bound} with $C_+=2A$ and $\tau_1=\tau_+$ to \cref{phased:eq:closed-radius-upper} gives
  \begin{equation}
    \mathcal L(\tau)
    \le C\frac{S_d(\tau)}{1+\tau}.
  \end{equation}
  This proves \cref{phased:eq:common-loss-upper}.

  \step{Lower bound for $\sigma(\tau)$.}

  Let $B$ be a fixed constant chosen large enough so that $B>2c_\mu c_0$, and set $\mu(\tau):=B/(1+\tau)$. Then,
  \begin{equation}
    \|\boldsymbol P_{>\mu(\tau)}
    (\boldsymbol\beta-\boldsymbol\theta(\tau))\|^2
    \le
    \frac{2}{\mu(\tau)}\mathcal L(\tau)
    = \frac{2(1+\tau)}{B} \mathcal L(\tau)
    \le \frac{C}{B} S_d(\tau).
  \end{equation}
  Define
  \begin{equation}
    \tau_B :=
    \begin{cases}
      \displaystyle
      \max\left\{
      \tau_\circ,
      c_\lambda^{-1},
      \frac{B}{\mu_\star} - 1,
      \left(\frac{B}{c_\lambda}\right)^2 - 1
      \right\}, & p=1,    \\
      \displaystyle
      \max\left\{
      \tau_\circ,
      c_\lambda^{-1},
      \frac{B}{\mu_\star} - 1
      \right\},
                & p\ne 1.
    \end{cases}
  \end{equation}
  For $\tau_B\le\tau\le\tau_+$, the definition of $\tau_B$ gives $\mu(\tau)\le\mu_\star$. Moreover, since $\tau\le\tau_+=c_0d^a$ and $B>2c_\mu c_0$, for all sufficiently large $d$ we have $\mu(\tau)\ge B/(1+\tau_+)\ge c_\mu d^{-a}$.
  Therefore, \cref{phased:lem:powerlaw-estimates} applies with $\mu(\tau) = B/(1+\tau)$ and gives
  \begin{equation}
    \|\boldsymbol P_{>\mu(\tau)}\boldsymbol\beta\|^2
    \gtrsim
    \begin{cases}
      B^{-(1-p)} (1+\tau)^{1-p},
       & p < 1, \\
      1 + \log\dfrac{c_\lambda(1+\tau)}{B},
       & p = 1, \\
      1,
       & p > 1.
    \end{cases}
  \end{equation}
  In the case $p=1$, since $\tau\ge \tau_B$, we have $1+\tau \ge (B/c_\lambda)^2$ and thus
  \begin{equation}
    1 + \log\frac{c_\lambda(1+\tau)}{B}
    \ge \frac{1}{2}\{1 + \log(1+\tau)\}.
  \end{equation}
  Therefore, by \cref{phased:eq:active-mass-profile},
  \begin{equation}
    \|\boldsymbol P_{>\mu(\tau)}\boldsymbol\beta\|^2
    \ge
    \begin{cases}
      cB^{-(1-p)}S_d(\tau), & p<1,    \\
      cS_d(\tau),           & p\ge 1.
    \end{cases}
  \end{equation}
  Thus, by choosing $B$ sufficiently large,
  \begin{equation}
    \begin{split}
      \sigma(\tau)^L
      & = \|\boldsymbol\theta(\tau)\|                              \\
      & \ge \|\boldsymbol P_{>\mu(\tau)}\boldsymbol\theta(\tau)\| \\
      & \ge
      \|\boldsymbol P_{>\mu(\tau)}\boldsymbol\beta\|
      - \|\boldsymbol P_{>\mu(\tau)}
      (\boldsymbol\beta-\boldsymbol\theta(\tau))\|                  \\
      & \ge
      \begin{cases}
        (c B^{-(1-p)/2} - C B^{-1/2})S_d(\tau)^{1/2}, & p<1,   \\
        (c - C B^{-1/2})S_d(\tau)^{1/2},              & p\ge 1
      \end{cases}
      \\
      & \ge c S_d(\tau)^{1/2}.
    \end{split}
  \end{equation}
  For $p<1$, choose $B$ so that $CB^{-p/2}\le c/2$; for $p\ge1$, choose $B$ so that $CB^{-1/2}\le c/2$. With $B$ thus fixed, the last inequality follows after adjusting $c$.
  Hence $\sigma(\tau)\ge cS_d(\tau)^{1/(2L)}$ for $\tau_B\le\tau\le\tau_+$.

  On the remaining fixed interval $\tau_\circ\le\tau\le\min\{\tau_B,\tau_+\}$, the monotonicity of $S_d$ and \cref{phased:eq:active-mass-profile} give $S_d(\tau)\lesssim1$ because $\tau_B$ is fixed, while \cref{phased:eq:common-radius-lower} gives $\sigma(\tau)\gtrsim1$.
  After decreasing $c$ if necessary, the lower bound holds throughout $[\tau_\circ,\tau_+]$. Together, this lower bound and \cref{phased:eq:closed-radius-upper} prove \cref{phased:eq:common-radius-law}.
  Finally, combining \cref{phased:eq:common-loss-upper,phased:eq:active-mass-profile} proves \cref{phased:eq:coarse-loss-regime-split}.
\end{proof}

The following corollary gives a precise version of
\cref{prop:predictor-norm-scaling}.
\begin{corollary}[Predictor-norm scaling under small initialization]
  \label{cor:small-init-predictor-norm-scaling}
  Consider Gaussian initialization satisfying \cref{phased:ass:small-initialization}. Then there exist constants $c,C,\tau_0>0$, independent of $d$, $s$, and $\tau$, such that, for all sufficiently large $d$, with probability at least $1-Ce^{-cd}$, uniformly for $\tau_0\le\tau\le\tau_+$,
  \begin{equation}
    \|\boldsymbol\theta(\tau)\|^2\asymp
    \begin{cases}
      \tau^{1-p}, & p < 1, \\
      \log\tau,   & p = 1, \\
      1,          & p > 1.
    \end{cases}
  \end{equation}
\end{corollary}

\begin{proof}
  Work on the event in \cref{commonwork:lem:fixed-s-isotropy}, which has probability at least $1-Ce^{-cd}$, and set $\tau_\circ=\tau_{\mathrm e}$.
  \Cref{phased:eq:active-mass-profile-low,phased:eq:early-radius} give \cref{phased:eq:common-initial-radius-upper}, while loss monotonicity and \cref{commonwork:lem:fixed-s-isotropy} give \cref{phased:eq:common-initial-loss-upper}.
  The uniform lower bounds \cref{phased:eq:common-radius-lower,phased:eq:common-kernel-lower} follow from \cref{phased:cor:post-entry-radius-kernel}.
  It remains to verify the additional condition \cref{phased:eq:radius-bootstrap-defect-bound} of \cref{phased:prop:common-radius-law}.

  By \cref{commonwork:lem:fixed-s-isotropy}, $\varepsilon\lesssim s^2$. If $p<1$, then $\psi_p(\tau)=1$, so $\varepsilon\psi_p(\tau_+)\lesssim s^2\lesssim\delta$. If $p=1$, then $\psi_p(\tau_+)\lesssim(\log d)^{(L-1)/L}$, and \cref{phased:ass:small-initialization} gives $\varepsilon\psi_p(\tau_+)\lesssim s^2(\log d)^{(L-1)/L}\lesssim\delta$. If $p>1$, then $\psi_p(\tau_+)\lesssim\log d$. Since $a(p-1)>0$, for all sufficiently large $d$, $\log d\lesssim d^{a(p-1)}$, and hence \cref{phased:ass:small-initialization} gives $\varepsilon\psi_p(\tau_+)\lesssim s^2d^{a(p-1)}\lesssim\delta$. Thus, after choosing the threshold $\delta_0$ in \cref{phased:ass:small-initialization} sufficiently small, all the conditions of \cref{phased:prop:common-radius-law} are satisfied.

  \Cref{phased:prop:common-radius-law} therefore gives $\sigma(\tau)\asymp S_d(\tau)^{1/(2L)}$. Since $\|\boldsymbol\theta(\tau)\|=\sigma(\tau)^L$,
  \begin{equation}
    \|\boldsymbol\theta(\tau)\|^2
    =\sigma(\tau)^{2L}\asymp S_d(\tau).
  \end{equation}
  Combining this relation with \cref{phased:lem:active-mass-profile}, and increasing $\tau_0$ by a fixed amount if necessary, proves the corollary.

\end{proof}

\begin{corollary}
  \label{cor:standard-init-scaling}
  Suppose $p<1$ and consider Gaussian initialization with $s=1$. Then there exist constants $c,C,\tau_0>0$, independent of $d$ and $\tau$, such that, for all sufficiently large $d$, with probability at least $1-Ce^{-cd}$, uniformly for $\tau_0\le\tau\le\tau_+$,
  \begin{equation}
    \|\boldsymbol\theta(\tau)\|^2\asymp\tau^{1-p}.
  \end{equation}
\end{corollary}

\begin{proof}
  Work on the event in \cref{commonwork:lem:fixed-s-isotropy}, and set $\tau_\circ=c_\lambda^{-1}$.
  Then \cref{phased:eq:active-mass-profile-low} gives $S_d(\tau_\circ)\asymp1$.
  Loss monotonicity and \cref{commonwork:lem:fixed-s-isotropy} give $\mathcal L(v)\lesssim1$ and hence $\|\boldsymbol g(v)\|\lesssim1$ for $0\le v\le\tau_\circ$.
  Thus, \cref{commonwork:eq:tau-predictor-dynamics,phased:eq:fixed-scale-kernel-bounds} give
  \begin{equation}
    r(\tau_\circ)
    \le r(0)+\int_0^{\tau_\circ}
    \|\overline{\boldsymbol K}_L(v)\boldsymbol g(v)\|\,dv
    \lesssim1.
  \end{equation}
  Hence $\sigma(\tau_\circ)\lesssim1$.
  Since $\tau_\circ$ is fixed and $S_d(\tau_\circ)\asymp1$, this estimate verifies \cref{phased:eq:common-initial-radius-upper}, while the loss bound above verifies \cref{phased:eq:common-initial-loss-upper}.
  The uniform lower bounds \cref{phased:eq:common-radius-lower,phased:eq:common-kernel-lower} follow from \cref{phased:lem:fixed-scale-radius-kernel}.

  Since $p<1$, we have $\psi_p(\tau)=1$, while \cref{commonwork:lem:fixed-s-isotropy} gives $\varepsilon\lesssim1$.
  Hence \cref{phased:eq:radius-bootstrap-defect-bound} holds for some fixed $\bar\varepsilon>0$, independent of $d$.
  \Cref{phased:prop:common-radius-law} therefore yields
  \begin{equation}
    \|\boldsymbol\theta(\tau)\|^2
    =\sigma(\tau)^{2L}\asymp S_d(\tau).
  \end{equation}
  By \cref{phased:lem:active-mass-profile}, $S_d(\tau)\asymp(1+\tau)^{1-p}$, and the result follows after choosing a fixed $\tau_0$ sufficiently large.
\end{proof}

\section{Common Spectral Law for the Loss}
\label{sec:common-spectral-law-for-loss}

\subsection{Difficult Tasks}
\label{phased:sec:hard-loss-law}

For $p<1$, we derive the loss scaling from the predictor-norm and loss bounds established in \cref{phased:prop:common-radius-law}.

\begin{proposition}
  \label{phased:prop:hard-loss-law}
  Suppose that $p<1$ and that the conditions of
  \cref{phased:prop:common-radius-law} hold. Then
  \begin{equation}
    \mathcal L(\tau)\asymp(1+\tau)^{-p}
    \label{phased:eq:hard-conclusion}
  \end{equation}
  for $\tau_\circ\le \tau\le \tau_+$.
\end{proposition}

\begin{proof}
  From \cref{phased:prop:common-radius-law,phased:eq:active-mass-profile}, we have
  \begin{equation}
    S_d(\tau)\asymp (1+\tau)^{1-p}
  \end{equation}
  for $\tau_\circ\le \tau\le \tau_+$. Therefore, from \cref{phased:eq:coarse-loss-regime-split},
  \begin{equation}
    \mathcal L(\tau)\lesssim (1+\tau)^{-p}.
    \label{phased:eq:hard-loss-upper}
  \end{equation}

  Next, we consider the lower bound.
  Fix $\mu_0>0$ sufficiently small. In particular, we take $\mu_0$ so small that $\mu_0/(1+\tau_\circ)\le\mu_\star$. With $c_\mu$ and $\mu_0$ fixed, reduce $c_0$ further if necessary so that $2c_\mu c_0<\mu_0$.
  Set $\mu(\tau):= \mu_0/(1+\tau)$. For any $\tau_\circ\le \tau\le \tau_+$, we have $\mu(\tau) \le \mu_0/(1+\tau_\circ) \le \mu_\star$. Moreover, for sufficiently large $d$, $1+\tau_+ \le 2\tau_+$, so
  \begin{equation}
    \mu(\tau) \ge \frac{\mu_0}{1+\tau_+}
    \ge \frac{\mu_0}{2c_0}d^{-a}
    \ge c_\mu d^{-a}.
  \end{equation}
  Therefore, $c_\mu d^{-a}\le \mu(\tau)\le \mu_\star$. The lower bound in \cref{phased:lem:powerlaw-estimates} gives
  \begin{equation}
    \|\boldsymbol P_{\le\mu(\tau)}
    \boldsymbol\Sigma^{1/2} \boldsymbol\beta\|
    \gtrsim \mu_0^{p/2}(1+\tau)^{-p/2}.
  \end{equation}
  From \cref{phased:prop:common-radius-law} and \cref{phased:eq:active-mass-profile}, for $\tau_\circ\le \tau\le \tau_+$ we have $r(\tau)^2 = \sigma(\tau)^{2L} \lesssim (1+\tau)^{1-p}$, so
  \begin{equation}
    \|\boldsymbol P_{\le\mu(\tau)}
    \boldsymbol\Sigma^{1/2} \boldsymbol\theta(\tau)\|
    \le \mu(\tau)^{1/2} r(\tau)
    \lesssim \mu_0^{1/2} (1+\tau)^{-p/2}.
  \end{equation}
  Consequently,
  \begin{equation}
    \|\boldsymbol P_{\le\mu(\tau)}
    \boldsymbol\Sigma^{1/2}\boldsymbol\theta(\tau)\|
    \lesssim
    \mu_0^{(1-p)/2}
    \|\boldsymbol P_{\le\mu(\tau)}
    \boldsymbol\Sigma^{1/2}\boldsymbol\beta\|.
  \end{equation}
  Since $p<1$, taking $\mu_0$ sufficiently small gives
  \begin{equation}
    \|\boldsymbol P_{\le\mu(\tau)}
    \boldsymbol\Sigma^{1/2}\boldsymbol\theta(\tau)\|
    \le \frac12
    \|\boldsymbol P_{\le\mu(\tau)}
    \boldsymbol\Sigma^{1/2}\boldsymbol\beta\|.
  \end{equation}
  Therefore,
  \begin{equation}
    \begin{split}
      \|\boldsymbol P_{\le\mu(\tau)}
      \boldsymbol\Sigma^{1/2}
      (\boldsymbol\beta-\boldsymbol\theta(\tau))\|
      & \ge
      \|\boldsymbol P_{\le\mu(\tau)}
      \boldsymbol\Sigma^{1/2} \boldsymbol\beta\|
      - \|\boldsymbol P_{\le\mu(\tau)}
      \boldsymbol\Sigma^{1/2} \boldsymbol\theta(\tau)\| \\
      & \ge
      \frac{1}{2}
      \|\boldsymbol P_{\le\mu(\tau)}
      \boldsymbol\Sigma^{1/2} \boldsymbol\beta\|        \\
      & \gtrsim
      \mu_0^{p/2}(1+\tau)^{-p/2}.
    \end{split}
  \end{equation}
  Hence, for $\tau_\circ\le \tau\le \tau_+$,
  \begin{equation}
    \mathcal L(\tau)
    \ge
    \frac12\|\boldsymbol P_{\le\mu(\tau)}
    \boldsymbol\Sigma^{1/2}
    (\boldsymbol\beta-\boldsymbol\theta(\tau))\|^2
    \ge c(1+\tau)^{-p}.
  \end{equation}
\end{proof}

\subsection{Boundary and Easy Tasks}
\label{phased:sec:pgeone-loss-law}

Throughout this subsection, suppose that $p\ge1$ and that \cref{phased:ass:small-initialization} holds.
We choose $0<\delta_0\le s_0^2$ sufficiently small to satisfy the conditions imposed below, and fix an arbitrary $0<\delta\le\delta_0$.
For all sufficiently large $d$, we work on the event in \cref{commonwork:lem:fixed-s-isotropy}, which has probability at least
$1-Ce^{-cd}$.
We use \cref{phased:prop:common-radius-law} with $\tau_\circ=\tau_{\mathrm e}$, as established in the proof of \cref{cor:small-init-predictor-norm-scaling}.

\begin{lemma}
  \label{phased:lem:exact-residual-decomposition}
  Suppose $0<s\le s_0$. Then, for all sufficiently large $d$ and every $\tau_{\mathrm e}\le \tau\le \tau_+$,
  \begin{equation}
    \begin{split}
      \boldsymbol\beta-\boldsymbol\theta(\tau)
      & =
      \left(\boldsymbol I-
      \sigma(\tau)^{L-1}\boldsymbol\Sigma
      \int_0^\tau e^{-(\tau-v)\boldsymbol\Sigma}
      \sigma(v)^{-(L-1)}\,dv\right)\boldsymbol\beta \\
      & \qquad-
      \sigma(\tau)^{L-1}
      e^{-(\tau-\tau_{\mathrm e})\boldsymbol\Sigma}
      \boldsymbol\xi_{\mathrm e}
      -
      \sigma(\tau)^{L-1}
      \int_{\tau_{\mathrm e}}^\tau e^{-(\tau-v)\boldsymbol\Sigma}
      \boldsymbol f_K(v)\,dv.
    \end{split}
    \label{phased:eq:exact-residual-decomposition}
  \end{equation}
\end{lemma}

\begin{proof}
  Substitute \cref{phased:eq:post-entry-remainder} into \cref{phased:eq:post-entry-theta} and rearrange.
\end{proof}

\begin{lemma}
  \label{phased:lem:low-frequency-teacher-response}
  Suppose that $p\ge1$ and that $\sigma(\tau)\asymp(s^L+\tau)^{1/L}$ for $0\le\tau\le\tau_{\mathrm e}$ and $\sigma(\tau)\asymp S_d(\tau)^{1/(2L)}$ for $\tau_{\mathrm e}\le\tau\le\tau_+$. Then, for every $\tau_{\mathrm e}\le\tau\le\tau_+$ and every eigenvalue $\lambda_i$ of $\boldsymbol\Sigma$,
  \begin{equation}
    \begin{split}
      1-C\lambda_i(1+\tau)
      \le
      1-
      \sigma(\tau)^{L-1}\lambda_i
      \int_0^\tau
      e^{-(\tau-v)\lambda_i}
      \sigma(v)^{-(L-1)}\,dv\le 1.
    \end{split}
    \label{phased:eq:low-frequency-teacher-response}
  \end{equation}
\end{lemma}

\begin{proof}
  Note that $e^{-(\tau-v)\lambda_i}\le1$ for all $0\le v\le\tau$.
  Since $\sigma(v)\gtrsim(s^L+v)^{1/L}$ for $0\le v\le\tau_{\mathrm e}$,
  \begin{equation}
    \int_0^{\tau_{\mathrm e}}\sigma(v)^{-(L-1)}\,dv
    \le
    C\int_0^{\tau_{\mathrm e}}
    (s^L+v)^{-(L-1)/L}\,dv
    \le C.
  \end{equation}
  For $p=1$, combining $\sigma(\tau)\asymp S_d(\tau)^{1/(2L)}$ with \cref{phased:lem:active-mass-profile} gives $\sigma(\tau)^{L-1} \asymp \psi(\tau)^{(L-1)/(2L)}$ for $\tau_{\mathrm e}\le\tau\le\tau_+$.

  Since $\psi(v) \ge 1$ and
  \begin{equation}
    \psi(v)
    = 1 + \log(1+v)
    \ge 1 + \frac{1}{2}\log(1+\tau)
    \ge \frac{1}{2} \psi(\tau)
  \end{equation}
  holds for $v \ge \sqrt{1+\tau} - 1$, we obtain
  \begin{equation}
    \begin{split}
      \int_{\tau_{\mathrm e}}^\tau\psi(v)^{-(L-1)/(2L)}\,dv
      & \le
      \int_0^{\sqrt{1+\tau}-1}\psi(v)^{-(L-1)/(2L)}\,dv
      +\int_{\sqrt{1+\tau}-1}^\tau\psi(v)^{-(L-1)/(2L)}\,dv \\
      & \le
      \sqrt{1+\tau}
      +C(1+\tau)\psi(\tau)^{-(L-1)/(2L)}                    \\
      & \le C(1+\tau)\psi(\tau)^{-(L-1)/(2L)}.
    \end{split}
    \label{phased:eq:logarithmic-integral-estimate}
  \end{equation}
  Therefore,
  \begin{equation}
    \begin{split}
      \sigma(\tau)^{L-1}
      \int_{\tau_{\mathrm e}}^\tau
      \sigma(v)^{-(L-1)}\,dv
      & \le
      C\psi(\tau)^{(L-1)/(2L)}
      \int_{\tau_{\mathrm e}}^\tau
      \psi(v)^{-(L-1)/(2L)}\,dv \\
      & \le
      C(1+\tau).
    \end{split}
  \end{equation}

  The contribution from $[0,\tau_{\mathrm e}]$ is at most $C\sigma(\tau)^{L-1}\le C(1+\tau)$, so the same bound holds over the full interval $[0,\tau]$.

  For $p > 1$, using the same argument and the fact that $\sigma(\tau)\asymp 1$, we get
  \begin{equation}
    \sigma(\tau)^{L-1}
    \int_0^\tau\sigma(v)^{-(L-1)}\,dv
    \le C(1+\tau).
  \end{equation}
  Therefore,
  \begin{equation}
    0 \le
    \sigma(\tau)^{L-1}\lambda_i
    \int_0^\tau e^{-(\tau-v)\lambda_i}\sigma(v)^{-(L-1)}\,dv
    \le C\lambda_i(1+\tau).
  \end{equation}
  This proves \cref{phased:eq:low-frequency-teacher-response}.
\end{proof}

Define
\begin{equation}
  \overline{\mathcal L}_p
  :=
  \sup_{\tau_{\mathrm e}\le \tau\le \tau_+}
  (1+\tau)^p\mathcal L(\tau).
  \label{phased:eq:pgeone-bootstrap}
\end{equation}

\begin{lemma}
  \label{phased:lem:pgeone-residual-upper}
  Under the assumptions stated at the beginning of this subsection, for
  $\tau_{\mathrm e} \le \tau \le \tau_+$ and $\mu > 0$, we have
  \begin{equation}
    \begin{split}
      &\left\|\boldsymbol P_{\le\mu}\boldsymbol\Sigma^{1/2}\sigma(\tau)^{L-1}\left(e^{-(\tau-\tau_{\mathrm e})\boldsymbol\Sigma}\boldsymbol\xi_{\mathrm e}+\int_{\tau_{\mathrm e}}^\tau e^{-(\tau-v)\boldsymbol\Sigma}\boldsymbol f_K(v)\,dv\right)\right\| \\
      &\qquad\qquad\qquad\le C\mu^{1/2}\left(sS_d(\tau_+)^{(L-1)/(2L)}+\varepsilon \overline{\mathcal L}_p^{1/2}S_d(\tau_+)^{(L-1)/L}\right).
    \end{split}
    \label{phased:eq:pgeone-perturbation-upper}
  \end{equation}
\end{lemma}

\begin{proof}
  From \cref{phased:prop:common-radius-law}, for $\tau_{\mathrm e}\le \tau\le \tau_+$,
  \begin{equation}
    \sigma(\tau)^2\asymp S_d(\tau)^{1/L}\asymp
    \begin{cases}
      \psi(\tau)^{1/L}, & p = 1, \\
      1,                & p > 1.
    \end{cases}
  \end{equation}

  \step{Upper bound for $\int_{\tau_{\mathrm e}}^{\tau_+}\|\boldsymbol f_K(v)\|\,dv$.}

  If $L=1$, then $\varepsilon=0$ and hence $\boldsymbol f_K\equiv\boldsymbol0$, so \cref{phased:eq:pgeone-forcing-sum} holds immediately. For the remainder of this step, assume $L\ge2$.

  Define $x_j := 2^j \tau_{\mathrm e}$ for $j=0,1,\ldots$ and let $J := \max\{j\ge 0: x_j \le \tau_+\}$. Then, we can decompose $[\tau_{\mathrm e},\tau_+]$ as $\bigcup_{j=0}^{J-1} [x_j,2x_j] \cup [x_J,\tau_+] = \bigcup_{j=0}^J [x_j,\min\{2x_j,\tau_+\}]$.

  The Cauchy--Schwarz inequality and \cref{commonwork:eq:tau-loss-dynamics,phased:eq:common-kernel-lower} give
  \begin{equation}
    \begin{split}
      \int_{x_j}^{\min\{2x_j,\tau_+\}}\|\boldsymbol g(v)\|\,dv
      & \le
      x_j^{1/2}
      \left\{\int_{x_j}^{\min\{2x_j,\tau_+\}}
      \|\boldsymbol g(v)\|^2\,dv\right\}^{1/2} \\
      & \le
      c_K^{-1/2}x_j^{1/2}
      \{\mathcal L(x_j)-
      \mathcal L(\min\{2x_j,\tau_+\})\}^{1/2}  \\
      & \le
      C \overline{\mathcal L}_p^{1/2}(1+x_j)^{(1-p)/2}.
    \end{split}
  \end{equation}

  From \cref{phased:prop:common-radius-law} and the monotonicity of $S_d$, for $x_j\le v\le\min\{2x_j,\tau_+\}$ we have
  \begin{equation}
    \sigma(v)^{-(L+1)}
    \le C S_d(v)^{-(L+1)/(2L)}
    \le C S_d(x_j)^{-(L+1)/(2L)}.
  \end{equation}

  Therefore,
  \begin{equation}
    \begin{split}
      \int_{x_j}^{\min\{2x_j,\tau_+\}}\|\boldsymbol f_K(v)\|\,dv
      & \le
      C\varepsilon
      \sup_{x_j\le v\le\min\{2x_j,\tau_+\}}\sigma(v)^{-(L+1)}
      \int_{x_j}^{\min\{2x_j,\tau_+\}}\|\boldsymbol g(v)\|\,dv \\
      & \le
      C\varepsilon  \overline{\mathcal L}_p^{1/2}
      (1+x_j)^{(1-p)/2}S_d(x_j)^{-(L+1)/(2L)}.
    \end{split}
  \end{equation}

  For $p=1$, we have $(1+x_j)^{(1-p)/2}=1$. Using $S_d(x_j)\asymp \psi(x_j) \asymp 1+j$ and $J+1=1+\left\lfloor\log_2\frac{\tau_+}{\tau_{\mathrm e}}\right\rfloor\le C\psi(\tau_+)$, we get
  \begin{equation}
    \begin{split}
      \sum_{j=0}^J (1+j)^{-(L+1)/(2L)}
      & \le
      1+\int_0^J (1+x)^{-(L+1)/(2L)}\,dx \\
      & \le C (1+J)^{(L-1)/(2L)}        \\
      & \le C S_d(\tau_+)^{(L-1)/(2L)}.
    \end{split}
  \end{equation}

  For $p>1$, because $S_d(x_j)\asymp 1$, it follows that
  \begin{equation}
    \sum_{j=0}^J
    (1+x_j)^{(1-p)/2}S_d(x_j)^{-(L+1)/(2L)}
    \le C\sum_{j\ge0}2^{j(1-p)/2}
    \le C.
  \end{equation}

  Therefore,
  \begin{equation}
    \int_{\tau_{\mathrm e}}^{\tau_+}\|\boldsymbol f_K(v)\|\,dv
    \le
    C\varepsilon  \overline{\mathcal L}_p^{1/2}S_d(\tau_+)^{(L-1)/(2L)}.
    \label{phased:eq:pgeone-forcing-sum}
  \end{equation}

  \step{Upper bound for the two terms involving $\boldsymbol\xi_{\mathrm e}$ and $\boldsymbol f_K$.}

  From \cref{phased:eq:entrance-remainder-bound} and $\sigma(\tau)^{L-1}\le CS_d(\tau_+)^{(L-1)/(2L)}$, we get
  \begin{equation}
    \left\|
    \boldsymbol P_{\le\mu}\boldsymbol\Sigma^{1/2}
    \sigma(\tau)^{L-1}
    e^{-(\tau-\tau_{\mathrm e})\boldsymbol\Sigma}
    \boldsymbol\xi_{\mathrm e}
    \right\|
    \le
    \mu^{1/2}\sigma(\tau)^{L-1}
    \|\boldsymbol\xi_{\mathrm e}\|
    \le
    C\mu^{1/2}sS_d(\tau_+)^{(L-1)/(2L)}.
  \end{equation}
  From \cref{phased:eq:pgeone-forcing-sum}, we get
  \begin{equation}
    \begin{split}
      \left\|
      \boldsymbol P_{\le\mu}\boldsymbol\Sigma^{1/2}
      \sigma(\tau)^{L-1}
      \int_{\tau_{\mathrm e}}^\tau
      e^{-(\tau-v)\boldsymbol\Sigma}\boldsymbol f_K(v)\,dv
      \right\|
      & \le
      \mu^{1/2}\sigma(\tau)^{L-1}
      \int_{\tau_{\mathrm e}}^\tau\|\boldsymbol f_K(v)\|\,dv \\
      & \le
      C\mu^{1/2}\varepsilon \overline{\mathcal L}_p^{1/2}S_d(\tau_+)^{(L-1)/L}.
    \end{split}
  \end{equation}
\end{proof}

\begin{proposition}
  \label{phased:prop:pgeone-loss-law}
  Suppose $p \ge 1$ and consider Gaussian initialization satisfying \cref{phased:ass:small-initialization}. Then there exist constants $c,C>0$ such that, for all sufficiently large $d$, with probability at least $1-Ce^{-cd}$, uniformly for $\tau_{\mathrm e} \le \tau \le \tau_+$,
  \begin{equation}
    \mathcal{L}(\tau) \asymp (1 + \tau)^{-p}.
    \label{phased:eq:pgeone-loss-law}
  \end{equation}
  The implicit constants are uniform over $0<\delta\le\delta_0$.
\end{proposition}

\begin{proof}
  Fix $\alpha > p$. Let $\tau_0 \ge \max\{\tau_{\mathrm e}, 2\}$ be sufficiently large that $\alpha/(2c_K(1+\tau_0)) \le \mu_\star$. Define $\mu(v):=\alpha/(2c_K(1+v))$ for $v\ge \tau_0$.

  \step{Upper bound for $\mathcal{L}(\tau)$.}\label{phased:step:pgeone-loss-upper}

  For $\tau_0\le v\le\tau_+$, using
  \cref{phased:eq:low-frequency-teacher-response} and the upper bound in
  \cref{phased:lem:powerlaw-estimates}, we obtain
  \begin{equation}
    \left\|
    \boldsymbol P_{\le\mu(v)}\boldsymbol\Sigma^{1/2}
    \left[\boldsymbol I-
      \sigma(v)^{L-1}\boldsymbol\Sigma
      \int_0^v e^{-(v-w)\boldsymbol\Sigma}
      \sigma(w)^{-(L-1)}\,dw\right]
    \boldsymbol\beta
    \right\|
    \le C\mu(v)^{p/2}.
  \end{equation}
  Therefore, combining this with \cref{phased:eq:pgeone-perturbation-upper}, we have
  \begin{equation}
    \begin{split}
      \mathcal L_{\le\mu(v)}(v)
      & =
      \frac12
      \|\boldsymbol P_{\le\mu(v)}
      \boldsymbol\Sigma^{1/2}
      (\boldsymbol\beta-\boldsymbol\theta(v))\|^2 \\
      & \le C\left(
      \mu(v)^p
      +\mu(v)s^2S_d(\tau_+)^{(L-1)/L}
      +\mu(v)\varepsilon ^2\overline{\mathcal L}_pS_d(\tau_+)^{2(L-1)/L}
      \right)                                     \\
      & \le C\left[
        (1+v)^{-p}
        +(1+v)^{-1}
        \left(
        s^2S_d(\tau_+)^{(L-1)/L}
        +\varepsilon ^2\overline{\mathcal L}_pS_d(\tau_+)^{2(L-1)/L}
        \right)
        \right].
    \end{split}
  \end{equation}
  Applying \cref{phased:eq:polynomial-loss-propagation} of \cref{phased:lem:loss-propagation} with initial time $\tau_0$, we obtain, for $\tau_0\le \tau\le \tau_+$,
  \begin{equation}
    \begin{split}
      \mathcal L(\tau)
      & \le
      \left(\frac{1+\tau_0}{1+\tau}\right)^\alpha\mathcal L(\tau_0) \\
      & \quad+
      C(1+\tau)^{-\alpha}
      \int_{\tau_0}^\tau
      (1+v)^{\alpha-p-1}\,dv \\
      & \quad+
      C(1+\tau)^{-\alpha}
      \int_{\tau_0}^\tau
      (1+v)^{\alpha-2}
      \left[
        s^2S_d(\tau_+)^{(L-1)/L}
        +\varepsilon ^2\overline{\mathcal L}_pS_d(\tau_+)^{2(L-1)/L}
        \right]\,dv \\
      & \le
      C(1+\tau)^{-\alpha}
      +\frac{C}{\alpha-p}(1+\tau)^{-p} \\
      & \quad
      +\frac{C}{\alpha-1}(1+\tau)^{-1}
      \left(
      s^2S_d(\tau_+)^{(L-1)/L}
      +\varepsilon ^2\overline{\mathcal L}_pS_d(\tau_+)^{2(L-1)/L}
      \right)                                                       \\
      & \le
      C\left[
        (1+\tau)^{-p}
        +(1+\tau)^{-1}
        \left(
        s^2S_d(\tau_+)^{(L-1)/L}
        +\varepsilon ^2\overline{\mathcal L}_pS_d(\tau_+)^{2(L-1)/L}
        \right)
        \right].
    \end{split}
    \label{phased:eq:pgeone-loss-upper0}
  \end{equation}
  For $\tau \in [\tau_{\mathrm e}, \tau_0]$, we have $(1+\tau)^p \mathcal L(\tau) \le C$. For sufficiently large $d$, we have $\tau_0\le \tau_+$. Thus, multiplying both sides of \cref{phased:eq:pgeone-loss-upper0} by $(1+\tau)^p$, we obtain
  \begin{equation}
    \begin{split}
      \overline{\mathcal L}_p
      \le C_0
      +C_1\left(
      s^2S_d(\tau_+)^{(L-1)/L}
      +\varepsilon ^2\overline{\mathcal L}_pS_d(\tau_+)^{2(L-1)/L}
      \right)(1+\tau_+)^{p-1}
    \end{split}
    \label{phased:eq:pgeone-bootstrap-closure}
  \end{equation}
  for some constants $C_0, C_1 > 0$.

  For $p=1$, we have $S_d(\tau_+) \asymp \log d$. Combining this with \cref{phased:eq:regime-smallness} and $\varepsilon\lesssim s^2$, we obtain
  \begin{equation}
    \begin{split}
      &s^2S_d(\tau_+)^{(L-1)/L}
      +\varepsilon ^2S_d(\tau_+)^{2(L-1)/L} \\
      &\quad\le C\left[
        s^2(\log d)^{(L-1)/L}
        +\{s^2(\log d)^{(L-1)/L}\}^2
        \right] \\
      &\quad\le C(\delta+\delta^2).
    \end{split}
    \label{phased:eq:pgeone-rate-p1}
  \end{equation}
  For $p>1$, we have $S_d(\tau_+) \asymp 1$.
  Since \cref{phased:eq:regime-smallness} gives $s^2d^{a(p-1)}\le\delta$ and $s^2\le\delta$,
  \begin{equation}
    \left(
    s^2S_d(\tau_+)^{(L-1)/L}
    +\varepsilon ^2S_d(\tau_+)^{2(L-1)/L}
    \right)(1+\tau_+)^{p-1}
    \le C(1+s^2)s^2d^{a(p-1)}
    \le C(\delta+\delta^2).
    \label{phased:eq:pgeone-rate-pg1}
  \end{equation}

  We further require $\delta_0$ to be sufficiently small that, for every $0<\delta\le\delta_0$ and all sufficiently large $d$,
  \begin{equation}
    C_1s^2S_d(\tau_+)^{(L-1)/L}(1+\tau_+)^{p-1}\le C_0,
    \qquad
    C_1\varepsilon ^2S_d(\tau_+)^{2(L-1)/L}(1+\tau_+)^{p-1}\le\frac12.
  \end{equation}
  Substituting these bounds into \cref{phased:eq:pgeone-bootstrap-closure}, we deduce
  \begin{equation}
    \overline{\mathcal L}_p \le 2C_0+\frac12\overline{\mathcal L}_p.
  \end{equation}
  Hence, $\overline{\mathcal L}_p \le 4C_0$. Therefore, by the definition of $\overline{\mathcal L}_p$ in \cref{phased:eq:pgeone-bootstrap}, for $\tau_{\mathrm e}\le \tau\le \tau_+$,
  \begin{equation}
    \mathcal L(\tau)\lesssim (1+\tau)^{-p}.
    \label{phased:eq:pgeone-loss-upper}
  \end{equation}

  \step{Lower bound for $\mathcal L(\tau)$.}

  Fix a sufficiently small constant $\mu_0>0$ such that $\mu_0\le\mu_\star$. If necessary, we decrease $c_0$ further to ensure that $4c_\mu c_0<\mu_0$. By \cref{phased:step:pgeone-loss-upper}, we have $\overline{\mathcal L}_p\le C$. For any $\tau_{\mathrm e}\le \tau\le \tau_+$, choose $\mu(\tau)\in[\mu_0/(2(1+\tau)),\,\mu_0/(1+\tau)]$. For sufficiently large $d$, we have $1+\tau_+\le 2\tau_+$. Therefore,
  \begin{equation}
    c_\mu d^{-a}
    \le\frac{\mu_0}{4c_0}d^{-a}
    \le\frac{\mu_0}{2(1+\tau_+)}
    \le\frac{\mu_0}{2(1+\tau)}
    \le\mu(\tau)
    \le\frac{\mu_0}{1+\tau}
    \le\mu_0
    \le\mu_\star.
  \end{equation}

  For every coordinate $i$ such that $\lambda_i(1+\tau)\le \mu_0$, \cref{phased:lem:low-frequency-teacher-response} shows that the coefficient multiplying $\beta_i$ in the first term on the right-hand side of \cref{phased:eq:exact-residual-decomposition} is at least $1-C\mu_0$.
  Since $\mu_0$ is sufficiently small, we may assume that $1-C\mu_0\ge 1/2$, so this coefficient is at least $1/2$.
  Combining this with the lower bound in \cref{phased:lem:powerlaw-estimates}, we obtain
  \begin{equation}
    \begin{split}
      \left\|
      \boldsymbol P_{\le\mu(\tau)}\boldsymbol\Sigma^{1/2}
      \left[
        \boldsymbol I-
        \sigma(\tau)^{L-1}\boldsymbol\Sigma
        \int_0^\tau e^{-(\tau-v)\boldsymbol\Sigma}
        \sigma(v)^{-(L-1)}\,dv\right]
      \boldsymbol\beta
      \right\|
      & \ge
      \frac12
      \|\boldsymbol P_{\le\mu(\tau)}
      \boldsymbol\Sigma^{1/2}\boldsymbol\beta\| \\
      & \ge c\mu(\tau)^{p/2}                   \\
      & \gtrsim (1+\tau)^{-p/2}.
    \end{split}
  \end{equation}
  On the other hand, by \cref{phased:lem:pgeone-residual-upper},
  \begin{equation}
    \begin{split}
      & \left\|
      \boldsymbol P_{\le\mu(\tau)}\boldsymbol\Sigma^{1/2}
      \sigma(\tau)^{L-1}
      \left(
      e^{-(\tau-\tau_{\mathrm e})\boldsymbol\Sigma}
      \boldsymbol\xi_{\mathrm e}
      +\int_{\tau_{\mathrm e}}^\tau
      e^{-(\tau-v)\boldsymbol\Sigma}\boldsymbol f_K(v)\,dv
      \right)
      \right\|    \\
      & \quad\le
      C\mu(\tau)^{1/2}
      \left(
      sS_d(\tau_+)^{(L-1)/(2L)}
      +\varepsilon S_d(\tau_+)^{(L-1)/L}
      \right)     \\
      & \quad\le
      C\mu_0^{-(p-1)/2}\mu(\tau)^{p/2}(1+\tau_+)^{(p-1)/2}
      \left(
      sS_d(\tau_+)^{(L-1)/(2L)}
      +\varepsilon S_d(\tau_+)^{(L-1)/L}
      \right)
    \end{split}
  \end{equation}
  and, as in the arguments for \cref{phased:eq:pgeone-rate-p1,phased:eq:pgeone-rate-pg1},
  \begin{equation}
    \begin{split}
      &(1+\tau_+)^{(p-1)/2}
      \left\{
      sS_d(\tau_+)^{(L-1)/(2L)}
      +\varepsilon S_d(\tau_+)^{(L-1)/L}
      \right\} \\
      &\quad\le
      \sqrt2\left[
        \left\{
        s^2S_d(\tau_+)^{(L-1)/L}
        +\varepsilon ^2S_d(\tau_+)^{2(L-1)/L}
        \right\}(1+\tau_+)^{p-1}
        \right]^{1/2} \\
      &\quad\le C(\delta+\delta^2)^{1/2}.
    \end{split}
  \end{equation}
  Finally, since the spectral cutoff $\mu_0$ has already been fixed, we require $\delta_0$ to be sufficiently small that
  \begin{equation}
    C\mu_0^{-(p-1)/2}(\delta_0+\delta_0^2)^{1/2}\le\frac c2,
  \end{equation}
  where $c$ is the constant in the lower bound for the first term above.
  Therefore, for every $0<\delta\le\delta_0$ and all sufficiently large $d$, \cref{phased:eq:exact-residual-decomposition} gives
  \begin{equation}
    \|
    \boldsymbol P_{\le\mu(\tau)}
    \boldsymbol\Sigma^{1/2}
    (\boldsymbol\beta-\boldsymbol\theta(\tau))
    \|
    \ge c\mu(\tau)^{p/2}.
    \label{phased:eq:pgeone-residual-lower}
  \end{equation}
  Therefore, for $\tau_{\mathrm e}\le \tau\le \tau_+$,
  \begin{equation}
    \begin{split}
      \mathcal L(\tau)
      & \ge \mathcal L_{\le\mu(\tau)}(\tau)         \\
      & = \frac12
      \|\boldsymbol P_{\le\mu(\tau)}
      \boldsymbol\Sigma^{1/2}
      (\boldsymbol\beta-\boldsymbol\theta(\tau))\|^2 \\
      & \ge c\mu(\tau)^p                            \\
      & \gtrsim (1+\tau)^{-p}.
    \end{split}
    \label{phased:eq:pgeone-loss-lower}
  \end{equation}

  \step{Combining the upper and lower bounds.}

  Combining \cref{phased:eq:pgeone-loss-upper,phased:eq:pgeone-loss-lower} yields the desired conclusion.

\end{proof}

\subsection{Common Spectral Law under Small Initialization}

Combining the cases $p<1$ and $p\ge1$ gives the common spectral loss law under small initialization. The following proposition gives a precise version of \cref{prop:common-spectral-law-loss}.
\begin{proposition}[Common spectral law under small initialization]
  \label{phased:prop:common-spectral-loss-law}
  Consider Gaussian initialization satisfying \cref{phased:ass:small-initialization}. Then there exist constants $c,C>0$, independent of $d$, $s$, and $\tau$, such that, for all sufficiently large $d$, with probability at least $1-Ce^{-cd}$, uniformly for $\tau_{\mathrm e}\le \tau\le\tau_+$,
  \begin{equation}
    \mathcal L(\tau)\asymp(1+\tau)^{-p}.
  \end{equation}
\end{proposition}

\begin{proof}
  If $p<1$, the proof of \cref{cor:small-init-predictor-norm-scaling} verifies that, under \cref{phased:ass:small-initialization}, the conditions of \cref{phased:prop:common-radius-law} hold with $\tau_\circ=\tau_{\mathrm e}$. \Cref{phased:prop:hard-loss-law} therefore gives $\mathcal L(\tau)\asymp(1+\tau)^{-p}$ for $\tau_{\mathrm e}\le\tau\le\tau_+$. If $p\ge1$, the same conclusion follows directly from \cref{phased:prop:pgeone-loss-law}. Combining the two cases proves the claim.
\end{proof}

\section{Training-Time Loss Scaling}
\label{phased:sec:main-scaling-proofs}

Using the predictor-norm estimates in \cref{phased:sec:radius-law}, we relate reparameterized time $\tau$ to training time $t$. We then express the spectral loss laws in \cref{phased:prop:hard-loss-law,phased:prop:common-spectral-loss-law} in terms of training time.

The following theorem gives a precise version of
\cref{thm:training-time-loss-scaling}.
\begin{theorem}[Training-time loss scaling under small initialization]
  \label{phased:thm:training-time-loss-scaling}
  Consider Gaussian initialization satisfying \cref{phased:ass:small-initialization}. Let $t_{\mathrm e}$ be defined by $\tau(t_{\mathrm e})=\tau_{\mathrm e}$, set $\bar t=t-t_{\mathrm e}$, and let $t_+$ be defined by $\tau(t_+)=\tau_+$, with $\bar t_+:=t_+-t_{\mathrm e}$. Then there exist constants $c,C,\bar t_0>0$, independent of $d$, $s$, and $\bar t$, such that, for all sufficiently large $d$, with probability at least $1-Ce^{-cd}$, uniformly for $\bar t_0\le\bar t\le\bar t_+$, we have
  \begin{equation}
    \mathcal L(t_{\mathrm e}+\bar t)\asymp
    \begin{cases}
      \bar t^{-pL/[1+p(L-1)]},            & p < 1, \\
      \bar t^{-1}(\log\bar t)^{-(L-1)/L}, & p = 1, \\
      \bar t^{-p},                        & p > 1.
    \end{cases}
  \end{equation}
\end{theorem}

\begin{proof}
  By \cref{phased:eq:time-reparametrization}, $d\tau/dt=\sigma^{2(L-1)}$. Since $\bar t=t-t_{\mathrm e}$ and $\|\boldsymbol\theta\|=\sigma^L$, we obtain
  \begin{equation}
    \frac{d\bar t}{d\tau} = \sigma^{-2(L-1)} = \|\boldsymbol\theta(\tau)\|^{-2(L-1)/L}.
  \end{equation}
  Equivalently,
  \begin{equation}
    \bar t = \int_{\tau_{\mathrm e}}^{\tau} \|\boldsymbol\theta(v)\|^{-2(L-1)/L}\,dv.
    \label{phased:eq:reparam-bar-t}
  \end{equation}
  Choose a sufficiently large fixed constant $\tau_0>\tau_{\mathrm e}$ such that the bounds in \cref{cor:small-init-predictor-norm-scaling} hold for all $\tau_0\le\tau\le\tau_+$. By \cref{phased:cor:post-entry-radius-kernel}, we get $\sigma(\tau)^2\ge c>0$ for $\tau\ge\tau_{\mathrm e}$. Therefore,
  \begin{equation}
    \int_{\tau_{\mathrm e}}^{\tau_0} \|\boldsymbol\theta(v)\|^{-2(L-1)/L}\,dv = \int_{\tau_{\mathrm e}}^{\tau_0} \sigma(v)^{-2(L-1)}\,dv=O(1).
    \label{phased:eq:short-initial-phase}
  \end{equation}
  Consequently,
  \begin{equation}
    \bar t = O(1)+ \int_{\tau_0}^{\tau} \|\boldsymbol\theta(v)\|^{-2(L-1)/L}\,dv.
    \label{phased:eq:bar-t-long}
  \end{equation}
  Since the $O(1)$ term is uniform in $d$ and $s$, we may choose the fixed constant $\bar t_0>0$ sufficiently large so that, for $\bar t\ge \bar t_0$, this term can be absorbed into the comparisons below.

  Combining \cref{cor:small-init-predictor-norm-scaling} with
  \cref{phased:eq:bar-t-long}, we obtain, for
  $\bar t_0\le\bar t\le\bar t_+$,
  \begin{equation}
    \bar t \asymp
    \begin{cases}
      \displaystyle \int_{\tau_0}^{\tau} v^{-(1-p)(L-1)/L}\,dv,   & p < 1, \\
      \displaystyle \int_{\tau_0}^{\tau} (\log v)^{-(L-1)/L}\,dv, & p = 1, \\
      \displaystyle \int_{\tau_0}^{\tau}dv,                       & p > 1.
    \end{cases}
  \end{equation}
  Evaluating these integrals gives
  \begin{equation}
    \bar t \asymp
    \begin{cases}
      \tau^{[1+p(L-1)]/L},       & p < 1, \\
      \tau(\log\tau)^{-(L-1)/L}, & p = 1, \\
      \tau,                      & p > 1.
    \end{cases}
    \label{phased:eq:bar-t-tau-scaling}
  \end{equation}
  For $p=1$, the estimate $\int_{\tau_0}^{\tau}(\log v)^{-(L-1)/L}\,dv \asymp \tau(\log\tau)^{-(L-1)/L}$ follows from the same splitting argument used in \cref{phased:eq:logarithmic-integral-estimate}.
  The relations in \cref{phased:eq:bar-t-tau-scaling} can be inverted as
  \begin{equation}
    \tau \asymp
    \begin{cases}
      \bar t^{L/[1+p(L-1)]},        & p < 1, \\
      \bar t(\log\bar t)^{(L-1)/L}, & p = 1, \\
      \bar t,                       & p > 1.
    \end{cases}
    \label{phased:eq:tau-asymp-bar-t}
  \end{equation}
  In the case $p=1$, we used the fact that $\bar t\asymp\tau(\log\tau)^{-(L-1)/L}$ implies $\log\bar t\asymp\log\tau$.
  Finally, \cref{phased:prop:common-spectral-loss-law} gives $\mathcal L(\tau)\asymp\tau^{-p}$ on this interval. Substituting \cref{phased:eq:tau-asymp-bar-t} yields
  \begin{equation}
    \mathcal L(t_{\mathrm e}+\bar t) \asymp
    \begin{cases}
      \bar t^{-pL/[1+p(L-1)]},            & p < 1, \\
      \bar t^{-1}(\log\bar t)^{-(L-1)/L}, & p = 1, \\
      \bar t^{-p},                        & p > 1.
    \end{cases}
  \end{equation}
\end{proof}

The following corollary gives a precise version of \cref{cor:standard-init-training-time-loss-scaling}.
\begin{corollary}[Training-time loss scaling under standard initialization for difficult tasks]
  \label{phased:cor:standard-init-training-time-loss-scaling}
  Suppose $p<1$ and consider Gaussian initialization with $s=1$. Let $t_+$ be defined by $\tau(t_+)=\tau_+$. Then there exist constants $c,C,t_0>0$, independent of $d$ and $t$, such that, for all sufficiently large $d$, with probability at least $1-Ce^{-cd}$, uniformly for $t_0\le t\le t_+$, we have $\mathcal L(t) \asymp t^{-pL/[1+p(L-1)]}$.
\end{corollary}

\begin{proof}
  Let $\tau_0>0$ be the fixed constant in \cref{cor:standard-init-scaling}. Then $\|\boldsymbol\theta(\tau)\|^2\asymp\tau^{1-p}$ uniformly for $\tau_0\le\tau\le\tau_+$, while \cref{phased:prop:hard-loss-law} gives $\mathcal L(\tau)\asymp\tau^{-p}$. By \cref{phased:eq:time-reparametrization}, $d\tau/dt=\sigma^{2(L-1)}$, and $\|\boldsymbol\theta\|=\sigma^L$.

  By \cref{phased:lem:fixed-scale-radius-kernel}, the training time elapsed over $0\le\tau\le\tau_0$ is bounded uniformly in $d$. Choosing $t_0$ sufficiently large, we therefore obtain, uniformly for $t_0\le t\le t_+$,
  \begin{equation}
    t = O(1)+ \int_{\tau_0}^{\tau} \|\boldsymbol\theta(v)\|^{-2(L-1)/L}\,dv \asymp \int_{\tau_0}^{\tau} v^{-(1-p)(L-1)/L}\,dv \asymp \tau^{[1+p(L-1)]/L}.
  \end{equation}
  Therefore,
  \begin{equation}
    \tau\asymp t^{L/[1+p(L-1)]}
  \end{equation}
  and substituting this into $\mathcal L(\tau)\asymp\tau^{-p}$ gives
  \begin{equation}
    \mathcal L(t) \asymp t^{-pL/[1+p(L-1)]}.
  \end{equation}
\end{proof}

\section{Experimental Details}
\label{sec:numerical-experiments-details}

\subsection{Synthetic Power-Law Experiments}

\paragraph{\Cref{fig:vanishing-loss-shifted-time,fig:vanishing-predictor-norm-tau,fig:reparameterized-time-loss}.}
These three figures use the same trajectories. We set $a=2$ and $c_\lambda=c_\beta=1$, take $d=2048$, and use $p\in\{0.6,1.0,1.4\}$, or equivalently $b\in\{0.1,0.5,0.9\}$. We consider $L\in\{1,2,3,4\}$, with every hidden layer having width $d$, and initialize each layer with scale $s=d^{-1/2}$. We minimize the exact population loss by gradient descent with a learning rate of $0.01$ and accumulate $\tau$ by the trapezoidal rule. Each curve is from a single random initialization. We set $\tau_{\mathrm e}=0.1$ and use the corresponding interpolated value of $t_{\mathrm e}$ in the shifted time $\bar t=t-t_{\mathrm e}$. The corresponding loss curves without a time shift are shown in \cref{fig:vanishing-loss-unshifted-time}.

\begin{figure}[h]
  \centering
  \begin{subfigure}[t]{0.32\textwidth}
    \centering
    \includegraphics[width=\linewidth]
    {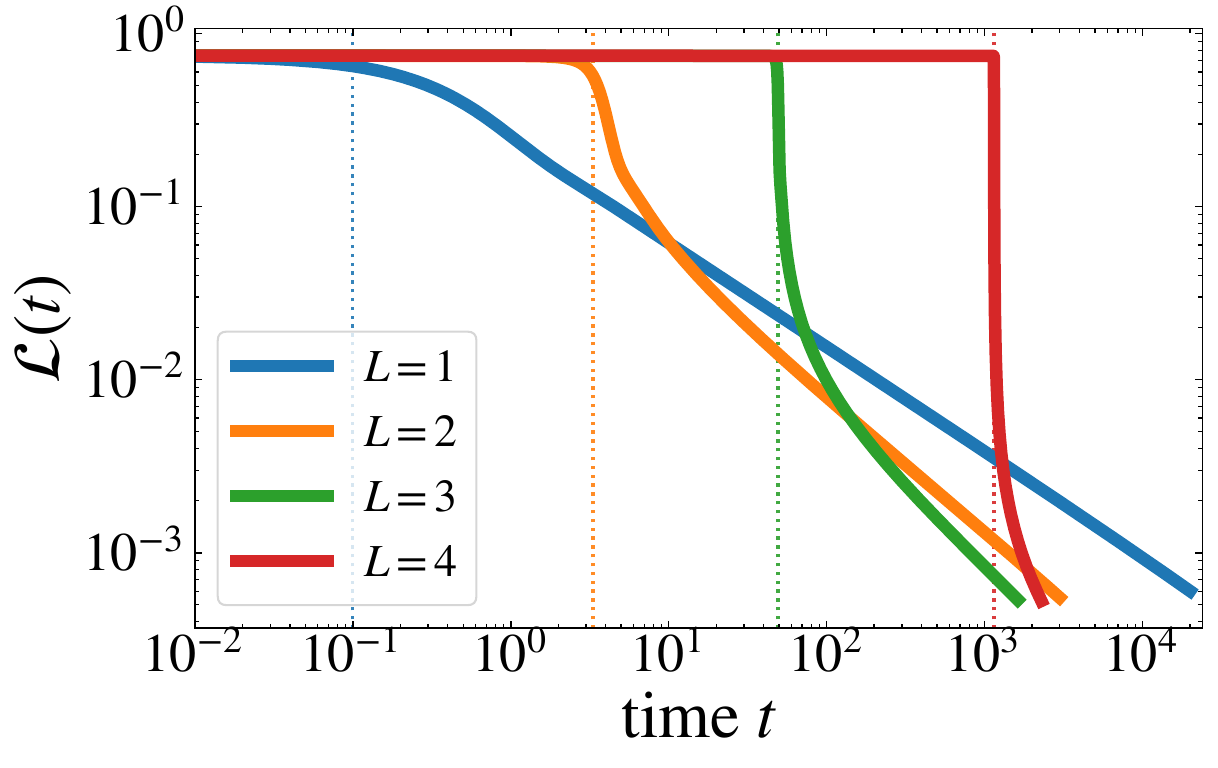}
    \caption{$p=0.6$}
    \label{fig:vanishing-loss-unshifted-p06}
  \end{subfigure}
  \hfill
  \begin{subfigure}[t]{0.32\textwidth}
    \centering
    \includegraphics[width=\linewidth]
    {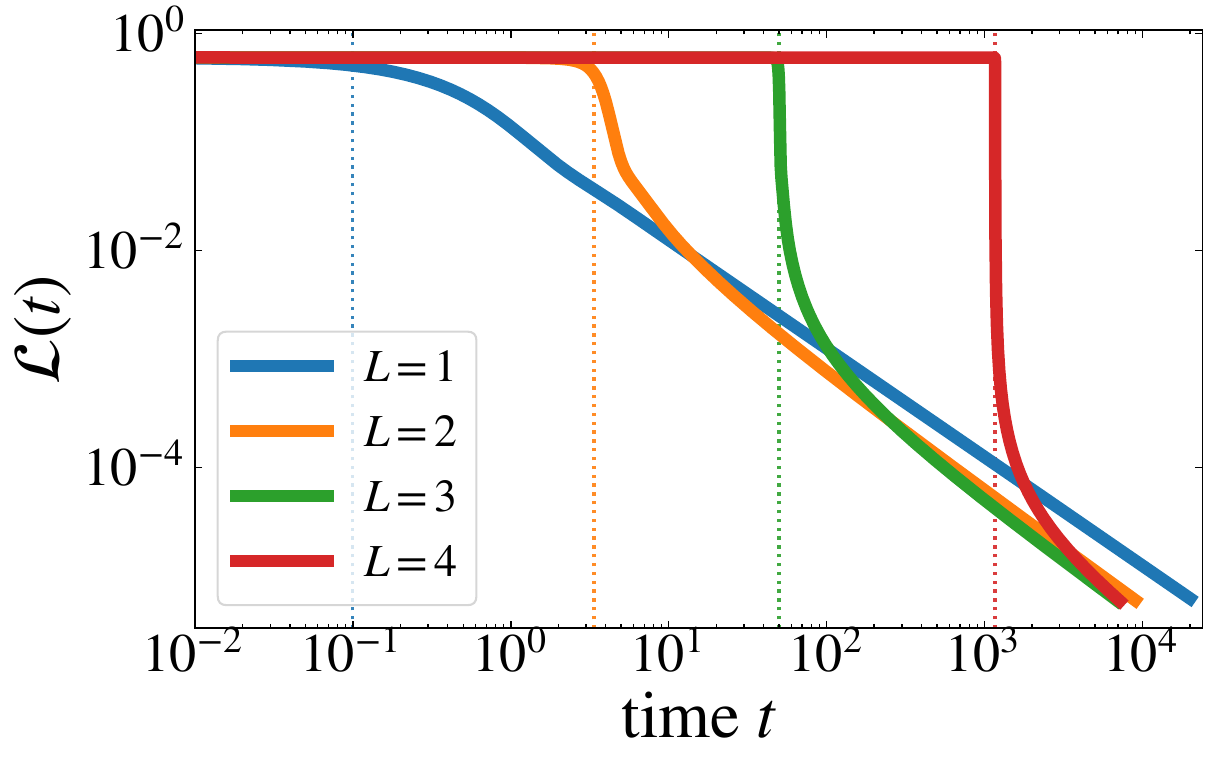}
    \caption{$p=1.0$}
    \label{fig:vanishing-loss-unshifted-p10}
  \end{subfigure}
  \hfill
  \begin{subfigure}[t]{0.32\textwidth}
    \centering
    \includegraphics[width=\linewidth]
    {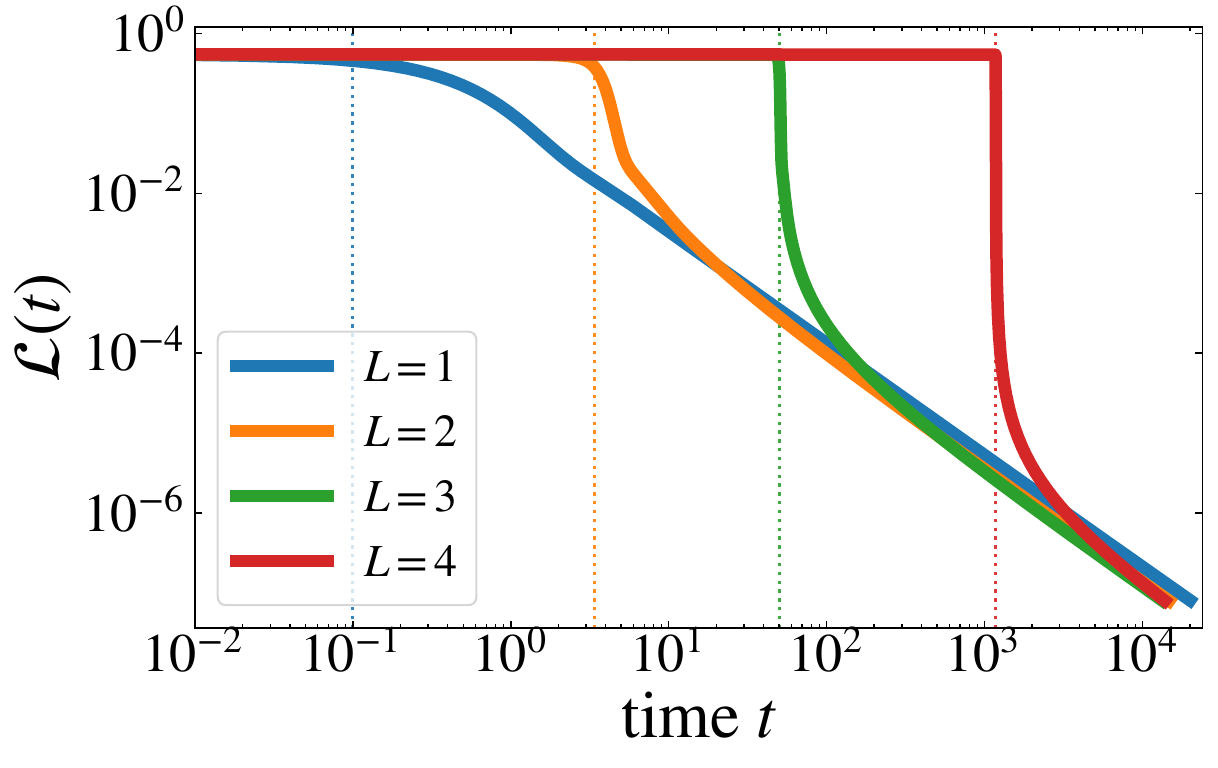}
    \caption{$p=1.4$}
    \label{fig:vanishing-loss-unshifted-p14}
  \end{subfigure}

  \caption{
    The same trajectories as in \cref{fig:vanishing-loss-shifted-time} are shown in training time $t$, without a time shift. Vertical dotted lines mark $t_{\mathrm e}$, where $\tau(t_{\mathrm e})=0.1$, for $L=1,2,3,4$.
  }
  \label{fig:vanishing-loss-unshifted-time}
\end{figure}

\paragraph{\Cref{fig:loss-decay-phase-diagram}.}
We use the same task, dimension, widths, initialization scale, and optimization settings as above, now with $p\in\{0.4,0.6,0.8,1.0,1.2,1.4\}$ and five random seeds for each $L\in\{1,2,3,4\}$. For $p\ne1$, over $50\le\tau\le5000$, we fit $\log\mathcal L=\log A-q\log(\bar t-B)$ by least squares over $A$, $B$, and $q$, with $B<\min\bar t$. At $p=1$, we interpolate each trajectory at logarithmically spaced values of $\tau$ over the same window and fit $\log\mathcal L(\bar t)+\frac{L-1}{L}\log\log\bar t=\log A-q\log(\bar t-B)$ by least squares over $A$, $B$, and $q$, again with $B<\min\bar t$. The markers and error bars show the mean and sample standard deviation of $q$ over the five seeds. Since even the largest standard deviation across seeds is approximately $0.004$, the error bars are smaller than the markers.

\paragraph{\Cref{fig:standard-initialization-loss}.}
We set $a=2.25$, $b=-0.175$, $c_\lambda=c_\beta=1$, so that $p=0.4$, and use $L\in\{1,2,3,4\}$ with hidden width $d$ and standard initialization $s=1$. We use $d\in\{128,256,512,1024,2048,4096,8192\}$; panel (a) displays $d=8192$, whereas panel (b) uses the full dimension sweep. For each pair $(L,d)$, we use five random seeds for initialization. All runs use gradient descent with a learning rate of $0.002$. To select a fitting window before the finite-dimensional spectral cutoff, we use the decoupled-mode proxy
\begin{equation}
  \mathcal L_{\mathrm{proxy}}(\tau)
  :=
  \frac12\sum_{i\ge1}\lambda_i\beta_i^2e^{-2\lambda_i\tau}.
\end{equation}
This is the loss obtained by approximating each residual mode as $\beta_i-\theta_i(\tau)=\beta_i e^{-\lambda_i\tau}$. Since $\lambda_i\beta_i^2=c_\lambda c_\beta^2i^{-1-ap}$, the fraction of this infinite-dimensional proxy carried by the modes $i>d$, which are absent from the dimension-$d$ experiment, is approximated by
\begin{equation}
  r_{\mathrm{tail}}(\tau;p,d)
  :=
  \frac{\sum_{i>d}\lambda_i\beta_i^2e^{-2\lambda_i\tau}}
  {\sum_{i\ge1}\lambda_i\beta_i^2e^{-2\lambda_i\tau}}
  \approx
  \frac{\int_d^\infty x^{-1-ap}e^{-2c_\lambda\tau x^{-a}}\,dx}
  {\int_0^\infty x^{-1-ap}e^{-2c_\lambda\tau x^{-a}}\,dx}
  =P\left(p,\frac{2c_\lambda\tau}{d^a}\right).
\end{equation}
Here, for $\alpha>0$ and $x\ge0$, the regularized lower incomplete gamma function is
\begin{equation}
  P(\alpha,x)
  =
  \frac{\int_0^x u^{\alpha-1}e^{-u}\,du}
  {\int_0^\infty u^{\alpha-1}e^{-u}\,du},
\end{equation}
and $P^{-1}(\alpha,\cdot)$ denotes its inverse with respect to the second argument. Setting $r_{\mathrm{tail}}=0.1$ and using $a=2.25$, $p=0.4$, and $c_\lambda=1$ gives $\tau_{\max}=d^{2.25}P^{-1}(0.4,0.1)/2$, and we set $\tau_{\min}=\max\{1,\tau_{\max}/1000\}$. Each run stops at the first iterate with $\tau\ge\tau_{\max}$.
For each seed, we linearly interpolate $t$ and $\mathcal L$ as functions of $\tau$ at logarithmically spaced points in this window and fit $\log\mathcal L=\log A-q\log t$ by ordinary least squares, without a time shift. In panel (a), the theoretical lines have fixed exponents $pL/[1+p(L-1)]$ and heights fitted in log space using the fitting points from all five seeds. Panel (b) reports the mean and sample standard deviation of the fitted exponents over the five seeds.

\subsection{Extensions beyond the Theoretical Setting}\label{sec:extensions-real-nonlinear}

\paragraph{\Cref{fig:cifar5m-real-label-dynamics}.}
We use all $N=5{,}002{,}240$ examples in the five parts of CIFAR-5M. Each $32\times32\times3$ RGB image is scaled by $1/255$, flattened, and centered by subtracting the empirical mean, giving $d=3072$. We likewise center the airplane-versus-rest target and compute the full empirical covariance $\boldsymbol\Sigma$ and the input--target cross moment $\boldsymbol e$. We rescale the inputs so that $\|\boldsymbol\Sigma\|_{\mathrm{op}}=1$ and rescale the target direction so that $\boldsymbol e^\top\boldsymbol\Sigma^\dagger\boldsymbol e=1$. We report the excess empirical squared loss relative to the best linear predictor, equivalently the squared loss against the corresponding realizable linear teacher. In the covariance eigenbasis, this teacher has coefficients $\beta_i=e_i/\lambda_i$. Motivated by \citet{wei2022more}, who study eigenvalue decay and target alignment in neural representations, we estimate the effective task exponent $p_{\mathrm{eff}}$ by fitting
\begin{equation}
  \sum_{i>k}\lambda_i\beta_i^2
  \propto \lambda_k^{p_{\mathrm{eff}}}
\end{equation}
using ordinary least squares in log--log coordinates. The fit uses the unique integer ranks obtained by rounding logarithmically spaced values from $10$ to $\operatorname{round}(0.9d)=2765$. The upper limit omits approximately the last $10\%$ of ranks from the fit to reduce finite-dimensional endpoint effects. This gives $p_{\mathrm{eff}}\approx0.478$, as shown in \cref{fig:cifar5m-task-tail}. We train networks with $L\in\{1,2,3,4\}$ and hidden width $3072$ from random initialization at scale $s=1$, using full-batch gradient descent on these empirical sufficient statistics. All runs use a learning rate of $0.002$ and continue to $\tau=10^6$. Each curve is from a single random initialization.

\begin{figure}[t]
  \centering
  \includegraphics[width=0.50\textwidth]{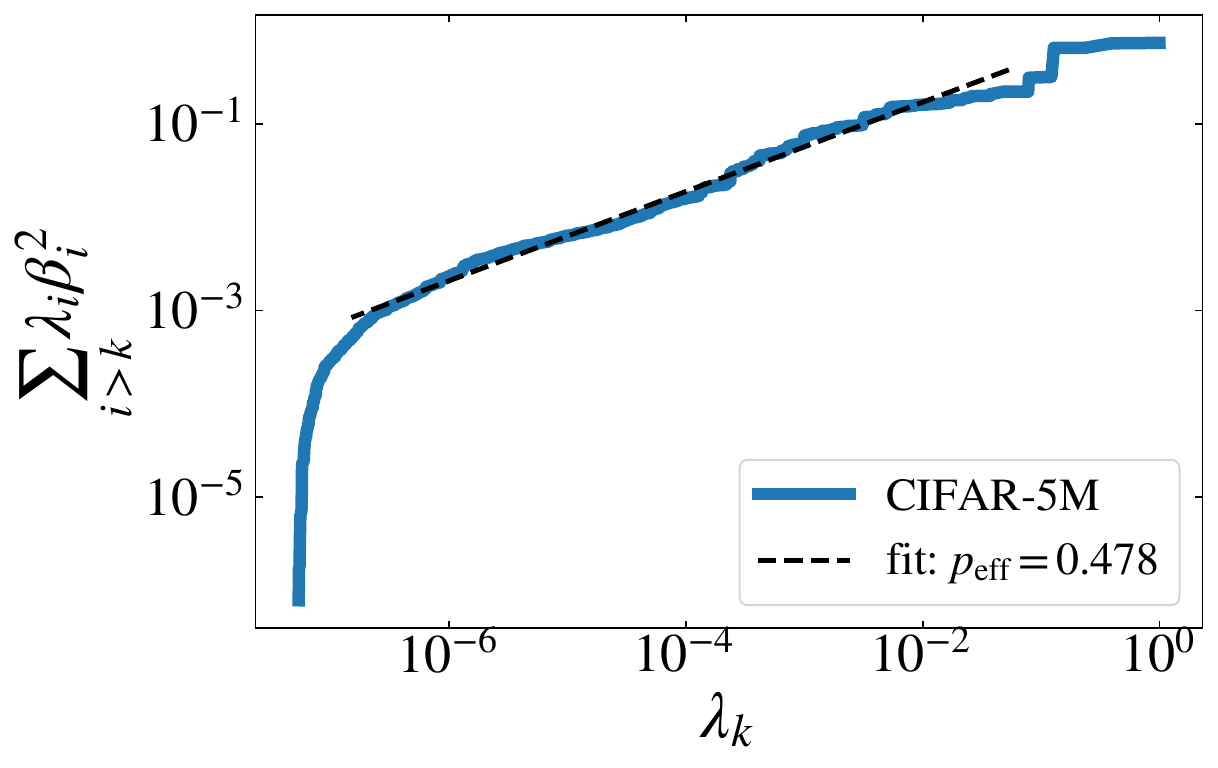}
  \caption{The tail sum $\sum_{i>k}\lambda_i\beta_i^2$ versus $\lambda_k$ for the CIFAR-5M task. The dashed line is a power-law fit with $p_{\mathrm{eff}}\approx0.478$, using
    $10\le k\le\operatorname{round}(0.9d)=2765$.}
  \label{fig:cifar5m-task-tail}
\end{figure}

\paragraph{\Cref{fig:relu-mlp-loss}.}
We draw fresh Gaussian inputs from the synthetic task with $a=2$, $c_\lambda=c_\beta=1$, and $d=2048$, and generate noiseless targets $y=\boldsymbol\beta^\top\boldsymbol x$. We use $b=0.1$ for $p=0.6$ and $b=0.9$ for $p=1.4$. The bias-free MLPs have $L\in\{2,3,4\}$ trainable linear maps, hidden width $2048$, ReLU activations after the first $L-1$ maps, and a scalar linear output. The hidden maps use standard He initialization and the output map uses standard fan-in initialization, both at scale one. Each curve is from a single random initialization. We train by online SGD for $100{,}000$ steps with fresh batches of size $4096$ and a learning rate of $0.01$, without momentum or weight decay. We evaluate the loss on a fixed set of $65{,}536$ independent samples.

\end{document}

%% file: math_commands.tex
\usepackage{amsmath,amsfonts,bm}

\def\eqref#1{equation~\ref{#1}}

\def\1{\bm{1}}

\DeclareMathAlphabet{\mathsfit}{\encodingdefault}{\sfdefault}{m}{sl}
\SetMathAlphabet{\mathsfit}{bold}{\encodingdefault}{\sfdefault}{bx}{n}

